\documentclass[10pt, twocolumn]{article}

\usepackage[top=0.75in,bottom=1in,left=0.75in,right=0.75in,columnsep=0.25in]{geometry}
\usepackage{times}
\usepackage[T1]{fontenc}

\usepackage[american]{babel}

\usepackage{natbib} 
\usepackage{mathtools} 
\usepackage{booktabs} 
\usepackage{tikz} 
\usepackage{subcaption}

\usepackage[nohints]{minitoc}
\usepackage{hyperref}
\usepackage{amsmath,amsfonts,amssymb,bbm,bm}
\usepackage{dsfont} 
\usepackage{xcolor}
\usepackage{amsthm}
\usepackage{thmtools}
\usepackage{mathtools}

\definecolor{niceRed}{RGB}{190,38,38}
\definecolor{niceYellow}{HTML}{f5b400}
\definecolor{blueGrotto}{HTML}{059DC0}
\definecolor{royalBlue}{HTML}{057DCD}
\definecolor{navyBlue}{HTML}{0B579C}
\definecolor{limeGreen}{HTML}{81B622}
\definecolor{nicePurple}{HTML}{9c27b0}
\definecolor{lightRoyalBlue}{HTML}{def2ff}
\definecolor{gold}{HTML}{ffa300}
\definecolor{frameblue}{RGB}{0,123,255}     
\definecolor{framebg}{RGB}{241,244,247}     
\definecolor{shadecolor}{gray}{0.90}

\usepackage{enumitem}
\usepackage{mdframed} 
\usepackage{amsthm}
\usepackage{thmtools}
\definecolor{shadecolor}{gray}{0.90}
\declaretheoremstyle[
headfont=\normalfont\bfseries,
notefont=\mdseries, notebraces={(}{)},
bodyfont=\normalfont,
postheadspace=0.5em,
spaceabove=0.5em,
spacebelow=0.5em,
mdframed={
 skipabove=8pt,
 skipbelow=8pt,
 hidealllines=true,
 backgroundcolor={shadecolor},
 innerleftmargin=4pt,
 innerrightmargin=4pt}
]{shaded}
\declaretheoremstyle[
  bodyfont=\normalfont  
]{nonitalic}
\newmdenv[
    backgroundcolor=framebg,
    skipabove=8pt,
    skipbelow=8pt,
    roundcorner=5pt,
    linewidth=0pt,
    innertopmargin=4pt
]{myframe}

\declaretheorem[within=section]{style}
\declaretheorem[style=shaded,sibling=style]{definition}
\declaretheorem[sibling=style]{main theorem}

\declaretheorem[style=shaded,sibling=style]{assumption}

\declaretheorem[style=nonitalic,sibling=style]{lemma}

\declaretheorem[style=nonitalic,sibling=style, numbered=no]{remark}
\declaretheorem[style=nonitalic,sibling=style, numbered=no]{proof sketch}

\usepackage[capitalize,nameinlink]{cleveref}
\crefname{assumption}{assumption}{assumptions}
\Crefname{assumption}{Assumption}{Assumptions}
\crefname{lemma}{lemma}{lemmas}
\Crefname{lemma}{Lemma}{Lemmas}
\crefname{theorem}{theorem}{theorems}
\Crefname{theorem}{Theorem}{Theorems}
\crefname{main theorem}{main theorem}{main theorems}
\Crefname{main theorem}{Main theorem}{Main theorems}
\crefname{definition}{definition}{definitions}
\Crefname{definition}{Definition}{Definitions}
\crefname{corollary}{corollary}{corollaries}
\Crefname{corollary}{Corollary}{Corollaries}
\crefname{proposition}{proposition}{propositions}
\Crefname{proposition}{Proposition}{Propositions}
\crefname{fact}{fact}{facts}
\Crefname{fact}{Fact}{Facts}
\crefname{identity}{identity}{identities}
\Crefname{identity}{Identity}{Identities}
\crefname{function}{function}{functions}
\Crefname{function}{Function}{Functions}

\newcommand{\R}{\mathbb{R}}

\renewcommand{\Pr}{\mathbb{P}}
\newcommand{\E}{\mathbb{E}}

\newcommand{\Var}{\mathrm{Var}}

\newcommand{\Ent}{\mathrm{Ent}}

\newcommand{\calA}{\mathcal{A}}
\newcommand{\calB}{\mathcal{B}}
\newcommand{\calC}{\mathcal{C}}
\newcommand{\calD}{\mathcal{D}}

\newcommand{\calF}{\mathcal{F}}
\newcommand{\calG}{\mathcal{G}}

\newcommand{\calI}{\mathcal{I}}

\newcommand{\calO}{\mathcal{O}}

\newcommand{\calS}{\mathcal{S}}

\newcommand{\calZ}{\mathcal{Z}}

\newcommand{\bone}{\boldsymbol{1}}

\def\<{\langle}
\def\>{\rangle}

\newcommand{\softmax}{\mathrm{softmax}}

\newcommand{\diag}{\mathrm{diag}}

\newcommand{\pr}{\mathbf{P}}     

\newcommand{\bw}{\mathbf{w}} 
\newcommand{\bu}{\mathbf{u}} 
\newcommand{\btu}{\mathbf{\tilde{u}}} 
\newcommand{\bbu}{\mathbf{\bar{u}}} 
\newcommand{\bff}{\mathbf{f}} 
\newcommand{\bg}{\mathbf{g}} 
\newcommand{\bh}{\mathbf{h}} 
\newcommand{\bp}{\mathbf{p}} 
\newcommand{\bq}{\mathbf{q}} 
\newcommand{\br}{\mathbf{r}} 
\newcommand{\bx}{\mathbf{x}} 
\newcommand{\by}{\mathbf{y}} 
\newcommand{\bv}{\mathbf{v}} 
\newcommand{\bzeta}{\boldsymbol{\zeta}} 
\newcommand{\bxi}{\boldsymbol{\xi}} 
\newcommand{\bz}{\mathbf{z}} 
\newcommand{\be}{\mathbf{e}} 
\newcommand{\bd}{\mathbf{d}} 
\newcommand{\bdelta}{\boldsymbol{\delta}} 
\newcommand{\ba}{\boldsymbol{a}} 
\newcommand{\bb}{\boldsymbol{b}} 
\newcommand{\bA}{\mathbf{A}} 
\newcommand{\bB}{\mathbf{B}} 
\newcommand{\bS}{\mathbf{S}} 

\newcommand{\mathd}{\mathrm{d}}
\newcommand{\mathe}{\mathrm{e}}

\newcommand{\ind}{\mathrel{\perp\!\!\!\perp}} 
\newcommand{\1}{\mathds{1}} 

\hypersetup{colorlinks=true,linkcolor=royalBlue,citecolor=royalBlue,urlcolor=navyBlue}
\usepackage{algorithm}
\usepackage{algorithmic}
\renewcommand{\algorithmiccomment}[1]{\hfill {\color{gray} // #1}}
\makeatletter
\newcommand{\PreComment}[1]{%
  \item[] 
  {\color{gray} // #1}%
}
\makeatother
\usepackage{float} 

\newfloat{function}{htbp}{lof}
\floatname{function}{Function}
\usepackage{multicol}
\usepackage[most]{tcolorbox}
\newtcolorbox{summarybox}[1][]{
    colback=gray!5, 
    colframe=black!75, 
    fonttitle=\bfseries,
    arc=3pt, 
    outer arc=3pt,
    left=10pt,
    right=10pt,
    top=10pt,
    bottom=10pt,
    boxrule=0.8pt,
    #1
}

\usepackage{amsmath}
\usepackage{amssymb}
\usepackage{amsthm}

\usepackage{titlesec}
\titleformat{\section}{\large\bfseries}{\thesection}{1em}{}
\titleformat{\subsection}{\normalsize\bfseries}{\thesubsection}{1em}{}
\titleformat{\subsubsection}{\normalsize\bfseries}{\thesubsubsection}{0.2em}{}
\titlespacing*{\subsubsection}{0pt}{1ex plus 1ex minus .2ex}{0.5ex plus .2ex}

\usepackage{enumitem}
\setlist{itemsep=2pt, topsep=4pt, parsep=0pt, partopsep=0pt, leftmargin=*}

\renewenvironment{abstract}
  {\centerline{\bf \large Abstract}\vspace{0.4ex}\begin{quote}\normalsize}
  {\end{quote}\vspace{0.4ex}}
\newenvironment{acknowledgements}
  {\section*{Acknowledgements}}
  {}
\date{} 

\begin{document}

\twocolumn[
  \begin{center}
    \rule{\textwidth}{3pt} 
    \vspace{0.05cm} 

    {\Large \bfseries First-Order Softmax Weighted Switching Gradient Method for \\ 
    Distributed Stochastic Minimax Optimization with Stochastic Constraints \par}

    \vspace{0.3cm} 
    \rule{\textwidth}{1.2pt} 
    \vspace{0.4cm} 
    
    {\normalsize 
      \textbf{Zhankun Luo}, \quad \textbf{Antesh Upadhyay}, \quad \textbf{Sang Bin Moon}, \quad \textbf{Abolfazl Hashemi} \\
      
      \vspace{0.3cm} 
      
      School of Electrical and Computer Engineering, Purdue University \\
      \texttt{\{luo333, aantesh, moon182, abolfazl\}@purdue.edu}
    }
    \vspace{1cm} 
  \end{center}
]

\begin{abstract}
This paper addresses the distributed stochastic minimax optimization problem subject to stochastic constraints. We propose a novel first-order \textit{Softmax-Weighted Switching Gradient} method tailored for federated learning. Under full client participation, our algorithm achieves the standard $\tilde{\mathcal{O}}(\epsilon^{-4})$ oracle complexity to satisfy a unified bound $\epsilon$ for both the optimality gap and feasibility tolerance. We extend our theoretical analysis to the practical partial participation regime by quantifying client sampling noise through a stochastic superiority assumption. Furthermore, by relaxing standard boundedness assumptions on the objective functions, we establish a strictly tighter lower bound for the softmax hyperparameter. We provide a unified error decomposition and establish a sharp $\mathcal{O}(\log\frac{1}{\delta})$ high-probability convergence guarantee. Ultimately, our framework demonstrates that a single-loop primal-only switching mechanism provides a stable alternative for optimizing worst-case client performance, effectively bypassing the hyperparameter sensitivity and convergence oscillations often encountered in traditional primal-dual or penalty-based approaches. We verify the efficacy of our algorithm via experiment on the Neyman-Pearson (NP) classification, fair classification, and federated safe reinforcement learning tasks.
\end{abstract}

\doparttoc[n] 
\faketableofcontents 

\section{Introduction}\label{sec:intro}


Federated Learning (FL) aims to solve distributed optimization problems of the form
\begin{equation}
\label{eq:fed_op_prb}
\min_{\mathbf{w} \in \Theta} \; \sum_{i=1}^{n} p_i f_i(\mathbf{w}),
\end{equation}
where $\Theta \subseteq \mathbb{R}^d$ is a compact convex set, $n$ is the number of clients, $f_i(\mathbf{w}) := \mathbb{E}_{\zeta \sim \mathcal{D}_i}[f_i(\mathbf{w},\zeta)]$ denotes the local expected loss at client $i$ and $p_i$ denote the probability weight of the client~\citep{mcmahan2017communication,kairouz2021advances}. Under statistical heterogeneity, where local distributions $\{\mathcal{D}_i\}_{i=1}^n$ are non-identical, this empirical risk minimization (ERM) objective inherently prioritizes average performance across clients~\citep{li2020federated,mohri2019agnostic}. As a result, the learned model is biased toward dominant client distributions and may exhibit severely degraded performance on underrepresented or hard clients~\citep{mohri2019agnostic,hashimoto2018fairness,li2019fair}.

To guarantee uniformly good performance across all devices, a powerful alternative is to frame the training process as a distributionally robust (or agnostic) optimization problem~\citep{mohri2019agnostic,deng2020distributionally,duchi2021learning}. Instead of minimizing the average loss, the algorithm minimizes the maximum expected loss over a global set of adversarial weights
\begin{equation}
\label{eq:afl_objective}
\min_{\mathbf{w} \in \Theta} \; \max_{\boldsymbol{\lambda} \in \Lambda} \; \sum_{i=1}^{n} \lambda_i f_i(\mathbf{w}),
\end{equation}
where 
$
\Lambda := \left\{ \boldsymbol{\lambda} \in \mathbb{R}_+^n \; : \; \sum_{i=1}^{n} \lambda_i = 1 \right\}
$
represents the probability weight assigned to each local client $i$. Intuitively, the inner maximization concentrates probability mass on the worst-performing clients. This recovers the equivalent minimax formulation \(\min_{\mathbf{w} \in \Theta} \; \max_{i \in \mathcal{I}} f_i(\mathbf{w})\) with $\mathcal{I}:=\{1, 2, \ldots, n\}$
which directly enforces robustness to client heterogeneity.

Existing minimax formulations in federated settings typically optimize this worst-case loss in isolation. However, in many practical deployments, models must simultaneously satisfy strict client-wise operational requirements, such as fairness mandates, safety limits, resource budgets, or regulatory thresholds~\citep{islamov2025safeef,upadhyay2026fedsgm}. Tracking $n$ distinct stochastic constraints, i.e., $g_i(\mathbf{w}) = \mathbb{E}_{\zeta \sim \mathcal{D}_i}[g_i(\mathbf{w},\zeta)] \le 0, \; \forall i \in \mathcal{I}$, which encode client-specific operational constraints, is prohibitively expensive in federated environments, as it requires maintaining and synchronizing $n$ distinct dual variables across a network with intermittent client availability. To circumvent this communication and memory bottleneck, these considerations naturally lead to a constrained worst-case formulation over the discrete maximum,
\begin{equation}
\label{eq:constrained_minimax_intro}
    \min_{\mathbf{w} \in \Theta} \max_{i \in \mathcal{I}} f_i(\mathbf{w})
    \quad \text{s.t.} \quad
    \max_{i \in \mathcal{I}} g_i(\mathbf{w}) \le 0,
\end{equation}
This stochastic minimax problem with stochastic constraints captures a strictly more challenging setting than the standard agnostic setting, as it requires controlling both worst-case performance and constraint violation across heterogeneous distributions without explicitly relying on separate dual variables. Solving \eqref{eq:constrained_minimax_intro} in FL environments raises several nontrivial theoretical and algorithmic challenges, including:

\paragraph{Non-smooth worst-case objective.}
The client-wise maximum $\max_{i \in \mathcal{I}} f_i(\mathbf{w})$ and $\max_{i \in \mathcal{I}} g_i(\mathbf{w})$ induces an inherently non-smooth objective and constraint landscape. In particular, when multiple clients attain similar worst-case losses or constraint violations, the subdifferential becomes highly sensitive to stochastic perturbations. Under noisy local estimates, the identity of the worst client may fluctuate across rounds, causing standard gradient-based methods to exhibit oscillatory behavior around the feasibility boundary.

\paragraph{Worst-Case constraints.}
$\max_{i \in \mathcal{I}} g_i(\mathbf{w}) \le 0$ is essential when hard operational requirements for \textit{every} client are non-negotiable, unlike average constraints. 
In federated control for safety-critical systems (e.g., robotic fleets~\citep{brohan2022rt}), a single agent violating its safety envelope causes irreversible damage regardless of the fleet's \textit{average} behavior. Maximizing the worst-client reward alone cannot control safety violations. Enforcing a global constraint $\max_{i \in \mathcal{I}} g_i(\mathbf{w}) \le 0$ compactly ensures \textit{all} client-level safety budgets are strictly respected without requiring the optimization algorithm to track separate client-level dual variables. 

\paragraph{Coupling of constraints with minimax optimization.}
While unconstrained agnostic FL \citep{mohri2019agnostic} can relax the discrete client maximum into a smooth probability simplex (as in \eqref{eq:afl_objective}) to employ standard stochastic saddle-point algorithms like SGDA, incorporating non-smooth constraints shatters this convenience. Standard constrained optimization relies on primal-dual or penalty-based approaches (e.g., ADMM), which require maintaining, tuning, and communicating dual variables. In federated networks, primal-dual methods suffer from severe ``dual drift'' and instability under stochastic gradients and partial client participation, as inactive clients cause their corresponding dual variables to become stale~\citep{wang2022fedadmm,sun2024fedpd}. Moreover, many existing constrained methods rely on deterministic gradients, bounded losses, or inner optimization subroutines, which are highly restrictive in large-scale stochastic federated settings.

Consequently, naive saddle-point reformulations yield degenerate adversarial dynamics, unstable dual updates, and communication overhead. This necessitates a new algorithmic approach to constrained federated and distributed minimax optimization that avoids the pitfalls of dual-variable synchronization while guaranteeing strict worst-case constraint satisfaction.

To address these issues, we propose a stochastic \emph{Softmax-Weighted Switching Gradient} method for distributed and federated stochastic minimax problem with stochastic constraints. The key idea is to replace the non-smooth hard maximum with a smooth, temperature-controlled Softmax approximation that generates smooth adversarial weights over participating clients,
\begin{equation}
\mathbf{p}_k = \operatorname{softmax}(\alpha \mathbf{f}(\mathbf{w}_k)),
\end{equation}
where $\mathbf{f}(\mathbf{w}_k) := (f_1(\mathbf{w}_k), \dots, f_n(\mathbf{w}_k))$ and $\alpha \geq 0$ controls the approximation tightness, similarly for $\bg(\bw_k)=(g_1(\mathbf{w}_k), \dots, g_n(\mathbf{w}_k))$. This formulation stabilizes the gradient landscape while preserving sensitivity to worst-case clients. Crucially, we couple this smooth minimax approximation with a first-order switching mechanism~\citep{polyak1967general,upadhyay2025optimization}: when the estimated global constraint violation is within a prescribed tolerance, the algorithm prioritizes worst-case objective minimization; otherwise, it adaptively redirects updates toward reducing constraint violations. This design eliminates the need for explicit dual variables, inner optimization loops, or deterministic gradient access, making the method fully compatible with stochastic first-order oracles, multiple local updates, and partial client participation.

From a theoretical perspective, our framework departs from existing robust and meta-learning approaches in several important ways. Unlike downstream adaptation and constrained meta-learning methods that rely on centralized optimization or restrictive boundedness assumptions on the (loss) functions \citep{wang2023task}, we analyze a fully stochastic federated setting with heterogeneous client distributions and stochastic constraint evaluations. Furthermore, our method operates as a single-loop first-order algorithm without requiring the solution of auxiliary optimization problems at each round, in sharp contrast to many primal-dual or ERM-oracle-based minimax methods. These distinctions make our approach particularly suitable for large-scale, communication-constrained federated systems where robustness, feasibility, and scalability must be addressed simultaneously.
Please see Appendix~\ref{sup:related} for a detailed discussion of related work.

\subsection{Contributions} 

In this paper, we present the \textit{Softmax-Weighted Switching Gradient} method to solve distributed stochastic optimization problems with stochastic constraints (formally defined in \cref{eq:opt_constrained}, Section~\ref{sec:problem_setup}). Our core contributions are summarized as follows:

\begin{itemize}
    \item \textbf{Novel Constrained Minimax Framework:} We propose a single-loop, first-order algorithm that solves stochastic constrained minimax problems in FL without explicit dual variables (Section~\ref{sec:algo}), achieving the canonical $\tilde{\mathcal{O}}(\epsilon^{-4})$ oracle complexity for stochastic constrained setting~\citep{lan2020algorithmsstochastic}. This fundamentally bypasses the ``dual drift'' and instability issues prevalent in heterogeneous federated networks.
    
    \item \textbf{Relaxation of Boundedness Assumptions:} Building upon foundational softmax-based minimax approaches for deterministic, centralized, and unconstrained settings~\citep{wang2023task}, our theoretical analysis successfully relaxes the requirement for strictly bounded objective functions. This advancement allows us to establish a tighter, more generalized lower bound for the softmax hyperparameter $\alpha$ (Section~\ref{sec:theory}), yielding improved theoretical guarantees that apply broadly, including in purely centralized environments.
    
    \item \textbf{Unified Error Decomposition for General FL Settings:} We establish rigorous high-probability convergence guarantees under practical federated constraints, explicitly accommodating multiple local updates and partial client participation. Our analysis cleanly decouples the optimality gap and feasibility tolerance into three distinct sources: optimization error, stochastic estimation error, and client sampling error (\Cref{box:err_decompose}).
    
    \item \textbf{Empirical Validation:} Finally, we evaluate the robustness of our approach through diverse empirical trials, including the NP classification and fair classification tasks. By benchmarking against traditional primal-dual and penalty-based methods, we offer important practical observations regarding the reliability of primal-only switching methods in the face of client heterogeneity
    (Section~\ref{sec:experiments}).
\end{itemize}

\section{Problem Setup}\label{sec:problem_setup}
In this section, we formally state the setup 
of a distributed stochastic minimax optimization problem with stochastic constraints,
and introduce notations for theoretical analysis.

\textbf{Problem Formulation.} 
Consider a set of $n$ clients indexed by $i \in \mathcal{I}=\{1, 2, \dots, n\}$, where each client is associated with a local data distribution $\mathcal{D}_i$ on the sample space $\mathcal{Z}$. We define two vector-valued functions, $\mathbf{f}: \Theta \times \mathcal{Z} \to \mathbb{R}^{n}$ and $\mathbf{g}: \Theta \times \mathcal{Z} \to \mathbb{R}^{n}$. 
Specifically, for each client $i$, let $f_i(\mathbf{w}, \zeta) = [\mathbf{f}(\mathbf{w}, \zeta)]_i$ and $g_i(\mathbf{w}, \zeta) = [\mathbf{g}(\mathbf{w}, \zeta)]_i$ denote the local objective and constraint values for a sample $\zeta \sim \mathcal{D}_i$. 
The corresponding local expectations are given by $f_i(\mathbf{w}) = \mathbb{E}_{\zeta\sim \mathcal{D}_i}[f_i(\mathbf{w}, \zeta)]$ and $g_i(\mathbf{w}) = \mathbb{E}_{\zeta\sim \mathcal{D}_i}[g_i(\mathbf{w}, \zeta)]$.
The expectation-constrained minimax optimization problem is then formulated as:
\begin{equation}\tag{\(\dagger\)}\label{eq:opt_constrained}
\begin{aligned}
    & \min_{\mathbf{w} \in \Theta} F(\mathbf{w}) = \min_{\mathbf{w} \in \Theta} \max_{i \in \mathcal{I}} [\bff(\mathbf{w})]_i = \min_{\mathbf{w} \in \Theta} \max_{i \in \mathcal{I}} f_i(\mathbf{w}) \\
    & \text{s.t.} \quad G(\mathbf{w}):= \max_{i\in\mathcal{I}} [\bg(\mathbf{w})]_i = \max_{i \in \mathcal{I}} g_i(\mathbf{w}) \leq 0
\end{aligned}
\end{equation}
In the above problem, we minimize the maximum of the objective functions \(F(\bw):=\max_{i \in \mathcal{I}} f_i(\bw)\) while all the constraints are less than or equal to 0, akin to \(G(\bw):=\max_{i \in \mathcal{I}} g_i(\bw)\leq 0\).
We assume the optimal solution \(\mathbf{w}^*\) exists and the optimal value is \(F(\mathbf{w}^*)\) with \(G(\mathbf{w}^*) \leq 0\).
\begin{equation}\tag{\(\ast\)}\label{eq:opt_sol}
\mathbf{w}^* \in \arg \min_{\mathbf{w} \in \Theta} F(\mathbf{w}) \quad \text{s.t.} \quad G(\mathbf{w}) \leq 0
\end{equation}

\textbf{Notations.} 
For the brevity of notations, the Jacobian matrices of the expectation functions \(\bff(\bw), \bg(\bw)\) can be written as \(\nabla \bff(\mathbf{w}) = (\nabla f_1(\mathbf{w}), \nabla f_2(\mathbf{w}), \ldots, \nabla f_{n}(\mathbf{w}))^\top\) and \(\nabla \bg(\mathbf{w}) = (\nabla g_1(\mathbf{w}), \nabla g_2(\mathbf{w}), \ldots, \nabla g_{n}(\mathbf{w}))^\top\) respectively,
where \(\nabla f_i(\bw) \in \partial f_i(\bw)\) and \(\nabla g_i(\bw) \in \partial g_i(\bw)\) are subgradients of \(f_i(\bw)\) and \(g_i(\bw)\) 
under \Cref{assumption:convexity} on convexity of \(f_i(\bw)\) and \(g_i(\bw)\).

To evaluate the function values of the objective and constraint functions \(\bff(\bw), \bg(\bw)\), 
we take batches of data samples 
\(\bxi^{(i)}=(\xi^{(i)}_1, \ldots, \xi^{(i)}_{B_\zeta})\stackrel{\text{i.i.d.}}{\sim} \calD_i\),
drawn from the distribution \(\calD_i\) to approximate 
the expectations \([\bff(\bw)]_i=f_i(\bw)=\E_{\zeta\sim \calD_i}[f_i(\bw, \zeta)]\) and \([\bg(\bw)]_i=g_i(\bw)=\E_{\zeta\sim \calD_i}[g_i(\bw, \zeta)]\)
for \(i\in \calI\) using the batches of data samples \(\bxi^{(i)}\) as follows:
\begin{equation}\label{eq:approx_function_values}
\begin{aligned}
    & f_i(\bw, \bxi^{(i)}) := \frac{1}{B_\zeta}\sum\limits_{s=1}^{B_\zeta} f_i(\bw, \xi^{(i)}_{s}), \\
    & g_i(\bw, \bxi^{(i)}) := \frac{1}{B_\zeta}\sum\limits_{s=1}^{B_\zeta} g_i(\bw, \xi^{(i)}_{s}).
\end{aligned}
\end{equation}
We write \(\bxi = (\bxi^{(i)})_{i\in\calI}\) and \(\bff(\bw, \bxi) = (f_i(\bw, \bxi^{(i)}))_{i\in\calI}\),
\(\bg(\bw, \bxi) = (g_i(\bw, \bxi^{(i)}))_{i\in\calI}\) to represent the batches of data samples 
and the corresponding approximated function values of the objective and constraint functions respectively.
Similarly, we can evaluate the gradients of the objective and constraint functions \(\bff(\bw), \bg(\bw)\) 
using batches of data samples \(\bzeta^{(i)}=(\zeta^{(i)}_1, \ldots, \zeta^{(i)}_{B_g})\stackrel{\text{i.i.d.}}{\sim} \calD_i\)
drawn from the distribution \(\calD_i\) as follows:
\begin{equation}\label{eq:approx_gradients}
\begin{aligned}
    & \nabla f_i(\bw, \bzeta^{(i)}) := \frac{1}{B_g}\sum\limits_{s=1}^{B_g} \nabla f_i(\bw, \zeta^{(i)}_{s}), \\
    & \nabla g_i(\bw, \bzeta^{(i)}) := \frac{1}{B_g}\sum\limits_{s=1}^{B_g} \nabla g_i(\bw, \zeta^{(i)}_{s}).
\end{aligned}
\end{equation}
Correspondingly, we write \(\bzeta = (\bzeta^{(i)})_{i\in\calI}\) and \(\nabla \bff(\bw, \bzeta) = (\nabla f_i(\bw, \bzeta^{(i)}))_{i\in\calI}\),
\(\nabla \bg(\bw, \bzeta) = (\nabla g_i(\bw, \bzeta^{(i)}))_{i\in\calI}\) to represent the batches of data samples 
used to evaluate the gradients, and the approximated gradients of the objective and constraint functions.

Regarding the minimization problem of the maximum function value \(F(\bw):=\max_{i \in \mathcal{I}} f_i(\bw)\) 
under the constraint \(G(\bw):=\max_{i \in \mathcal{I}} g_i(\bw)\leq 0\).
We can approximate the maximum function value \(F(\bw)\) and the constraint \(G(\bw)\) 
using the batches of data samples \(\bxi\) 
such that \(F(\bw, \bxi) = \max_{i \in \mathcal{I}} f_i(\bw, \bxi^{(i)})\) and \(G(\bw, \bxi) = \max_{i \in \mathcal{I}} g_i(\bw, \bxi^{(i)})\).

\section{Algorithm}\label{sec:algo}
In this section, we propose a \textit{Softmax-Weighted Switching Gradient} method for solving the constrained minimax optimization problem in~\eqref{eq:opt_constrained} (Algorithm~\ref{alg:switching_gd_softmax_fedpartial}).

\textbf{Basic Switching Strategy.} 
The switching strategy implements a simple mechanism~\citep{polyak1967general,nesterov2018lectures} based on the maximum constraint violation at the \(k\)-th iteration of global round over all clients \(i\in \calI\),
\begin{equation}
G(\mathbf{w}_k,\bxi_k)  
= \max_{i \in \mathcal{I}} g_i(\mathbf{w}_k, \bxi^{(i)}_{k}). 
\end{equation}
When constraints are satisfied such that \(G(\mathbf{w}_k,\bxi_k) \leq \epsilon\) at iteration \(k\), the algorithm focuses on minimizing the objective function using approximated gradients \(\nabla \bff(\bw_k, \bzeta_k)\) computed over the data batches \(\bzeta_k\). Otherwise, it prioritizes feasibility by updating the constraint gradients \(\nabla \bg(\bw_k, \bzeta_k)\).

\begin{algorithm}[t]
\caption{(\textit{Softmax-Weighted Switching Gradient})\\Federated Learning with Partial Participation ($m\leq n$)\\
(Local Update Steps \(E\geq1\))}
\label{alg:switching_gd_softmax_fedpartial}
\begin{algorithmic}[1] 
\STATE \textbf{Input:} Initial parameters $\mathbf{w}_0$, global/local step size $(\eta,\gamma)$, 
threshold $\lambda$, softmax hyperparameter $\alpha$, size of subsets $m(\leq n)$
\STATE $\mathcal{S} \gets \emptyset$ \COMMENT{Initialize set of constraint-satisfied iterations}
\FOR{global round $k \in [K]$}
\PreComment{Sample subsets with cardinality $m$ from $\calI$ uniformly and independently}
\PreComment{$\calC_{m}(\calI) = \{A \subseteq \calI \mid |A| = m\}$}
\STATE \(\calI_k\sim \text{Unif}(\calC_{m}(\calI))\)
    \STATE \textbf{broadcast} $\mathbf{w}_k$ to clients $\calI_k$
    \FOR{each client $i \in \calI_k$ \textbf{in parallel}}
        \PreComment{Samples for Function Value Evaluation}
        \STATE $\bxi_k^{(i)} \equiv ([\bxi_k^{(i)}]_s)_{s=1}^{B_\zeta} \stackrel{\text{i.i.d.}}{\sim} \calD_i$ 
        \STATE $[\bff(\mathbf{w}_k, \bxi_k)]_i \gets f_i(\mathbf{w}_k, \bxi_k^{(i)})$ \COMMENT{See \Cref{eq:approx_function_values}}
        \STATE $[\bg(\mathbf{w}_k, \bxi_k)]_i \gets g_i(\mathbf{w}_k, \bxi_k^{(i)})$ \COMMENT{See \Cref{eq:approx_function_values}}
    \ENDFOR
    \STATE \textbf{collect} $[\bff(\mathbf{w}_k, \bxi_k)]_i, [\bg(\mathbf{w}_k, \bxi_k)]_i$ from client $i\in \calI_k$\\
    \STATE $\bp_k \gets \softmax_{\calI_k}(\alpha \bff(\mathbf{w}_k, \bxi_k))$ \COMMENT{Softmax weights}
    \STATE $\bq_k \gets \softmax_{\calI_k}(\alpha \bg(\mathbf{w}_k, \bxi_k))$ \COMMENT{See \Cref{eq:masked_softmax}}
    \IF{${\color{brown}G_k(\mathbf{w}_k, \bxi_k; \calI_k) \equiv\langle \bq_k, \bg(\mathbf{w}_k, \bxi_k) \rangle} 
    \leq {\color{royalBlue}\lambda}$}
        \STATE $\mathcal{S} \gets \mathcal{S} \cup \{k\}$ \COMMENT{Add iteration to set}
    \ENDIF
    \STATE \textbf{broadcast} $\1_k \equiv \1_{G_k(\mathbf{w}_k, \bxi_k; \calI_k) \leq \lambda}$ 
    to clients $\calI_k$
    \FOR{each client $i \in \calI_k$ \textbf{in parallel}}
        \PreComment{\Cref{function:local_solver} to compute the update direction}
        \STATE $\bu_k^{(i)} \gets \texttt{LocalSolver}(\mathbf{w}_k, \1_k, \gamma; i)$
    \ENDFOR
    \STATE \textbf{collect} $\bu_k^{(i)}$ from each client $i\in \calI_k$
    \STATE $\bu_k \gets \sum_{i \in \calI_k} (\1_k [\bp_k]_i + (1-\1_k)[\bq_k]_i)\bu_k^{(i)}$
    \STATE $\mathbf{w}_{k+1} \gets \mathbf{w}_k - \eta \bu_k$ \COMMENT{Update parameters}
\ENDFOR
\STATE $\overline{\mathbf{w}}_K \gets \frac{1}{|\mathcal{S}|} \sum_{k \in \mathcal{S}} \mathbf{w}_k$ \COMMENT{Averaged solution}
\STATE \textbf{Output:} Optimized parameters $\overline{\mathbf{w}}_K$
\end{algorithmic}
\end{algorithm}

In a federated setting, evaluating these updates requires local client participation. We denote \([K]:=\{0,1,\ldots,K-1\}\) as the index set for global rounds, and \([E]:=\{0,1,\ldots,E-1\}\) for local update iterations. For each global round \(k\in [K]\), clients \(i\in \calI\) compute either the objective gradient \(\nabla f_i(\bw_k, \bzeta_{k,\tau}^{(i)})\) or the constraint gradient \(\nabla g_i(\bw_k, \bzeta_{k,\tau}^{(i)})\) in parallel over \(E\) local steps using their respective batches.


\begin{function}[t]
\caption{Local Solver to compute local update direction
}
\label{function:local_solver}
\begin{algorithmic}[1]
\STATE \textbf{function} \texttt{LocalSolver}$(\mathbf{w}_k, \1_k, \gamma; i)$ \\
    \STATE $\bw_{k,0}^{(i)} \gets \bw_k$ \COMMENT{Initialize local parameters}
    \FOR{local update step $\tau \in[E]$}
        \PreComment{Samples for Local Gradient Evaluation}
        \STATE $\bzeta_{k,\tau}^{(i)} \equiv ([\bzeta_{k,\tau}^{(i)}]_s)_{s=1}^{B_g} \stackrel{\text{i.i.d.}}{\sim} \calD_i$ 
        \PreComment{Local Gradient Evaluation}
        \IF{$\1_k = 1$} 
            \STATE $\bv_{k,\tau}^{(i)} \gets \nabla f_i(\bw_{k,\tau}^{(i)}, \bzeta_{k,\tau}^{(i)})$ \COMMENT{See \Cref{eq:approx_gradients}}
        \ELSE 
            \STATE $\bv_{k,\tau}^{(i)} \gets \nabla g_i(\bw_{k,\tau}^{(i)}, \bzeta_{k,\tau}^{(i)})$ \COMMENT{See \Cref{eq:approx_gradients}}
        \ENDIF
        \STATE $\bw_{k,\tau+1}^{(i)} \gets \bw_{k,\tau}^{(i)} - \gamma \bv_{k,\tau}^{(i)}$ \COMMENT{Update local parameters}
    \ENDFOR
    \STATE \textbf{return} $\frac{\bw_{k,0}^{(i)}-\bw_{k,E}^{(i)}}{\gamma E}$ \COMMENT{Local update direction \(\bu_k^{(i)}\)}
\STATE \textbf{end function}
\end{algorithmic}
\end{function}

\textbf{Softmax-Weighted Constraint Evaluation.} 
Building upon the basic strategy, we introduce a softmax-weighted constraint evaluation. Rather than strictly tracking the single worst-case client, this approach evaluates constraints through a softmax-weighted combination of clients. This provides an approximation of the maximum function, which stabilizes the constraint evaluation against noisy local estimates and smoothly distributes the weights across near-worst-case clients. This evaluation is formulated as follows,
\begin{equation*}
    \begin{aligned}
    G_k(\mathbf{w}_k,\bxi_k) &:= \langle \softmax(\alpha \bg(\mathbf{w}_k,\bxi_k)), \bg(\mathbf{w}_k,\bxi_k) \rangle\\
    &= \frac{\sum_{i \in \mathcal{I}} \exp(\alpha g_i(\mathbf{w}_k,\bxi_k^{(i)})) g_i(\mathbf{w}_k,\bxi_k^{(i)})}{\sum_{i' \in \mathcal{I}} \exp(\alpha g_{i'}(\mathbf{w}_k,\bxi_k^{(i')}))}.
    \end{aligned}
\end{equation*}
Instead of checking the hard constraint violation \(G(\mathbf{w}_k) \leq \epsilon\), 
we check if \(G_k(\mathbf{w}_k,\bxi_k) \leq \lambda\) is satisfied at iteration \(k\). 
$\lambda$ is a tunable threshold parameter in the practical setting $\lambda \gets \frac{\epsilon}{1+A^{-1}}$ for any constant $A \geq 1$. 
This tightened tolerance accounts for the approximation gap between the softmax mean and the true maximum. Instead of relying on the gradient of a single worst-case $\nabla f_{i^\ast}(\bw)$ or $\nabla g_{i^\ast}(\bw)$, we assign smooth weights to each client $i$ using $\softmax$ based on their approximated function values $\bp_k = \softmax(\alpha \bff(\bw_k, \bxi_k)), \space
    \bq_k = \softmax(\alpha \bg(\bw_k, \bxi_k))$.

\textbf{Full Participation.} 
We consider the full participation scenario where all clients participate in the optimization process. In this setting, the algorithm operates through a sequence of broadcasting, local updating, and global aggregation at each global round $k \in [K]$.

First, the server broadcasts the global parameters $\bw_k$ to all clients $i \in \mathcal{I}$. The system evaluates the constraint violation $G_k(\bw_k, \bxi_k)$ to determine the global switching indicator $\1_k = \1_{G_k(\bw_k, \bxi_k) \le \lambda}$.

Next, clients perform $E$ local update iterations in parallel by initializing their local model as $\bw_{k,0}^{(i)} = \bw_k$. If the constraint is satisfied ($\1_k = 1$), clients update their parameters using the objective gradient $\nabla f_i(\bw_{k,\tau}^{(i)}, \bzeta_{k,\tau}^{(i)})$. Otherwise ($\1_k = 0$), they use the constraint gradient $\nabla g_i(\bw_{k,\tau}^{(i)}, \bzeta_{k,\tau}^{(i)})$. Using a local step size $\gamma$, this process repeats for $\tau \in [E]$, after which each client computes its normalized local updates $\bu_k^{(i)} = (\bw_{k,0}^{(i)} - \bw_{k,E}^{(i)})/(\gamma E)$ and returns it to the server.

Finally, the server aggregates these local updates to perform the global update. It computes the softmax weights $\bp_k$ or $\bq_k$ based on the approximated function values, and updates the global parameters with a global step size $\eta$,
\begin{equation}\label{eq:global_update}
\bw_{k+1} = \bw_k - \eta \sum_{i\in \calI} (\1_k[\bp_k]_i + [1-\1_k][\bq_k]_i) \bu_k^{(i)}.
\end{equation}

\textbf{Partial Participation.} 
In practical federated learning deployments, only a subset of clients participate in the optimization process during each global round. We denote the selected set of participating clients at iteration $k$ as $\calI_k \subseteq \calI$, with a fixed cardinality $|\calI_k| = m \leq n$. We assume these subsets are sampled uniformly and independently, such that $(\calI_k)_{k\in \mathbb{Z}_{\geq 0}} \stackrel{\text{i.i.d.}}{\sim} \text{Unif}(\mathcal{C}_m(\mathcal{I}))$, where $\calC_m(\calI) = \{A \subseteq \calI \mid |A| = m\}$.

To accommodate partial participation, we must restrict the softmax probability mass strictly to the participating clients. We achieve this by introducing a masked softmax operator associated with the subset $\calI_k$. Letting $\bone_{\calI_k}$ denote the indicator vector of $\calI_k$ (where $[\bone_{\calI_k}]_i = 1$ if $i \in \calI_k$ and $0$ otherwise), the masked softmax is defined as

\begin{equation}\label{eq:masked_softmax}
    \softmax_{\calI_k}(\bv) := \frac{\bone_{\calI_k}\odot \exp(\bv)}{\bone_{\calI_k}^\top \exp(\bv)}.
\end{equation}

Because only clients in $\calI_k$ participate at round $k$, the server only collects function values and gradients from this active subset. For concise notation, we denote the localized function evaluations as $\bff(\bw_k, \bxi_k)_{\calI_k} =[f_i(\bw_k, \bxi_k^{(i)})]_{i\in\calI_k}$ and $\bg(\bw_k, \bxi_k)_{\calI_k}=[g_i(\bw_k, \bxi_k^{(i)})]_{i\in\calI_k}$. Furthermore, the maximum function values over the participating subset are 

\begin{equation}\label{eq:max_function_values_subsets}
    \begin{aligned}
    F(\bw; \calI_k) & := \max_{i\in \calI_k} [\bff(\bw)]_i = \max_{i\in \calI_k} f_i(\bw), \\
    G(\bw; \calI_k) & := \max_{i\in \calI_k} [\bg(\bw)]_i = \max_{i\in \calI_k} g_i(\bw). \\
    \end{aligned}
\end{equation}

Instead of evaluating the global constraint criteria over all clients, we rely on the participating clients to determine feasibility. We evaluate the subset constraint $G_k(\bw_k, \bxi_k; \calI_k) \leq \lambda$ and encode this strategy in the switching indicator $\1_k = \1_{G_k(\bw_k, \bxi_k; \calI_k) \leq \lambda}$. This subset constraint is computed using the masked softmax function:
\begin{equation}\label{eq:constraint_criteria_subsets}
    \begin{aligned}
    G_k(\bw_k, \bxi_k; \calI_k) &:= \langle \softmax_{\calI_k}(\alpha \bg(\bw_k, \bxi_k)), \bg(\bw_k, \bxi_k) \rangle\\
    & =  \frac{\sum_{i \in \calI_k} \exp(\alpha g_i(\bw_k, \bxi_k^{(i)})) g_i(\bw_k, \bxi_k^{(i)})}{\sum_{i' \in \calI_k} \exp(\alpha g_{i'}(\bw_k, \bxi_k^{(i')}))}.
    \end{aligned}
\end{equation}

Similarly, the server computes the client update weights over the selected subset using the masked softmax, defined as $\bp_k=\softmax_{\calI_k}(\alpha \bff(\bw_k, \bxi_k))$ and $\bq_k =\softmax_{\calI_k}(\alpha \bg(\bw_k, \bxi_k))$. By restricting the evaluation to $\calI_k$, the generalization of the global objective and constraint is feasible only if the clients share certain structural proximity. As we establish in Section~\ref{sec:theory}, this approximation remains theoretically sound under some regularity assumption, provided the participating subset adequately captures the information fo the worst-case client.

\section{Theoretical Analysis}\label{sec:theory}
\begin{table*}[!t]
    \centering
    \caption{Comparison of iteration, sample, and total oracle complexities under high-probability guarantees.}
    \label{tab:oracle_complexity}
    \renewcommand{\arraystretch}{1.15}
    \begin{tabular}{@{}lccc@{}}
        \toprule
        Method & Iteration complexity ($K$) & Sample size ($B_\zeta$) & Total oracle complexity \\
        \midrule
        \citealp{lan2020algorithmsstochastic} & $\mathcal{O}(\epsilon^{-2}\log^2\frac{1}{\delta})$ & $\mathcal{O}(\epsilon^{-2}\log^2\frac{1}{\delta})$ & $\mathcal{O}(\epsilon^{-4}\log^4\frac{1}{\delta})$ \\
        Ours (\Cref{theorem:convergence_gd_softmax_lipschitz}) & $\mathcal{O}(\epsilon^{-2}\log\frac{1}{\delta})$ & $\mathcal{O}(\epsilon^{-2}\log\frac{1}{\epsilon\delta})$ & $\mathcal{O}(\epsilon^{-4}\log\frac{1}{\delta}\log\frac{1}{\epsilon\delta})$ \\
        \bottomrule
    \end{tabular}
\end{table*}


In this section, we analyze our proposed \cref{alg:switching_gd_softmax_fedpartial} and provide a comprehensive theoretical analysis. We detail the computational complexity and provide guidance for the selection of the softmax hyperparameter $\alpha$.
First, we begin with the analysis of a simplified version of the proposed algorithm under full participation ($m=n$) and a single local update per global round ($E=1$) (see \Cref{alg:switching_gd_softmax} in Appendix~\ref{sup:algo}). Our analysis then proceeds to the more general full participation case ($m=n$) with multiple local updates ($E \geq 1$) per global round (see \Cref{alg:switching_gd_softmax_fedfull} in Appendix~\ref{sup:algo}). Finally, we analyze the general scenario of \Cref{alg:switching_gd_softmax_fedpartial} with partial participation ($m \le n$) and multiple local updates ($E \ge 1$). 
We formally state our assumptions.


\begin{assumption}[Convexity of functions]\label{assumption:convexity}
    \(f_i(\bw)\) and \(g_i(\bw)\) are convex with respect to \(\bw\)
    for all \(i\in \calI, \bw \in \Theta\).
\end{assumption}

\begin{assumption}[Diameter of the parameter space]\label{assumption:diameter_parameter_space}
    The diameter of the parameter space \(\Theta\) is bounded by \(D\), i.e., \(\|\bw_1-\bw_2\| \leq D\) for all \(\bw_1, \bw_2 \in \Theta\).
\end{assumption}

\begin{assumption}[Lipschitz Continuity]\label{assumption:lipschitz_continuity}
    All components of \(\bff\) and \(\bg\) are \(L\)-Lipschitz continuous, i.e., there exist constants \(L>0\) such that
    \begin{equation*}
        \begin{aligned}
        |f_i(\bw_1) - f_i(\bw_2)| \leq L \|\bw_1-\bw_2\|, \\
        |g_i(\bw_1) - g_i(\bw_2)| \leq L \|\bw_1-\bw_2\|,
        \end{aligned}
    \end{equation*}
    for all \(\bw_1, \bw_2 \in \Theta\) and \(i \in \mathcal{I}\).
\end{assumption}

In the following assumptions, 
we state the sub-guassianity of noise for the approximation of function values and the stochastic gradients, and thereby establish high probablity ganrantees with these assumptions.

\begin{assumption}[Sub-Gaussianity of Stochastic Estimates]\label{assumption:subgaussianity_stochastic_estimates}
    All components of the stochastic estimates \(\bff(\bw, \zeta), \bg(\bw, \zeta)\) are sub-Gaussian with variance proxy \(\sigma_\zeta^2\) such that
    \(\E_{\zeta \sim \calD_i}[f_i(\bw, \zeta)] = f_i(\bw), \E_{\zeta \sim \calD_i}[g_i(\bw, \zeta)] = g_i(\bw)\) and
    \begin{eqnarray*}
        \E_{\zeta \sim \calD_i}\left[\exp([f_i(\bw, \zeta)-f_i(\bw)]^2/\sigma_\zeta^2)\right] \leq 2, \\
        \E_{\zeta \sim \calD_i}\left[\exp([g_i(\bw, \zeta)-g_i(\bw)]^2/\sigma_\zeta^2)\right] \leq 2,
    \end{eqnarray*}
    for all \(i\in \calI\) and \(\bw \in \Theta\).
\end{assumption}

\begin{assumption}[Sub-Gaussianity of Stochastic Gradients]\label{assumption:subgaussianity_stochastic_gradients}
    The stochastic gradients are sub-Gaussian with variance proxy \(\sigma_g^2\) such that 
    \(\E_{\zeta \sim \calD_i}[\nabla f_i(\bw, \zeta)] = \nabla f_i(\bw), \E_{\zeta \sim \calD_i}[\nabla g_i(\bw, \zeta)] = \nabla g_i(\bw)\) and
    \begin{eqnarray*}
        \E_{\zeta \sim \calD_i}\left[\exp(\|\nabla f_i(\bw, \zeta)-\nabla f_i(\bw)\|^2/\sigma_g^2)\right] \leq 2, \\
        \E_{\zeta \sim \calD_i}\left[\exp(\|\nabla g_i(\bw, \zeta)-\nabla g_i(\bw)\|^2/\sigma_g^2)\right] \leq 2,
    \end{eqnarray*}
    for all \(i\in \calI\) and \(\bw \in \Theta\), where 
    \(\nabla f_i(\bw) \in \partial f_i(\bw)\), \(\nabla g_i(\bw) \in \partial g_i(\bw)\) 
    are subgradients of \(f_i(\bw), g_i(\bw)\).
\end{assumption}

\paragraph{Full participation and single local update.} 

We begin by analyzing a simplified case with full participation ($m=n$) and a single local update ($E=1$). Here, the "effective variance" $\bar{\sigma}_g^2 := \tfrac{\sigma_g^2/B_g}{L^2 E}$, representing the relative gradient estimation error, simplifies to $\tfrac{\sigma_g^2/B_g}{L^2}$.

\begin{myframe}
\begin{main theorem}[Convergence Guarantee of~\cref{alg:switching_gd_softmax} (special case of ~\cref{alg:switching_gd_softmax_fedpartial} with $m=n,E=1$)]\label{theorem:convergence_gd_softmax_lipschitz}
    Consider the optimization function \(F\) and the constraint function \(G\) as defined in \cref{eq:opt_constrained},
    and the optimal solution \(\mathbf{w}^*\) defined in \cref{eq:opt_sol}.
    Suppose~\Cref{assumption:convexity,assumption:lipschitz_continuity,assumption:diameter_parameter_space,assumption:subgaussianity_stochastic_estimates,assumption:subgaussianity_stochastic_gradients} hold.
    Consider running \cref{alg:switching_gd_softmax}, 
    with step size \(\eta=\gamma=\frac{D}{2L\sqrt{K}}\), 
    tolerance \(\epsilon=\epsilon' + 4\sigma_\zeta\sqrt{\frac{2\ln (12K n/\delta) }{B_\zeta}}\)
    and \(\epsilon' = \frac{2DL}{\sqrt{K}}\left[1 + \bar{\sigma}_g^2 \left(3 + \frac{8\ln\frac{8}{\delta}}{K}\right)
    + \bar{\sigma}_g \sqrt{8\ln \frac{8}{\delta}} \right]\), and threshold \(\lambda = \frac{\epsilon}{2}\), softmax hyperparameter \(\alpha\)
    satisfying \(\alpha \geq \frac{2\ln n}{\epsilon'}\), 
    where \(\bar{\sigma}_g^2 := \frac{\sigma_g^2/B_g}{L^2}\).
    Then, \cref{alg:switching_gd_softmax} finds a solution \(\overline{\mathbf{w}}_K\) with probability at least \(1 - \delta\) for some \(\delta \in (0, 1)\) such that:
    \[
    F(\overline{\mathbf{w}}_K) - F(\mathbf{w}^*) \leq \epsilon,\quad G(\overline{\mathbf{w}}_K) \leq \epsilon.
    \]
\end{main theorem}
\end{myframe}

\begin{remark}\label{remark:oracle_complexity}
    As established in Equation~(2.39) and Theorem~11 of \citet{lan2020algorithmsstochastic}, their method requires an iteration complexity of $K = \mathcal{O}(\epsilon^{-2}\log^2\frac{1}{\delta})$ and a per-iteration sample size of $B_\zeta = \mathcal{O}(\epsilon^{-2}\log^2\frac{1}{\delta})$ for evaluating function values. (Note that \citet{lan2020algorithmsstochastic} use $N$ for the number of iterations and $J$ for the sample size; see page~5 of their work.) Consequently, their total oracle complexity is $\mathcal{O}(\epsilon^{-4}\log^4\frac{1}{\delta})$.
    In contrast, for our method, the unified tolerance $\epsilon$ comprises two primary error sources: the optimization error and the estimation error. To bound these errors and achieve an $\epsilon$-accurate solution with high probability $1-\delta$, our algorithm requires an iteration complexity of $K = \mathcal{O}(\epsilon^{-2} \log\frac{1}{\delta})$ and a sample complexity of $B_\zeta = \mathcal{O}(\epsilon^{-2}\log\frac{1}{\epsilon\delta})$. Multiplying these yields our total oracle complexity of $\mathcal{O}(\epsilon^{-4}\log\frac{1}{\delta}\log\frac{1}{\epsilon\delta})$. This strictly improves the dependency on the high-probability parameters compared to prior work; see \Cref{tab:oracle_complexity} for a summary.
\end{remark}
\begin{remark}\label{remark:alpha_bound}
    While the problem formulation in \citet{wang2023task} is $\min_{\mathbf{w}} \max_i [\mathbf{f}(\mathbf{w})]_i = \min_{\mathbf{w}} \max_i f_i(\mathbf{w})$, their approach focuses on unconstrained optimization. Therefore, they did not introduce a switching gradient method—which uses either the gradients of the objectives $\nabla \mathbf{f}$ or the gradients of the constraints $\nabla \mathbf{g}$ to update $\mathbf{w}$—that is suitable for~\eqref{eq:opt_constrained}.
    Furthermore, \citet{wang2023task} is in a centralized, deterministic setting and require the boundedness assumption $0 < f_i(\mathbf{w}) < B$ to establish their lower bound $\alpha \gtrsim \frac{1}{\epsilon'}(\ln n + \ln \frac{1}{\epsilon'} + \ln B)$. In contrast, we address a complicated constrained optimization setting with a switching gradient method and eliminate this boundedness requirement on $f_i$. Our analysis provides high-probability convergence guarantees for a federated, stochastic setting and establishes a tighter lower bound of $\alpha \gtrsim \frac{\ln n}{\epsilon'}$. This strictly improves upon prior work, as our bound depends only on the optimization error $\epsilon'$ and the number of clients $n$.
\end{remark}

\paragraph{Full participation and multiple local updates.} 
We then generalize our analysis to multiple local updates ($E \geq 1$).

\begin{myframe}
\begin{main theorem}[Convergence Guarantee of~\Cref{alg:switching_gd_softmax_fedfull} (special case of~\Cref{alg:switching_gd_softmax_fedpartial} with $m=n, E\geq 1$)]\label{theorem:convergence_gd_softmax_fedfull_lipschitz}
    Consider the optimization function \(F\) and the constraint function \(G\) as defined in \cref{eq:opt_constrained},
    and the optimal solution \(\mathbf{w}^*\) defined in \cref{eq:opt_sol}.
    Suppose~\Cref{assumption:convexity,assumption:lipschitz_continuity,assumption:diameter_parameter_space,assumption:subgaussianity_stochastic_estimates,assumption:subgaussianity_stochastic_gradients} hold.
    Consider running \cref{alg:switching_gd_softmax_fedfull} with $m=n$, global step size \(\eta=\frac{D}{L\sqrt{8K}}\), 
    local step size \(\gamma=\frac{D}{LE\sqrt{8K}}\), 
    tolerance \(\epsilon=\epsilon' + 4 \sigma_\zeta\sqrt{\frac{2\ln (12K n/\delta) }{B_\zeta}}\) 
    and \(\epsilon' = \frac{DL}{\sqrt{K/32}}\left[1 + 2\bar{\sigma}_g^2 (3 + \frac{8\ln\frac{8}{\delta}}{K})
    + \bar{\sigma}_g \sqrt{8\ln \frac{8}{\delta}}(1+ \frac{E}{\sqrt{6K}}) \right]\)
    and threshold \(\lambda = \frac{\epsilon}{2}\),
    softmax hyperparameter \(\alpha\)
    satisfying \(\alpha \geq \frac{2\ln n}{\epsilon'}\), 
    where \(\bar{\sigma}_g^2 := \frac{\sigma_g^2/B_g}{L^2 E}\).
    Then, \cref{alg:switching_gd_softmax_fedfull} finds a solution \(\overline{\mathbf{w}}_K\) with probability at least \(1 - \delta\) for some \(\delta \in (0, 1)\) such that:
    \[
    F(\overline{\mathbf{w}}_K) - F(\mathbf{w}^*) \leq \epsilon,\quad G(\overline{\mathbf{w}}_K) \leq \epsilon.
    \]
\end{main theorem}
\end{myframe}

\begin{remark}
In the more general scenario where the number of local updates $E$ is greater than $1$, we can still achieve an optimization error of $\epsilon' \asymp \frac{DL}{\sqrt{K}}(1+\bar{\sigma}_g^2)$ provided that $E \lesssim \sqrt{K}$. In particular, when the number of global iterations is $K = \mathcal{O}(\epsilon^{-2})$, we select the number of local updates to be $E = \mathcal{O}(\epsilon^{-1})$. To ensure that our optimization error maintains the classical $\mathcal{O}(DL/\sqrt{K})$ convergence rate, the ``effective variance'' $\bar{\sigma}_g^2 := \frac{\sigma_g^2 / B_g}{E L^2}$ must be relatively small ($\bar{\sigma}_g^2 \lesssim 1$). This requirement implies that the total stochastic gradient complexity per global round, $B_g \cdot E$, must be larger than the inverse square of the Signal-to-Noise Ratio (SNR), namely $B_g \cdot E \gtrsim (L / \sigma_g)^{-2}$.
\end{remark}

\paragraph{Partial participation and multiple local update.} 
The goal of parameter updates under partial participation is to generalize to unseen clients using information from a sampled subset. This generalization is feasible only if the clients share certain structural proximity. To formalize this, we assume the function values from clients are concentrated near their maximum. To quantify this concentration and establish high-probability guarantees, we introduce a key concept from probability theory.

\begin{definition}[Stochastic Superiority via First-Order Stochastic Dominance (FSD), see \citet{shaked2007stochastic}]
Let $X$ and $Y$ be two random variables. We say $X$ is \textit{stochastically superior} to $Y$ in the sense of First-Order Stochastic Dominance (FSD), denoted by $X \succeq_{st} Y$ or $Y \preceq_{st} X$, if:
\begin{equation*}
    \Pr(X \geq t) \geq \Pr(Y \geq t) \quad \text{for all } t \in \mathbb{R}.
\end{equation*}
\end{definition}

\begin{remark}
Consider a discrete random variable $X$ taking values from the set $\bx = (x_1, \ldots, x_n)$ with equal probability $1/n$, and let $U \sim \text{Unif}[0, \sigma]$ for some $\sigma > 0$. We aim to stochastically upper bound the relative difference $D := \max \bx - X$ by $U$, such that $D \preceq_{st} U$, which is defined as, $\Pr(\max \bx - X \geq t) \leq \Pr(U \geq t), \ \forall t \in \mathbb{R}.$
\end{remark}

\begin{assumption}[Uniformly Bounded Relative Gap]\label{assumption:uniform_cdf_bound}
Let \(i\) be an index chosen uniformly at random from \(\calI\),
and let the relative differences \(D_{\bff(\bw)}:=F(\bw)-f_i(\bw), 
D_{\bg(\bw)}:= G(\bw)-g_i(\bw)\). We assume that 
\[D_{\bff(\bw)}\preceq_{st} U,\quad D_{\bg(\bw)}\preceq_{st} U\] 
for \(U\sim \text{Unif}[0, \sigma]\) with some \(\sigma >0\) and any \(\bw \in \Theta\).
\end{assumption}

\begin{remark}
By applying the aforementioned assumptions to the relative gaps $D_{\mathbf{f}(\mathbf{w})}$ and $D_{\mathbf{g}(\mathbf{w})}$, we establish that a uniform random variable $U$ is \textit{stochastically superior} to these gaps. Consequently, $D_{\mathbf{f}(\mathbf{w})}$ and $D_{\mathbf{g}(\mathbf{w})}$ are uniformly upper-bounded in a probabilistic sense. By leveraging this characterization of stochastic superiority, we can extend our analysis from full participation to partial participation, thereby establishing high-probability convergence guarantees (\Cref{theorem:convergence_gd_softmax_fedpartial_lipschitz}) for our proposed \Cref{alg:switching_gd_softmax_fedpartial}.
Standard assumptions (e.g., bounded variance and sub-Gaussianity) are insufficient for maximum objectives because the expected gap grows unboundedly with the number of clients. However, in a simple setting, a valid stochastic upper bound $U$ exists whenever the error distribution has bounded support; see \Cref{supsub:limitations} for more details on why this
stochastic control is fundamental in large-scale federated learning. Accordingly, we can relax \Cref{assumption:uniform_cdf_bound} so that the relative gaps can be dominated by \emph{any} nonnegative random variable $U$, rather than requiring $U \sim \mathrm{Unif}[0,\sigma]$. This yields generalized high-probability bounds for a broad class of tail behaviors. We can further generalize \Cref{theorem:convergence_gd_softmax_fedpartial_lipschitz} by replacing the original sampling-error bound with an exact generalized bound expressed through quantile functions; see \Cref{supsub:general_bound} for derivations and examples.
\end{remark}





\begin{myframe}
\begin{main theorem}[Convergence Guarantee of~\cref{alg:switching_gd_softmax_fedpartial}]\label{theorem:convergence_gd_softmax_fedpartial_lipschitz}
    Consider the optimization function \(F\) and the constraint function \(G\) as defined in \cref{eq:opt_constrained},
    and the optimal solution \(\mathbf{w}^*\) defined in \cref{eq:opt_sol}.
    Suppose~\Cref{assumption:convexity,assumption:lipschitz_continuity,assumption:diameter_parameter_space,assumption:subgaussianity_stochastic_estimates,assumption:subgaussianity_stochastic_gradients,assumption:uniform_cdf_bound} hold.
    Consider running \cref{alg:switching_gd_softmax_fedpartial} with
    global step size \(\eta=\frac{D}{L\sqrt{8K}}\), local step size \(\gamma=\frac{D}{LE\sqrt{8K}}\), 
    tolerance \(\epsilon=\epsilon' + 4 \sigma_\zeta\sqrt{\frac{2\ln (24K m/\delta) }{B_\zeta}}
    + \frac{4 \sigma}{|\ln\left(1-r\right)|n}\ln \frac{32}{\delta}\) 
    and \(\epsilon' = \frac{DL}{\sqrt{K/32}}\left[1 + 2\bar{\sigma}_g^2 (3 + \frac{8\ln\frac{8}{\delta}}{K})
    + \bar{\sigma}_g \sqrt{8\ln \frac{8}{\delta}}(1+ \frac{E}{\sqrt{6K}}) \right]\)
    and threshold \(\lambda = \frac{\epsilon}{2}\),
    softmax hyperparameter \(\alpha\)
    satisfying \(\alpha \geq \frac{2\ln m}{\epsilon'}\), 
    where \(\bar{\sigma}_g^2 := \frac{\sigma_g^2/B_g}{L^2 E}, r := \frac{m}{n}\),
    and \(\kappa := \frac{|\calS|}{K} \in(0, 1]\) is a constraint-satisfied ratio.
    Then, \cref{alg:switching_gd_softmax_fedpartial} finds a solution \(\overline{\mathbf{w}}_K\) with probability at least \(1 - \delta\) for some \(\delta \in (0, 1)\) such that:
    \[
        F(\overline{\mathbf{w}}_K) - F(\mathbf{w}^*) \leq \epsilon,\quad
        G(\overline{\mathbf{w}}_K) \leq \epsilon + \frac{4 \sigma \ln \frac{1}{2\kappa}}{|\ln\left(1-r\right)|n}.
    \]
\end{main theorem}
\end{myframe}

\begin{remark}
Compared to full participation, 
partial participation adds a sampling error of $\mathcal{O}(\frac{\sigma}{|\ln(1-r)|n})$, 
where $r := m/n$ is the participation ratio. 
As $r \to 1$, $|\ln(1-r)| \to \infty$, 
causing this term to vanish and recovering full-participation guarantees. 
For small $r$, 
then $|\ln(1-r)| \approx r$ implies $|\ln(1-r)|n \approx m$, 
yielding a sampling error of 
$\mathcal{O}(\sigma/m)$.
Our method requires transmitting only $\mathcal{O}(1)$ scalar function estimates per client, not $\mathcal{O}(d)$ model vectors, easily piggybacking on standard updates. In practical partial-participation settings, the masked softmax evaluates only over the $m$ selected clients ($m \leq n$), reducing per-round arithmetic cost to $\mathcal{O}(m)$. Consequently, the scalability factor is the statistical tradeoff between participation rate $r:=m/n$ and worst-client coverage, explicitly bounded by our theoretical sampling error $\propto\frac{1}{|\log(1-r)|n}$.
\end{remark}

\begin{summarybox}\label{box:err_decompose}
\textbf{Error Decomposition}

The optimality gap / feasibility tolerance $\epsilon$ is composed of three distinct error terms:
\begin{equation*}
    \epsilon \asymp \underbrace{\frac{DL}{\sqrt{K}}(1+ \bar{\sigma}_g^2)}_{\text{optimization error } \epsilon'} + \underbrace{\vphantom{\frac{DL}{\sqrt{K}}}\frac{\sigma_\zeta}{\sqrt{B_\zeta}}}_{\text{estimation error}} + \underbrace{\vphantom{\frac{DL}{\sqrt{K}}}\frac{\sigma}{|\ln(1-r)|n}}_{\text{sampling error}}
\end{equation*}
where $r := m/n$ is the participation ratio, and the remaining parameters are defined as:
\begin{itemize}
    \item $\bar{\sigma}_g^2 := \frac{\sigma_g^2/B_g}{L^2 E}$ is \textbf{effective variance} in stochastic gradient approximation during optimization process, scaled by batch size $B_g$, number of local updates $E$, and $L$-Lipschitz constant.
    \item $\sigma_\zeta$ is associated with the \textbf{estimation} of expectation function values using a finite batch of $B_\zeta$ samples.
    \item $\sigma$ quantifies the client \textbf{sampling noise} arising from the heterogeneity of function values across the population of $n$ clients.
\end{itemize}

This bound is achieved with a smaller softmax hyperparameter $\alpha \gtrsim \frac{\ln m}{\epsilon'}$, compared to $\frac{\ln n}{\epsilon'}$ required in the full participation case ($m=n$).
\end{summarybox}

\section{Experiments}\label{sec:experiments}
In this section, we evaluate our algorithm on classical stochastic constrained optimization problems across three datasets. We begin with the classical convex setting of Neyman-Pearson (NP) classification on the breast cancer dataset \citep{breast_cancer_wisconsin_dataset}. We then extend our numerical analysis to the non-convex setting of fair classification using deep neural networks. 
To demonstrate concrete use cases and broader evaluation beyond tabular datasets, we introduce a federated safe reinforcement learning experiment using a heterogeneous Constrained Markov Decision Process (CMDP) CartPole environment.
Detailed experimental setups and hyperparameters are presented in Appendix~\ref{sup:exp},
and the code is available at \url{https://github.com/sangbinM/SoftmaxSGM}.

\textbf{Neyman Pearson Classification}
NP Classification involves a constrained optimization problem where the objective is to minimize the empirical loss on the majority class while bounding the minority class loss below a specific threshold. Translating this to the formulation in \cref{eq:opt_constrained}, we define the objective and constraint for client $i$ as $f_i(\mathbf{w}):=\frac{1}{m_{i,0}}\sum_{x\in\mathcal{D}_i^{(0)}}\phi(\mathbf{w};(x,0))$ and $g_i(\mathbf{w}):=\frac{1}{m_{i,1}}\sum_{x\in\mathcal{D}_i^{(1)}}\phi(\mathbf{w};(x,1))$, respectively. Here, $\mathcal{D}_i^{(0)}$ and $\mathcal{D}_i^{(1)}$ denote the local datasets for class-0 and class-1, with respective sizes $m_{i,0}$ and $m_{i,1}$, and $\phi$ represents the binary logistic loss. Across experiments, solid lines represent mean results across five seeds; shaded regions indicate variance.
\begin{figure}[!t]
    \centering
    \includegraphics[width=\linewidth]{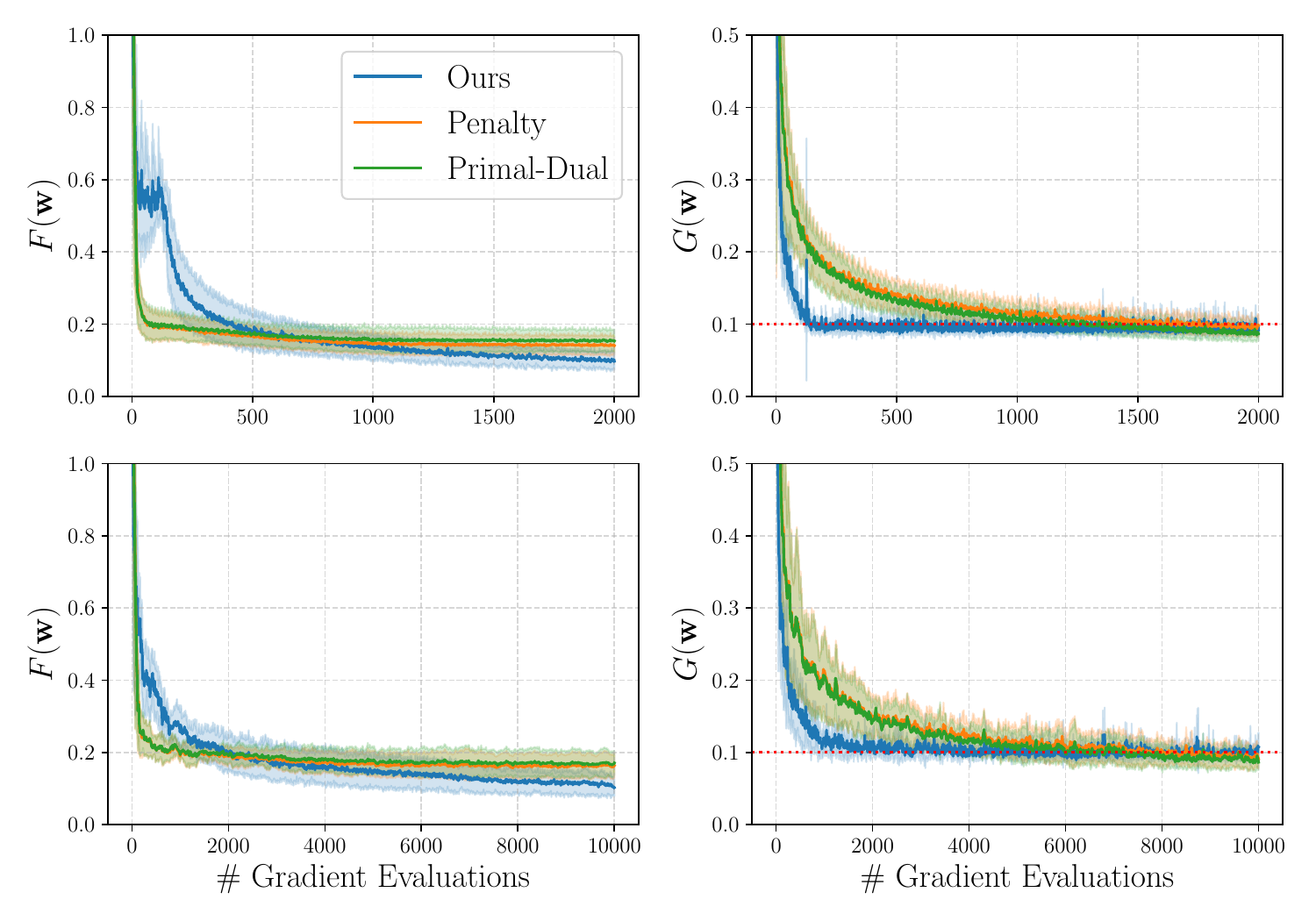}
    \caption{\textbf{NP classification.} Objective $F(\mathbf{w}_k)$ and constraint $G(\mathbf{w}_k)$ vs. gradient evaluations. Comparisons against penalty and primal-dual baselines under full participation ($E=1, m=n$; top) and partial participation ($E=5, \frac{m}{n}=0.5$; bottom). Red dashed line: tolerance ($\epsilon$).}
    \label{fig:NP_classification_main}
\end{figure}
As demonstrated in \Cref{fig:NP_classification_main}, our algorithm rapidly achieves constraint feasibility ($G(\mathbf{w}) \le \epsilon$) while consistently minimizing the worst-case objective $F(\mathbf{w})$ in both settings. Furthermore, compared to the penalty-based and primal-dual baselines, our approach secures a lower objective value for a comparable level of constraint satisfaction.
\begin{figure}[!t]
    \centering
    \includegraphics[width=\linewidth]{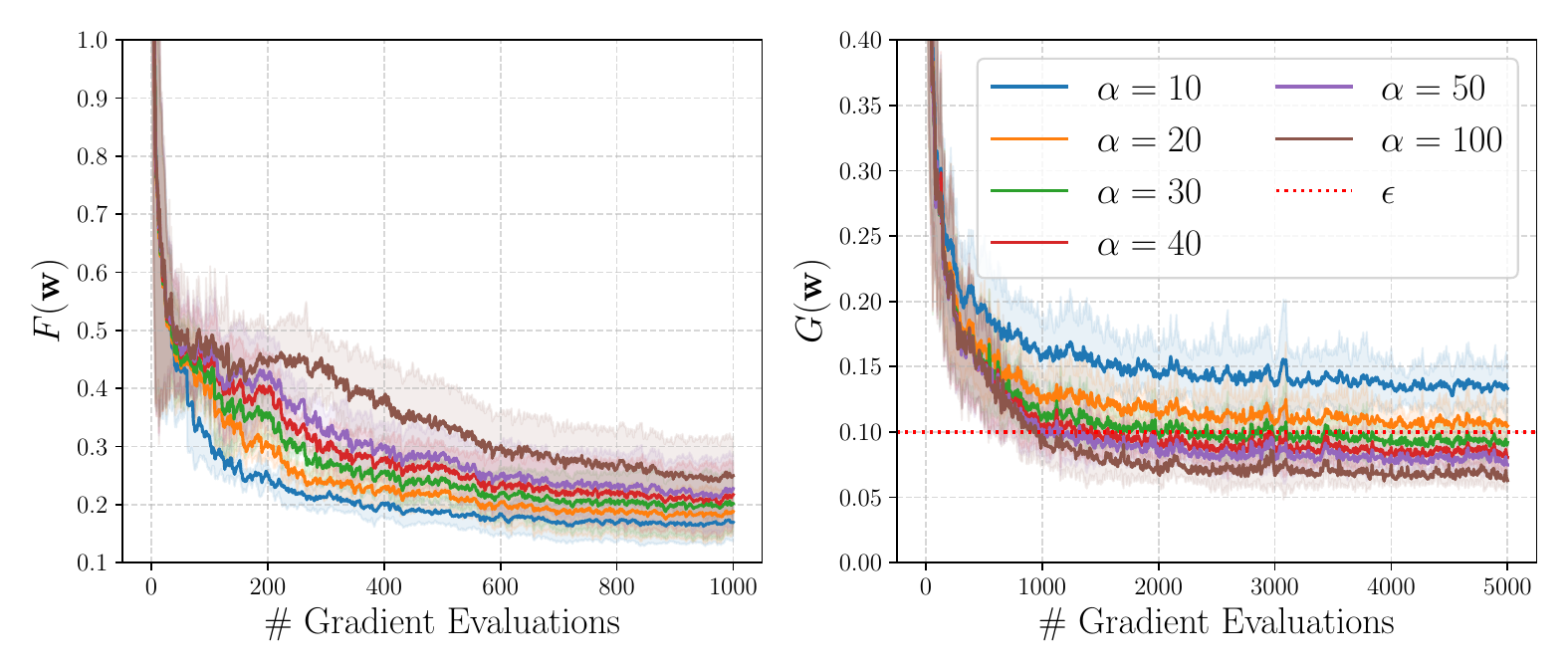}
    \caption{$\alpha$-\textbf{sensitivity.} Impact of temperature $\alpha$. High $\alpha$ approximates the hard $\max$ operator, while low $\alpha$ smooths the objective toward a simple average.}
    \label{fig:alpha_sensitivity}
\end{figure}
We further validate the theoretical efficacy of our approach by varying the softmax parameter ($\alpha$), as illustrated in \Cref{fig:alpha_sensitivity}. As $\alpha\rightarrow0$, the softmax approximation relaxes into an average over the clients' local values; conversely, as $\alpha\rightarrow\infty$, it recovers the discrete maximum over clients. Consequently, with a lower $\alpha$, the algorithm assigns more uniform weights across clients, easing the satisfaction of the smoothed constraint and prioritizing objective minimization. However, this uniformity inherently fails to enforce strict feasibility with respect to the true worst-case constraint.

\textbf{Fair Classification.}
We formulate the fair classification task as the minimization of the binary cross-entropy (BCE) loss subject to a demographic parity constraint. Mapping this to the constrained minimax formulation in \cref{eq:opt_constrained}, each client evaluates the local objective and constraint defined respectively as $f_i(\mathbf{w}):=\frac{1}{m_i}\sum_{(x,y)\in\mathcal{D}_i}\ell_\text{BCE}(\pi(x;\mathbf{w}),y)$, $g_i(\mathbf{w}):=\left|\frac{1}{m_{i,p}}\sum_{x\in\mathcal{D}_{i,p}}\pi(x;\mathbf{w})-\frac{1}{m_{i,u}}\sum_{x\in\mathcal{D}_{i,u}}\pi(x;\mathbf{w})\right|$, where $\pi(x;\mathbf{w})$ denotes the model's positive prediction probability. The sets $\mathcal{D}_{i,p}$ and $\mathcal{D}_{i,u}$ represent the protected and unprotected subgroups on client $i$, with corresponding cardinalities $m_{i,p}$ and $m_{i,u}$. For this task, we employ a deep neural network, which renders the optimization landscape highly non-convex and non-smooth. We conduct experiments using the Adult income dataset \citep{kohavi1996adult}.
\begin{figure}[!t]
    \centering
    \includegraphics[width=\linewidth]{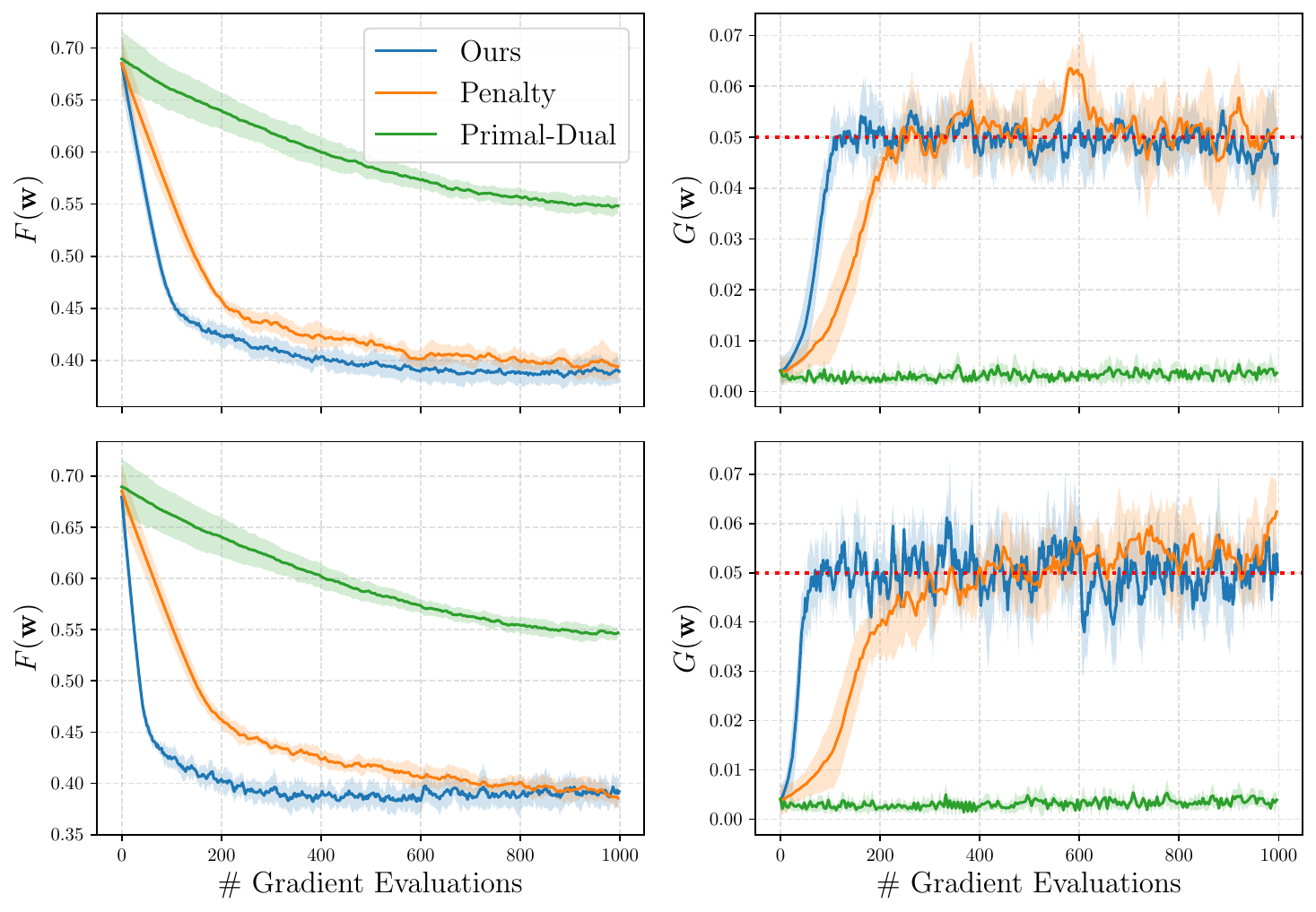}
    \caption{\textbf{Fair classification.} 
    Comparisons against penalty and primal-dual baselines. 
    Top: full participation ($E=1, m=n$). 
    Bottom: partial participation ($E=2, \frac{m}{n}=0.5$).}
    \label{fig:fair_classification}
\end{figure}
We compare our method against the penalty-based and primal-dual baselines. As illustrated in \Cref{fig:fair_classification}, our algorithm demonstrates accelerated convergence with respect to cumulative gradient evaluations. Furthermore, while penalty-based and primal-dual methods require meticulous tuning of the penalty parameter and dual step size to ensure stability and feasibility, our approach achieves highly competitive performance using a static, default value of $\alpha=1$.

\textbf{Federated Safe RL (Constrained MDP).}
We build on TRPO~\citep{schulman2015trust} and compare our \textit{Softmax-Weighted Switching Gradient} (Softmax SGM) method against Parallel CRPO~\citep{xu2021crpo}, which solves the average-case constrained problem. Softmax SGM controls conservatism via the softmax parameter $\alpha$; a larger $\alpha$ strictly enforces robust, worst-case optimization.
As shown in \Cref{fig:federated_safe_rl}, Softmax SGM outperforms Parallel CRPO in return maximization and constraint satisfaction. By targeting the maximum value problem, our approach ensures the most vulnerable clients remain within their safety envelopes.

\begin{figure}[!t]
    \centering
    \includegraphics[width=\linewidth]{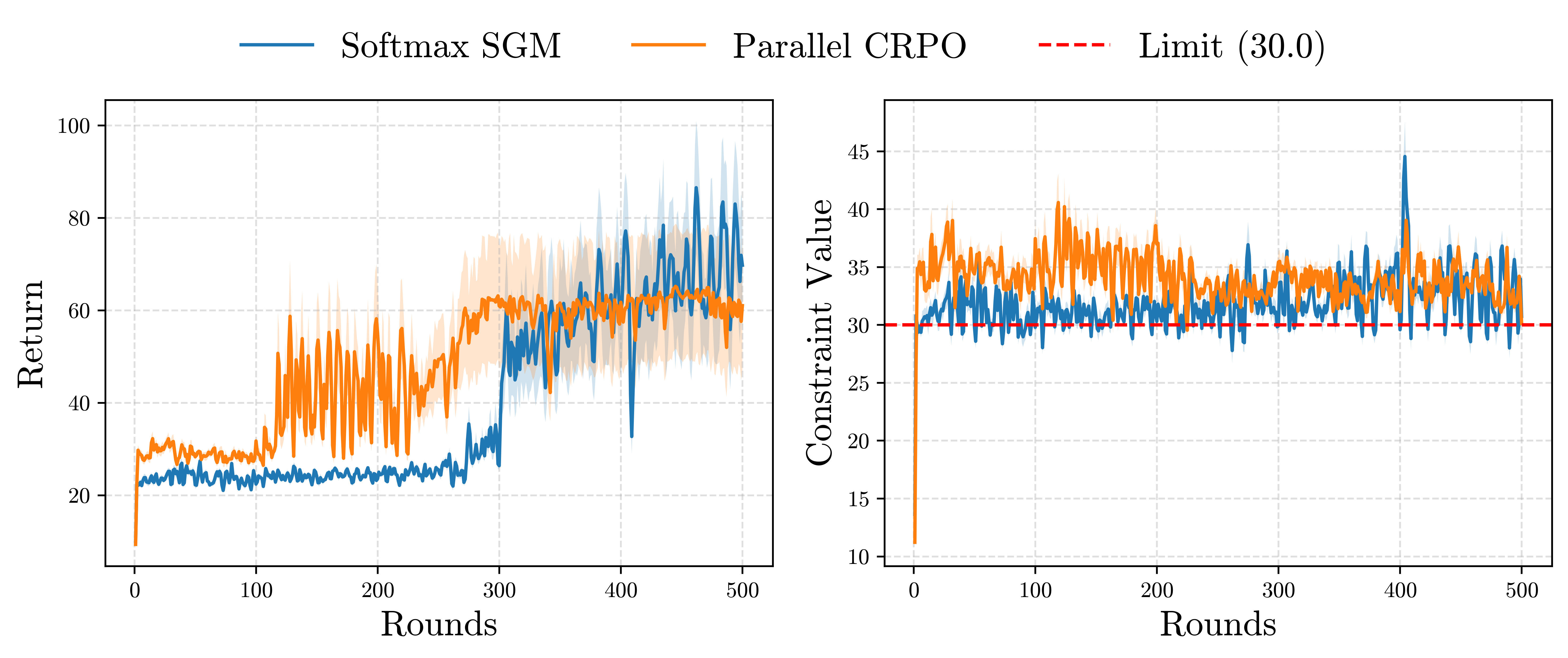}
    \caption{\textbf{Federated safe RL.} Worst-case return and constraint violations for Softmax SGM vs. Parallel CRPO.}
    \label{fig:federated_safe_rl}
\end{figure}

\section{Conclusion}\label{sec:conclusion}
In this paper, we introduce a novel primal-only first-order algorithm for solving constrained stochastic minimax optimization problems in federated environments. We theoretically establish that our method achieves the standard convergence rates without relying on explicit dual variables or strict functional boundedness assumptions. Our unified error decomposition successfully decouples optimization dynamics, stochastic estimation variance, and client sampling noise, offering guidelines for hyperparameter selection. Empirical results across Neyman-Pearson, fair classification, and federated safe reinforcement learning confirm the method's stability and practical edge. Future work may explore incorporating momentum-based variance reduction and extending this framework to weakly convex objectives~\citep{huang2023oracle}.


\begin{acknowledgements} 

This work was supported in part by NSF DMS-2502560 and CNS-2313109. The authors would like to thank the anonymous reviewers for their valuable comments and advice.
\end{acknowledgements}
\bibliography{ref_softmax_switchgd}

\onecolumn 
\newpage

\begin{center}
    \rule{\textwidth}{3pt} 
    \vspace{0.05cm} 

    {\Large \bfseries 
    Supplementary Materials: \\
    First-Order Softmax Weighted Switching Gradient Method for \\ 
    Distributed Stochastic Minimax Optimization with Stochastic Constraints \par}

    \vspace{0.3cm} 
    \rule{\textwidth}{1.2pt} 
    \vspace{0.4cm} 
\end{center}

\appendix
\part{}
\parttoc

\vspace{0.5cm}
\section*{Appendices}
We organize the Appendices as follows:
\begin{itemize}
    \item Appendix~\ref{sup:algo}: We provide foundational baseline algorithms that establish the core switching logic under full participation ($m=n$), covering both the single local update ($E=1$) and multiple local update ($E \ge 1$) regimes.
    \item Appendix~\ref{sup:lemma}: We present the technical lemmas required to establish the convergence results for our proposed framework.
    \item Appendix~\ref{sup:switch}: We provide the proof of \Cref{theorem:convergence_gd_softmax_lipschitz}, establishing the convergence guarantee for the foundational switching strategy (\Cref{alg:switching_gd_softmax}) with $E=1$ and full participation.
    \item Appendix~\ref{sup:fedfull}: We provide the proof of \Cref{theorem:convergence_gd_softmax_fedfull_lipschitz} for the convergence of \Cref{alg:switching_gd_softmax_fedfull} under the federated setting with multiple local updates ($E \ge 1$) and full participation ($m=n$).
    \item Appendix~\ref{sup:fedpart}: We provide the proof of \Cref{theorem:convergence_gd_softmax_fedpartial_lipschitz}, which characterizes the convergence of our primary algorithm, \Cref{alg:switching_gd_softmax_fedpartial}, in the general case of partial participation ($m \le n$).
    \item Appendix~\ref{sup:general}: We discuss the limitations of assumptions, provide a general upper bound for the sampling error under the relaxed Assumption~\ref{assumption:uniform_cdf_bound}, and elaborate on the threshold selection for the constraint criterion.
    \item Appendix~\ref{sup:tradeoff}: We provide a simple example to illustrate the trade-off between the adaptivity and the stability for the selection of the softmax hyperparameter $\alpha$.
    \item Appendix~\ref{sup:exp}: We provide the details of experimental settings and additional empirical results.
    \item Appendix~\ref{sup:related}: We discuss related work on the problems we are studying in this work.
\end{itemize}

\newpage
\section{Algorithms}\label{sup:algo}




\begin{multicols}{2}
\begin{algorithm}[H]
\caption{(\textit{Softmax-Weighted Switching Gradient})\\Full Participation ($m= n$),
(Local Update Steps \(E=1\))}
\label{alg:switching_gd_softmax}
\begin{algorithmic}[1] 
\STATE \textbf{Input:} Initial parameters $\mathbf{w}_0$, step size $\eta$, 
threshold $\lambda$, softmax hyperparameter $\alpha$
\STATE $\mathcal{S} \gets \emptyset$ \COMMENT{Initialize set of constraint-satisfied iterations}
\FOR{$k \in [K]$}
    \PreComment{Samples for Function Value Evaluation}
    \STATE $\bxi_k \equiv (\bxi_k^{(i)})_{i\in\calI},\bxi_k^{(i)} \equiv ([\bxi_k^{(i)}]_s)_{s=1}^{B_\zeta} \stackrel{\text{i.i.d.}}{\sim} \calD_i, \forall i \in \mathcal{I}$ 
    \PreComment{Samples for Gradient Evaluation}
    \STATE $\bzeta_k \equiv (\bzeta_k^{(i)})_{i\in\calI},\bzeta_k^{(i)} \equiv ([\bzeta_k^{(i)}]_s)_{s=1}^{B_g} \stackrel{\text{i.i.d.}}{\sim} \calD_i, \forall i \in \mathcal{I}$ 
    
    \STATE $\bp_k \gets \softmax(\alpha \bff(\mathbf{w}_k, \bxi_k))$ \COMMENT{Softmax weights}
    \STATE $\bq_k \gets \softmax(\alpha \bg(\mathbf{w}_k, \bxi_k))$ \COMMENT{See \Cref{eq:approx_function_values}}
    \IF{${\color{brown}G_k(\mathbf{w}_k, \bxi_k) \equiv\langle \bq_k, \bg(\mathbf{w}_k, \bxi_k) \rangle} 
    \leq {\color{royalBlue}\lambda}$}
        \STATE $\mathcal{S} \gets \mathcal{S} \cup \{k\}$ \COMMENT{Add iteration to set}
        \STATE $\bu_k \gets (\nabla \bff(\mathbf{w}_k, \bzeta_k))^\top \bp_k$ \COMMENT{Update direction}
    \ELSE
        \STATE $\bu_k \gets (\nabla \bg(\mathbf{w}_k, \bzeta_k))^\top \bq_k$ \COMMENT{See \Cref{eq:approx_gradients}}
    \ENDIF
    \STATE $\mathbf{w}_{k+1} \gets \mathbf{w}_k - \eta \mathbf{u}_k$ \COMMENT{Update parameters}
\ENDFOR
\STATE $\overline{\mathbf{w}}_K \gets \frac{1}{|\mathcal{S}|} \sum_{k \in \mathcal{S}} \mathbf{w}_k$ \COMMENT{Compute averaged solution}
\STATE \textbf{Output:} Optimized parameters $\overline{\mathbf{w}}_K$
\end{algorithmic}
\end{algorithm}



\begin{algorithm}[H]
\caption{(\textit{Softmax-Weighted Switching Gradient})\\Federated Learning with Full Participation ($m= n$)\\
(Local Update Steps \(E\geq1\))}
\label{alg:switching_gd_softmax_fedfull}
\begin{algorithmic}[1] 
\STATE \textbf{Input:} Initial parameters $\mathbf{w}_0$, global/local step size $(\eta,\gamma)$, 
threshold $\lambda$, softmax hyperparameter $\alpha$
\STATE $\mathcal{S} \gets \emptyset$ \COMMENT{Initialize set of constraint-satisfied iterations}
\FOR{global round $k \in [K]$}
    \STATE \textbf{broadcast} $\mathbf{w}_k$ to all clients $\calI$
    \FOR{each client $i \in \mathcal{I}$ \textbf{in parallel}}
        \PreComment{Samples for Function Value Evaluation}
        \STATE $\bxi_k^{(i)} \equiv ([\bxi_k^{(i)}]_s)_{s=1}^{B_\zeta} \stackrel{\text{i.i.d.}}{\sim} \calD_i$ 
        \STATE $[\bff(\mathbf{w}_k, \bxi_k)]_i \gets f_i(\mathbf{w}_k, \bxi_k^{(i)})$ \COMMENT{See \Cref{eq:approx_function_values}}
        \STATE $[\bg(\mathbf{w}_k, \bxi_k)]_i \gets g_i(\mathbf{w}_k, \bxi_k^{(i)})$ \COMMENT{See \Cref{eq:approx_function_values}}

    \ENDFOR
    \STATE \textbf{collect} $\bff(\mathbf{w}_k, \bxi_k), \bg(\mathbf{w}_k, \bxi_k)$ from all clients $\calI$\\
    \STATE $\bp_k \gets \softmax(\alpha \bff(\mathbf{w}_k, \bxi_k))$ 
    \STATE $\bq_k \gets \softmax(\alpha \bg(\mathbf{w}_k, \bxi_k))$ \COMMENT{Softmax weights}
    \IF{${\color{brown}G_k(\mathbf{w}_k, \bxi_k) \equiv\langle \bq_k, \bg(\mathbf{w}_k, \bxi_k) \rangle} 
    \leq {\color{royalBlue}\lambda}$}
        \STATE $\mathcal{S} \gets \mathcal{S} \cup \{k\}$ \COMMENT{Add iteration to set}
    \ENDIF
    \STATE \textbf{broadcast} $\1_k \equiv \1_{G_k(\mathbf{w}_k, \bxi_k) \leq \lambda}$ to all clients $\calI$
    \FOR{each client $i \in \mathcal{I}$ \textbf{in parallel}}
        \PreComment{\Cref{function:local_solver} to compute the update direction}
        \STATE $\bu_k^{(i)} \gets \texttt{LocalSolver}(\mathbf{w}_k, \1_k, \gamma; i)$
    \ENDFOR
    \STATE \textbf{collect} $\bu_k^{(i)}$ from each client $i\in \mathcal{I}$
    \STATE $\bu_k \gets \sum_{i \in \mathcal{I}} (\1_k [\bp_k]_i + (1-\1_k)[\bq_k]_i)\bu_k^{(i)}$
    \STATE $\mathbf{w}_{k+1} \gets \mathbf{w}_k - \eta \bu_k$ \COMMENT{Update parameters}
\ENDFOR
\STATE $\overline{\mathbf{w}}_K \gets \frac{1}{|\mathcal{S}|} \sum_{k \in \mathcal{S}} \mathbf{w}_k$ \COMMENT{Averaged solution}
\STATE \textbf{Output:} Optimized parameters $\overline{\mathbf{w}}_K$
\end{algorithmic}
\end{algorithm}

\end{multicols}

To provide a rigorous foundation for our proposed framework, we first analyze two baseline variants in the Appendix: \Cref{alg:switching_gd_softmax}, which establishes the core switching logic with single local update \(E=1\) under full participation \(m=n\), and \Cref{alg:switching_gd_softmax_fedfull}, which extends this logic to multiple local steps \(E\geq 1\) with full participation \(m=n\). These baseline analyses serve as the theoretical stepping stones for our primary theoretical contribution presented in the main text for  \Cref{alg:switching_gd_softmax_fedpartial}, which addresses the more general scenario of partial participation \(m\leq n\) with multiple local updates \(E\geq 1\).

\textbf{Softmax Weighting}

Softmax weighting is a standard technique derived from entropy-regularized optimization (see Lemma~4, p.~139 of \citet{nesterov2005smooth} and Section~4 of \citet{beck2012smoothing}) and was not first introduced by \citet{wang2023task}. Let $\mathbf{f}_k$ be the estimated function values at iteration $k$. The softmax weights $\mathbf{q}_k \in \Delta_n$ on the probability simplex are induced by the following entropy-regularized maximization:
$$
\mathbf{q}_k = \arg\max_{\mathbf{q} \in \Delta_n}\; \langle \mathbf{q}, \mathbf{f}_k \rangle + \frac{1}{\alpha} \mathsf{Ent}(\mathbf{q}).
$$
This formulation naturally interpolates between two extremes. As $\alpha \to \infty$, the entropy term $\frac{1}{\alpha}\mathsf{Ent}(\mathbf{q})$ vanishes, forcing the weights to pick the maximum component among $\mathbf{f}_k$, consistent with $\max_{1 \leq i \leq n} [\mathbf{f}_k]_i = \max_{\mathbf{q} \in \Delta_n} \langle \mathbf{q}, \mathbf{f}_k \rangle$ for worst-case analysis. Conversely, as $\alpha \to 0$, the Shannon entropy term dominates and reaches its maximum $\mathsf{Ent}(\mathbf{q}) \leq \mathsf{Ent}(\mathbf{1}/n) = \ln n$ at the uniform distribution $\mathbf{q} = \mathbf{1}/n$, representing the average case.
\newpage
\section{Lemmas Used in Proofs}\label{sup:lemma}
\begin{lemma}[Three-point Bregman Divergence Identity, see equation (4.1) on page 297 of~\citep{bubeck2015convex}]\label{lemma:three_point_bregman_divergence_identity}
    Let \(\psi\) be a convex function, then for Bregman divergence \(D_{\psi}[\bx||\bx'] := \psi(\bx) - \psi(\bx') - \langle \nabla \psi(\bx'), \bx - \bx'\rangle \geq 0\) with any \(\bx, \bx'\) in the domain of \(\psi\), 
    the following identity holds for any three points \(\bx, \bx', \hat{\bx}\) in the domain of \(\psi\):
    \[
     \langle \nabla \psi(\bx) - \nabla \psi(\hat{\bx}), \hat{\bx}-\bx' \rangle  = D_{\psi}[\bx'||\bx] - D_{\psi}[\bx'||\hat{\bx}] - D_{\psi}[\hat{\bx}||\bx] 
    \]
\end{lemma}

\begin{lemma}[Polarization Identity]\label{lemma:polarization}
    The update direction \(\mathbf{u}_k\) and the parameters \(\mathbf{w}_k, \mathbf{w}_{k+1}\) in the update rule \(\bw_{k+1} = \bw_k - \eta \mathbf{u}_k\) with step size \(\eta>0\) satisfies the following identity:
    \begin{equation*}
        \langle \mathbf{u}_k, \mathbf{w}_{k} - \mathbf{w}^\ast \rangle
         = \frac{1}{2\eta} \left( \eta^2\|\bu_{k}\|^2 + \|\mathbf{w}_k - \mathbf{w}^\ast\|^2 - \|\mathbf{w}_{k+1} - \mathbf{w}^\ast\|^2 \right) 
    \end{equation*}
\end{lemma}


\begin{lemma}[Properties of Weighted Functions]\label{lemma:properties_of_weighted_functions}
    \(F(\bw)\), \(G(\bw)\) and \(F(\bw; \calI_k), G(\bw, \calI_k)\) are:
    \begin{itemize}
        \item Convex on \(\Theta\), if all components of \(\bff(\bw)\) and \(\bg(\bw)\) are convex on \(\Theta\).
        \item \(L\)-Lipschitz continuous on \(\Theta\), if all components of \(\bff(\bw)\) and \(\bg(\bw)\) are \(L\)-Lipschitz continuous on \(\Theta\).
    \end{itemize}
\end{lemma}


\begin{lemma}[Properties of Softmax Mean]\label{lemma:softmax_mean}
    Let the softmax mean be defined as
    \[
    m(\bx, \alpha) := \langle \softmax(\alpha \bx), \bx \rangle
    = \frac{\bx^\top \exp(\alpha \bx)}{\bone^\top \exp(\alpha \bx)}
    \]
    where \(\bx=([\bx]_i)_{i\in\calI} = (x_1, \ldots,x_{n}) \in \mathbb{R}^{n}\) and \(\alpha \geq 0\).
    The softmax mean \(m(\bx, \alpha)\) with hyperparameter \(\alpha\) satisfies:
    \begin{eqnarray*}
        m(\bx + C\bone, \alpha) &=&  m(\bx, \alpha) + C, \quad \forall C\in \mathbb{R}\\
        m(C \bx, \alpha) &=& C \cdot m(\bx, C \alpha) , \quad \forall C \in \mathbb{R}_{\geq 0}\\
        m(\bx, \alpha') &\geq& m(\bx, \alpha),\quad \text{for } \alpha' \geq \alpha
    \end{eqnarray*}
    Moreover, the following inequality holds:
    \[
    0 \leq \max_{i\in \calI} [\bx]_i - m(\bx, \alpha)
    = \lim_{\alpha \to \infty} m(\bx, \alpha) - m(\bx, \alpha) < k, 
    \quad \text{for } \alpha \geq \underline{\alpha} := \frac{\ln n}{k}, k>0
    \]
\end{lemma}

\begin{lemma}[Properties of Masked Softmax Mean]\label{lemma:masked_softmax_mean}
    Let the masked softmax mean with a nonempty subset \(\calI' \subseteq \calI\) be
    \[
    m_{\calI'}(\bx, \alpha) := \langle \softmax_{\calI'}(\alpha \bx), \bx \rangle
    = \frac{\bx^\top (\bone_{\calI'}\odot \exp(\alpha \bx))}{\bone_{\calI'}^\top \exp(\alpha \bx)}
    = \frac{\bx_{\calI'}^\top \exp(\alpha \bx_{\calI'})}{\bone^\top \exp(\alpha \bx_{\calI'})} 
    = \langle \softmax(\alpha \bx_{\calI'}), \bx_{\calI'} \rangle = m(\bx_{\calI'}, \alpha)
    \]
    where \(\bx=([\bx]_i)_{i\in\calI} = (x_1, \ldots,x_{n}) \in \mathbb{R}^{n}, \bx_{\calI'} = ([\bx]_i)_{i\in\calI'} = (x_{i})_{i\in\calI'} \in \mathbb{R}^{m}\), \(\alpha \geq 0\) and the masked softmax operator \(\softmax_{\calI'}\) is defined in \cref{eq:masked_softmax}.
    The masked softmax mean \(m_{\calI'}(\bx, \alpha)\) with hyperparameter \(\alpha\) satisfies:
    \begin{eqnarray*}
        m_{\calI'}(\bx + C\bone, \alpha) &=&  m_{\calI'}(\bx, \alpha) + C, \quad \forall C\in \mathbb{R}\\
        m_{\calI'}(C \bx, \alpha) &=& C \cdot m_{\calI'}(\bx, C \alpha) , \quad \forall C \in \mathbb{R}_{\geq 0}\\
        m_{\calI'}(\bx, \alpha') &\geq& m_{\calI'}(\bx, \alpha),\quad \text{for } \alpha' \geq \alpha
    \end{eqnarray*}
    Moreover, the following inequality holds:
    \[
    0 \leq \max_{i\in \calI'} [\bx]_i - m_{\calI'}(\bx, \alpha)
    = \lim_{\alpha \to \infty} m_{\calI'}(\bx, \alpha) - m_{\calI'}(\bx, \alpha) < k, 
    \quad \text{for } \alpha \geq \underline{\alpha} := \frac{\ln m}{k}, k>0
    \]
\end{lemma}

\begin{lemma}[Deviation Bound of Softmax Mean]\label{lemma:deviation_bound_softmax_mean}
    Let \(\bx=([\bx]_i)_{i\in\calI} = (x_1, \ldots,x_{n}) \in \mathbb{R}^{n}, 
    \bdelta=([\bdelta]_i)_{i\in\calI} = (\delta_1, \ldots, \delta_n) \in \mathbb{R}^{n}\) 
    and \(\alpha \geq 0\), then the deviation bound of \(\max_{i \in \calI} [\bx]_i - \langle \softmax(\alpha [\bx+\bdelta]), \bx \rangle\) satisfies:
    \[
    \max_{i \in \calI} [\bx]_i - \langle \softmax(\alpha [\bx+\bdelta]), \bx \rangle 
    \leq 2\|\bdelta\|_\infty + k, \quad \text{for } \alpha \geq \underline{\alpha} := \frac{\ln n}{k}, k>0
    \]
    \[
    m(\bx + \bdelta, \alpha) - \langle \softmax(\alpha [\bx+\bdelta]), \bx \rangle 
    \leq \|\bdelta\|_\infty.
    \]
\end{lemma}

\begin{proof}[Proof for \cref{lemma:softmax_mean}]
    The first two properties are straightforward to verify by the definition of softmax mean.
    For the third non-decreasing property, it is due to the non-negativity of the partial derivative with respect to \(\alpha\). 
    By letting \(\bp(\alpha) := \softmax(\alpha \bx)\), then:
    \[
    \frac{\partial m(\bx, \alpha)}{\partial \alpha} = \bx^\top(\diag(\bp(\alpha)) - \bp(\alpha)\bp(\alpha)^\top)\bx =\Var_{\bp(\alpha)}[\bx] \geq 0
    \]
    The left hand side of thelast inequality follows from the definition of softmax mean.
    \[
    \max_{i\in \calI} [\bx]_i = \langle \bp(\alpha), \max_{i\in \calI}[\bx]_i \bone\rangle \geq \langle \bp(\alpha), \bx\rangle = m(\bx, \alpha)
    \]
    Let's focus on the right hand side of the last inequality. 
    Since the inequality holds when \(n=1\), we only need to prove the inequality for \(n>1\).
    Letting \(\bp(\alpha) := \softmax(\alpha \bx)\), then by applying Jensen's inequality to the concave function \(\ln (\cdot)\),
    we derive the upper bound of the entropy of \(\bp(\alpha)\):
    \[
    \Ent[\bp(\alpha)]:=
    \sum_{i\in \calI} [\bp(\alpha)]_i \ln\frac{1}{[\bp(\alpha)]_i} = \E_{\bp(\alpha)} \ln\frac{1}{\bp(\alpha)} 
    \leq \ln \E_{\bp(\alpha)} \frac{1}{\bp(\alpha)} = \ln n
    \]
    By the definition of softmax mean \(m(\bx, \alpha) = \E_{\bp(\alpha)} [\bx]\), then for the case \(n>1\), we have:
    \[
    \Ent[\bp(\alpha)] = - \E_{\bp(\alpha)}[\alpha \bx] + \ln\sum_{i\in \calI} \exp(\alpha [\bx]_i) > -\alpha\cdot m(\bx, \alpha) + \alpha \max_{i\in \calI} [\bx]_i 
    \]
    Combining these two inequalities, when \(\alpha \geq \frac{\ln n}{k}> 0\) for some \(k>0, n>1\), we have:
    \[
    \max_{i\in \calI} [\bx]_i - m(\bx, \alpha) < \frac{\ln n}{\alpha} \leq k
    \]
\end{proof}

\begin{proof}[Proof for \cref{lemma:deviation_bound_softmax_mean}]
    Let \(\bp(\alpha):= \softmax(\alpha(\bx+\bdelta))\) and \(\bz:= \bx+\bdelta\),
    by using the last inequality in \cref{lemma:softmax_mean}, we have:
    \[
    m(\bz, \alpha)=
    \langle \bp(\alpha), \bz\rangle
    \ge
    \max_{i\in\calI} [\bz]_i - \frac{\ln n}{\alpha}.
    \]
    Substituting \(\bz = \bx+\bdelta\), and using \(\max_{i\in\calI} [\bz]_i - \min_{i\in\calI} [\bdelta]_i \geq \max_{i\in\in\calI}[\bx]_i\) gives
    \[
    \langle \bp(\alpha), \bx+\bdelta\rangle
    \ge
    \max_{i\in\calI}[\bx]_i  +  \min_{i\in\calI}[\bdelta]_i
    -
    \frac{\ln n}{\alpha}
    \]
    Rearranging and using \(\langle \bp(\alpha), \bdelta\rangle - \min_{i\in\calI}[\bdelta]
    \leq \max_{i\in\calI} [\bdelta]_i - \min_{i\in\calI}[\bdelta]_i \leq 2\|\bdelta\|_\infty\)
    and \(\alpha \geq \frac{\ln n}{k}, k>0\) yields
    \[
    \max_{i\in\calI} [\bx]_i
    -
    \langle \bp(\alpha), \bx\rangle
    \leq
    2\|\bdelta\|_\infty
    +
    \frac{\ln n}{\alpha} \leq 2\|\bdelta\|_\infty + k.
    \]
    This completes the proof of the first inequality. 
    Regarding the second inequality, we again apply the definition of softmax mean and the inequality \(\langle \bp(\alpha), \bdelta \rangle \leq \max_{i\in\calI} [\bdelta]_i \leq \|\bdelta\|_\infty\) and \(\alpha \geq \frac{\ln n}{k}, k>0\) to get:
    \[
    m(\bx + \bdelta, \alpha) - \langle \bp(\alpha), \bx \rangle
    = \langle \bp(\alpha), \bdelta \rangle
    \leq \|\bdelta\|_\infty.
    \]
    This completes the proof of the second inequality.
\end{proof}

\begin{lemma}[Subgaussianity]\label{lemma:subgaussianity}
    Suppose a random variable \(X\) such that 
    \(\E[X]=0\) and \(\E[\exp(X^2)]\leq 2\), then \(Z\) is 1-subgaussian, i.e.,
    \[
    \ln\E[\exp(\lambda X)] \leq \frac{\lambda^2}{2} \quad \forall \lambda \in \mathbb{R}
    \]
\end{lemma}

\begin{remark} Suppose a random vector \(\bz\) such that \(\E[\bz \mid \calF] = \vec{0}\) 
    and \(\E[\exp(\|\bz\|^2/\sigma^2)\mid \calF]\leq 2\), then \(\bz\) is \(\sigma^2\)-subgaussian, i.e.,
    \[
    \ln \E[\exp(\lambda \langle \be, \bz \rangle)\mid \calF] \leq \frac{\lambda^2 \sigma^2}{2} \quad \forall \be\in \mathbb{S}^{d-1}, \forall \lambda \in \mathbb{R}
    \]
    since by letting \(X = \|z\|/\sigma\) and using the above lemma, we have:
    \[
    \ln \E[\exp(\lambda \langle \be, \bz \rangle)\mid \calF] 
    \leq \ln \E[\exp(\lambda \|\bz\|)\mid \calF] 
    = \ln \E[\exp(\lambda \sigma \|\bz\|/\sigma)\mid \calF] 
    = \ln \E[\exp((\lambda\sigma) X)\mid \calF] \leq \frac{\lambda^2 \sigma^2}{2}
    \]
\end{remark}

\begin{lemma}[Average of Subgaussian Random Vectors]\label{lemma:average_subgaussian_random_vectors}
    Suppose random vectors \((\bz_\tau)_{\tau\in[E]}\) and a filtration \((\mathcal{F}_\tau)_{\tau\in\mathbb{Z}_{\geq0}}\)
    are such that \(\bz_\tau\) is \(\mathcal{F}_{\tau+1}\)-measurable and 
    \(\E[\bz_\tau \mid \mathcal{F}_\tau] = \vec{0}, \E[\exp(\|\bz_\tau\|^2)\mid \calF_\tau] \leq 2\) for all \(\tau\in[E]\),
    then the average of the random vectors \( \bz := \frac{1}{2\sqrt{E}} \sum_{\tau\in[E]} \bz_\tau\) 
    satisfies \(\E[\bz \mid \calF_0] = \vec{0}\) and the following subgaussian tail bound:
    \[
    \E[\exp(\|\bz\|^2)\mid \calF_0] 
    = \E\left[ \exp\left(\frac{1}{4E} \left\|\sum_{\tau\in[E]} \bz_\tau\right\|^2\right) \mid \calF_0\right] \leq 2
    \]
\end{lemma}

\begin{remark}[Remark 1 for \cref{lemma:average_subgaussian_random_vectors}]
Let \((\calF_s)_{s=1}^B\) be a filtration, and \(\calD\)
be a distribution on a measurable space \(\calZ\).
Let a sequence \((\zeta_s)_{s=1}^B\stackrel{\text{i.i.d.}}{\sim} \calD\) be drawn from \(\calD\) independently and identically,
and is independent of \(\calF_0\), i.e., \((\zeta_s)_{s=1}^B\ind \calF_0\).
The filtration is generated sequentially such that \(\calF_s = \sigma(\calF_{s-1}, \zeta_s)\),
and \(\bw_0 \in \Theta\) is \(\calF_0\)-measurable.
Then for a measurable function \(\bh: \Theta \times \calZ \to \mathbb{R}^{d}\), 
assume that for any deterministic \(\bw \in \Theta\), 
we have \(\E_{\zeta \sim \calD}[\bh(\bw, \zeta)] = \vec{0}\)
and \(\E_{\zeta \sim \calD}[\exp(\|\bh(\bw, \zeta)\|^2)] \leq 2\).
Then for \(\bz := \frac{1}{2\sqrt{B}} \sum_{s=1}^B \bh(\bw_0, \zeta_s)\), 
it satisfies \(\E[\bz \mid \calF_0] = \vec{0}\) and a subgaussian bound:
\[
\E[\exp(\|\bz\|^2)\mid \calF_0] 
\leq 2
\]
\end{remark}

\begin{remark}[Remark 2 for \cref{lemma:average_subgaussian_random_vectors}]
    Let \((\calF_\tau)_{\tau\in[E]}\) be a filtration, and \(\calD'=\calD^{\otimes B}\)
    be a distribution on a measurable space \(\calZ'=\calZ^B\).
    Let a sequence \((\bzeta_\tau)_{\tau\in[E]}\stackrel{\text{i.i.d.}}{\sim} \calD'\) be drawn from \(\calD'\) independently and identically,
    and is independent of \(\calF_{0}\), i.e., \((\bzeta_\tau)_{\tau\in[E]}\ind \calF_{0}\).
    The filtration is generated sequentially such that \(\calF_{\tau} = \sigma(\calF_{\tau-1}, \bzeta_{\tau-1})\),
    and \(\bw_\tau \in \Theta\) is \(\calF_{\tau-1}\)-measurable.
    Then for a measurable function \(\bh': \Theta \times \calZ' \to \mathbb{R}^{d}\), 
    assume that for any deterministic \(\bw \in \Theta\), 
    \(\E_{\bzeta \sim \calD'}[\bh'(\bw, \bzeta)] = \vec{0}\)
    and \(\E_{\bzeta \sim \calD'}[\exp(\|\bh'(\bw, \bzeta)\|^2)] \leq 2\).
    Then \(\bz := \frac{1}{2\sqrt{E}} \sum_{\tau\in[E]} \bh'(\bw_\tau, \bzeta_\tau)\), 
    satisfies \(\E[\bz \mid \calF_0] = \vec{0}\) and a subgaussian bound:
    \[
    \E[\exp(\|\bz\|^2)\mid \calF_0] \leq 2
    \]
\end{remark}

\begin{proof}[Proof for \cref{lemma:subgaussianity}]
We show these elementary inqualities \((\exp(|u|/\sqrt{2})-1)^2 \geq \exp|u|-|u|-1 \geq \exp(u) -u -1\) hold for any \(u\in \mathbb{R}\), since
\[
(\exp(|u|/\sqrt{2})-1)^2 -(\exp|u|-|u|-1) = \sum_{k= 2}^\infty \frac{(2^{k/2}+1)(2^{k/2}-2)}{2^{k/2}k!}|u|^k \geq 0
\]
\[
\exp|u|-|u|-1 = \exp(u) -u -1 + 2(\sinh|u|-|u|)\1_{u<0} \geq \exp(u) -u -1
\]
Then by Cauchy-Schwarz inequality, we have:
\begin{eqnarray*}
(\exp(\lambda^2/2)-1)(\exp(X^2)-1)
&=&\left(\sum_{k\geq 1} \frac{[\lambda/\sqrt{2}]^{2k}}{k!}\right) \left( \sum_{k\geq 1} \frac{X^{2k}}{k!} \right)
\geq\left(\sum_{k\geq 1} \frac{(|\lambda X|/\sqrt{2})^k}{k!}\right)^2
= \left(\mathe^{\frac{|\lambda X|}{\sqrt{2}}}-1\right)^2\\
&\geq& \exp(\lambda X)-\lambda X -1
\end{eqnarray*}
By taking the expectation on both sides, and using \(\E[X]=0\) and \(\E[\exp(X^2)] \leq 2\), we have:
\[
\E[\exp(\lambda X)] \leq (\exp(\lambda^2/2)-1)(\E[\exp(X^2)]-1) +\lambda\E[X] +1 \leq
\exp(\lambda^2/2)
\]
\end{proof}

\begin{proof}[Proof for \cref{lemma:average_subgaussian_random_vectors}]
We introduce the partial sum of random vectors \(\bS_\tau = \sum_{\tau'\in[\tau]} \bz_{\tau'}\) for all \(\tau\in[E+1]\) , 
which is \(\mathcal{F}_\tau\)-measurable, 
then \(\bS_0 = \vec{0}\) and \(\bS_{E} = \sum_{\tau\in[E]} \bz_\tau\). 
By using \(\E[\exp(\|\bz_\tau\|^2)\mid \calF_\tau] \leq 2\) and applying \cref{lemma:subgaussianity}, 
for \(\calF_\tau\)-measurable \(\bS_\tau\) and any \(\lambda \in \mathbb{R}\):
\[
\E[\exp(\lambda \langle \bS_\tau, \bz_\tau\rangle)\mid \calF_\tau] 
\leq \exp\left(\frac{\lambda^2 \|\bS_\tau\|^2}{2}\right)
\]
Consider such a quantity 
\(\exp(\lambda_\tau \|\bS_\tau\|^2)\) with a sequence \((\lambda_\tau)_{\tau=1}^E\) of positive numbers, 
we will select \(\lambda_\tau\) wisely to establish a recursive relationship 
between \(\exp(\lambda_\tau \|\bS_\tau\|^2)\) and \(\exp(\lambda_{\tau+1} \|\bS_{\tau+1}\|^2)\).

From the defintion of \(\bS_\tau\) and using \(\bS_\tau\) is \(\calF_\tau\)-measurable, we have:
\[
\E[\exp(\lambda_{\tau+1} \|\bS_{\tau+1}\|^2) \mid \calF_\tau] 
\leq \exp(\lambda_{\tau+1}\|\bS_{\tau}\|^2)
\E[\exp(\lambda_{\tau+1} \|\bz_\tau\|^2 + 2\lambda_{\tau+1} \langle \bS_\tau, \bz_{\tau}\rangle) \mid \calF_\tau]
\]
Then applying Hölder's inequality \(\E[|XY|\mid \calF] \leq \E[|X|^p\mid \calF]^{\frac{1}{p}} \E[|Y|^q\mid \calF]^{\frac{1}{q}}\) with \(p=1/\lambda_{\tau+1}\) and \(q=1/(1-\lambda_{\tau+1})\), 
to upper bound the conditional expectation by relating it to 
\(\E[\exp(\|\bz_\tau\|^2)] \leq 2\) and \(\E[\exp(\lambda \langle \bS_\tau, \bz_\tau\rangle) \mid \calF_\tau] \leq \exp(\frac{\lambda^2 \|\bS_\tau\|^2}{2})\).
\begin{eqnarray*}
\E[\exp(\lambda_{\tau+1} \|\bz_\tau\|^2 + 2\lambda_{\tau+1} \langle \bS_\tau, \bz_{\tau}\rangle) \mid \calF_\tau]
&\leq& \E[\exp(\|\bz_\tau\|^2) \mid \calF_\tau]^{\lambda_{\tau+1}} \E\left[\exp\left(\frac{2\lambda_{\tau+1}}{1-\lambda_{\tau+1}} \langle \bS_\tau, \bz_{\tau}\rangle\right) \mid \calF_\tau\right]^{1-\lambda_{\tau+1}}\\
& \leq & 2^{\lambda_{\tau+1}} \exp\left((1-\lambda_{\tau+1})\cdot\frac{1}{2}\left(\frac{2\lambda_{\tau+1}}{1-\lambda_{\tau+1}}\right)^2 
\|\bS_\tau\|^2\right)
\end{eqnarray*}
By selecting the sequence \((\lambda_\tau)_{\tau=1}^E\) such that
the final value \(\lambda_{E} = \frac{1}{4E}\) and satisfies the recursive relationship:
\[
\lambda_{\tau} := \lambda_{\tau+1} + (1-\lambda_{\tau+1})\cdot\frac{1}{2}\left(\frac{2\lambda_{\tau+1}}{1-\lambda_{\tau+1}}\right)^2
= \lambda_{\tau+1} \cdot \frac{1+\lambda_{\tau+1}}{1-\lambda_{\tau+1}}
\]

then we can establish the recursive relationship for \(M_\tau := \exp(\lambda_\tau \|\bS_\tau\|^2 - \sum_{\tau'=1}^{\tau} \lambda_{\tau'})\)
with \(M_0 = \exp(0) =1\):
\[
\E[\exp(\lambda_{\tau+1} \|\bS_{\tau+1}\|^2 - \sum_{\tau'=1}^{\tau+1} \lambda_{\tau'}) \mid \calF_\tau] \leq 
\exp(\lambda_\tau \|\bS_\tau\|^2) \cdot 2^{\lambda_{\tau+1}} \exp(- \sum_{\tau'=1}^{\tau+1} \lambda_{\tau'})
\leq \exp(\lambda_\tau \|\bS_\tau\|^2 - \sum_{\tau'=1}^{\tau} \lambda_{\tau'}) 
\]
From the above recursive relationship, we show that \(M_\tau\) is a supermartingale, 
namely \(\E[M_{\tau+1} \mid \calF_\tau] \leq M_\tau\).
Then by the tower property of conditional expectation (Theorem 4.2.4 on page 189 of~\cite{durrett2019probability}) and the fact that \(M_0 = 1\), we have:
\[
\E[M_E\mid \calF_0] \leq \E[M_0\mid \calF_0] = M_0 = 1, \quad \lambda_E = \frac{1}{4E}
\implies 
\E\left[ \exp\left(\frac{1}{4E} \|\bS_E\|^2\right)\mid\calF_0\right] \leq \exp\left(\sum_{\tau=1}^E \lambda_\tau\right)
\]

To upperbound \(\sum_{\tau=1}^E \lambda_\tau\), we introduce \(a_\tau := \lambda_{E+1-\tau}\) with initial value \(a_1 = \lambda_E = \frac{1}{4E}\) 
and use the recursive relation:
\[
\frac{1}{a_\tau} = \frac{1}{a_1} +
\sum_{\tau'=1}^{\tau-1}\left(\frac{1}{a_{\tau'+1}} - \frac{1}{a_{\tau'}}\right)
= 4E-\sum_{\tau'=1}^{\tau-1}\frac{2}{1+a_{\tau'}} > 4E-2(\tau-1)
\]
Therefore, we have:
\[
\sum_{\tau=1}^E \lambda_\tau = \sum_{\tau=1}^E a_\tau < \frac{1}{2}\sum_{\tau=1}^E \frac{1}{E+(E+1-\tau)}
= \frac{1}{2}\sum_{\tau=1}^E \frac{1}{E+\tau} < \frac{1}{2} \int_0^E \frac{\mathd \tau}{\tau+E} = \frac{\ln 2}{2}
\]
By selecting \(\bz := \frac{1}{2\sqrt{E}} \bS_E = \frac{1}{2\sqrt{E}} \sum_{\tau\in[E]} \bz_\tau\), we have:
\[
\E\left[\exp(\|\bz\|^2)\mid \calF_0\right] = \E\left[\exp\left(\frac{1}{4E} \|\bS_E\|^2\right) \mid \calF_0\right] \leq \exp\left(\frac{\ln 2}{2}\right)
=\sqrt{2} < 2
\]

\end{proof}

\begin{lemma}[Freezing Lemma: Conditional Expectation with Independent Random Variable]\label{lemma:conditional_expectation_independent_rv}
    Let \(X\) and \(Z\) be random variables on a probability space \((\Omega, \mathcal{F}, \Pr)\). Assume:
    \(X\) is \(\mathcal{G}\)-measurable, where \(\mathcal{G}\) is a sub-\(\sigma\)-algebra of \(\mathcal{F}\).
    \(Z\) is independent of \(\mathcal{G}\). For any measurable function \(h\) satisfying \(\E[|h(X, Z)|] < \infty\), then it holds almost surely that for fixed value \(x\):
    \[
    \E[h(X, Z) \mid \mathcal{G}] = \E[h(x, Z)]_{x = X}
    \]
\end{lemma}
\begin{remark}
This lemma is known as the ``Freezing Lemma'',
and can be extended from random variables \(X\) and \(Z\) to random vectors \(\bx\) and \(\bz\)
with a measurable function \(h\) satisfies \(\E[|h(\bx, \bz)|] < \infty\) such that
\[
\E[h(\bx, \bz) \mid \mathcal{G}] = \E[h(\cdot, \bz)]_{\cdot = \bx}
\]
\end{remark}

\begin{proof}[Proof of \cref{lemma:conditional_expectation_independent_rv}]
    We first establish the theorem for the case of indicator function, from which the general result follows by using the \textbf{monotone class argument} (see example of ``Standard Machine'' procedure in proof of Theorem 4.1.14 on page 184 of \cite{durrett2019probability}).
    Let's begin with \(h(x, z) = \1_A(x) \cdot \1_C(z)\), where \(A\) and \(C\) are Borel sets on the real line. Then \(\E[h(x, Z)]_{x = X} = \1_A(X) \cdot \E[\1_C(Z)]\) is \(\mathcal{G}\)-measurable, since \(X\) is \(\mathcal{G}\)-measurable and \(\1_A(X)\) is also \(\mathcal{G}\)-measurable, therefore it is  \(\mathcal{G}\)-measurable as a constant multiple of a \(\mathcal{G}\)-measurable variable.
    For any set of outcomes \(B \in \mathcal{G}\), suppose that \(\E[h(X, Z) \cdot \1_B] = \E[ \E[h(x, Z)]_{x = X}\cdot \1_B]\) holds, then it implies \(\E[h(X, Z) \mid \mathcal{G}] = \E[h(x, Z)]_{x = X}\) by the definition of conditional expectation on page 178 of \cite{durrett2019probability}.
    The two sides are identical for any \(B \in \mathcal{G}\) as shown below by using that \(Z\) is independent of \(\mathcal{G}\) and \(X\) is \(\mathcal{G}\)-measurable (see also pages 38-39 and page 184 of \cite{durrett2019probability}).
    \[
    \E[h(X, Z) \cdot \1_B] = \E[\1_A(X) \cdot \1_C(Z) \cdot \1_B] = \E[\1_A(X) \cdot \1_B] \cdot \E[\1_C(Z)] = \E[ \E[h(x, Z)]_{x = X}\cdot \1_B]
    \]
    Using \textbf{monotone class argument}, the above result is extended to all general measurable functions \(h\) with \(\E[|h(X, Z)|] < \infty\). This monotone class argument involves three steps:
    (1) the linearity of conditional expectation generalizes the property from indicator functions to all simple functions (finite linear combinations of indicators);
    (2) the Monotone Convergence Theorem (Theorem 1.5.7 on page 23 of \cite{durrett2019probability}) then extends it to all non-negative measurable functions;
    (3) the decomposition of measurable function into its positive and negative parts (\(h = h^+ - h^-\)) covers all measurable functions satisfying \(\E[|h(X, Z)|] < \infty\).
\end{proof}

\begin{proof}[Proof of Remark 1 for \cref{lemma:average_subgaussian_random_vectors}]
As \(\bw_0 \in \Theta\) is \(\calF_0\)-measurable, and \(\calF_s = \sigma(\calF_{s-1}, \zeta_s)\), then we have \(\calF_{s-1}\subseteq \calF_s\) and therefore \(\bw_0\) is \(\calF_s\)-measurable for all \(s=1,2,\cdots,B\).
As \(\calF_{s} = \sigma(\calF_{s-1}, \zeta_s)\), then \(\zeta_s\) is \(\calF_s\)-measurable,
and therfore \(\bz_s := \bh(\bw_0, \zeta_s)\) is \(\calF_s\)-measurable.
Since \(\zeta_s\) is independent of \(\calF_0\) and all prior \(\zeta_1, ..., \zeta_{s-1}\),
then \(\zeta_s\) is independent of \(\calF_{s-1}= \sigma(\calF_0, \zeta_1, ..., \zeta_{s-1})\).
Therefore, by applying the Freezing Lemma \cref{lemma:conditional_expectation_independent_rv}, we have:
\[
\E[\bz_s \mid \calF_{s-1}] = \E[\bh(\bw_0, \zeta_s) \mid \calF_{s-1}] = 
\E[\bh(\bw, \zeta_s)]_{\bw = \bw_0} = \E_{\zeta \sim \calD}[\bh(\bw, \zeta)]_{\bw = \bw_0} = \vec{0}
\]
\[
\E[\exp(\|\bz_s\|^2) \mid \calF_{s-1}] = 
\E[\exp(\|\bh(\bw, \zeta_s)\|^2)]_{\bw = \bw_0} = \E_{\zeta \sim \calD}[\exp(\|\bh(\bw, \zeta)\|^2)]_{\bw = \bw_0} \leq 2
\]
Given \(\E[\bz_s \mid \calF_{s-1}] = \vec{0}\) and \(\E[\exp(\|\bz_s\|^2) \mid \calF_{s-1}] \leq 2\),
then by applying the subgaussian tail bound \cref{lemma:average_subgaussian_random_vectors}
to \(\bz := \frac{1}{2\sqrt{B}} \sum_{s=1}^B \bz_s = \frac{1}{2\sqrt{B}} \sum_{s=1}^B \bh(\bw_0, \zeta_s)\), we have
\(\E[\bz \mid \calF_0] = \vec{0}\) and \(\E[\exp(\|\bz\|^2) \mid \calF_0] \leq 2\).
\end{proof}

\begin{proof}[Proof of Remark 2 for \cref{lemma:average_subgaussian_random_vectors}]
Since \(\bw_\tau\) is \(\calF_\tau\)-measurable,
and \(\bzeta_\tau\) is \(\calF_{\tau+1}\)-measurable since \(\calF_{\tau+1} = \sigma(\calF_{\tau}, \bzeta_\tau)\).
Therefore, \(\bz_\tau := \bh'(\bw_\tau, \bzeta_\tau)\) is \(\calF_{\tau+1}\)-measurable.
Since \(\bzeta_\tau\) is independent of \(\calF_0\) and all prior \(\bzeta_0, ..., \bzeta_{\tau-1}\),
then \(\bzeta_\tau\) is independent of \(\calF_{\tau}= \sigma(\calF_0, \bzeta_0, ..., \bzeta_{\tau-1})\).
Therefore, by applying the Freezing Lemma \cref{lemma:conditional_expectation_independent_rv}, we have:
\[
\E[\bz_\tau \mid \calF_{\tau}] = \E[\bh'(\bw_\tau, \bzeta_\tau) \mid \calF_{\tau}] = 
\E[\bh'(\bw, \bzeta_\tau)]_{\bw = \bw_\tau} = \E_{\bzeta \sim \calD'}[\bh'(\bw, \bzeta)]_{\bw = \bw_\tau} = \vec{0}
\]
\[
\E[\exp(\|\bz_\tau\|^2) \mid \calF_{\tau}] = 
\E[\exp(\|\bh'(\bw, \bzeta_\tau)\|^2)]_{\bw = \bw_\tau} = \E_{\bzeta \sim \calD'}[\exp(\|\bh'(\bw, \bzeta)\|^2)]_{\bw = \bw_\tau} \leq 2
\]
Given \(\E[\bz_\tau \mid \calF_{\tau}] = \vec{0}\) and \(\E[\exp(\|\bz_\tau\|^2) \mid \calF_{\tau}] \leq 2\),
then by applying the subgaussian tail bound \cref{lemma:average_subgaussian_random_vectors}
to \(\bz := \frac{1}{2\sqrt{E}} \sum_{\tau=0}^{E-1} \bz_\tau = \frac{1}{2\sqrt{E}} \sum_{\tau=0}^{E-1} \bh'(\bw_\tau, \bzeta_\tau)\), we have
\(\E[\bz \mid \calF_0] = \vec{0}\) and \(\E[\exp(\|\bz\|^2) \mid \calF_0] \leq 2\).
\end{proof}

\begin{lemma}[Maximal Inequality for Conditionally Sub-Gaussian Martingale Differences]\label{lemma:maximal_inequality_subgaussianity}
    Let a sequence of random vectors $(\mathbf{z}_k)_{k \in \mathbb{Z}_{\geq 0}}$ in $\mathbb{R}^{n}$ 
    and a sequence of indicator random variables $(\1_k)_{k \in \mathbb{Z}_{\geq 0}}$ 
    with $\1_k \in \{0, 1\}$ be adapted to a filtration $(\mathcal{F}_{k+1})_{k \in \mathbb{Z}_{\geq 0}}$ with $\mathcal{F}_0 = \{\emptyset, \Omega\}$, 
    namely $\bz_k, \1_k$ are $\mathcal{F}_{k+1}$-measurable. 
    Assume that the sequence satisfies the martingale difference property 
    $\mathbb{E}[\mathbf{z}_k \mid \mathcal{F}_k] = \vec{0}$ and 
    possesses a sub-Gaussian tail bound such that for any $\lambda \in \R$ and any $i \in \calI \equiv \{1, \cdots, n\}$:
    \[
    \ln \E\left[\exp(\lambda [\mathbf{z}_k]_i) \mid \mathcal{F}_k\right] \leq \frac{\lambda^2}{2}.
    \]
    Then for any $\delta \in (0, 1)$, the following inequality holds with probability at least $1 - \delta$:
    \[
    \sum_{k\in[K]} \|\bz_k\|_\infty \1_k  
    \leq \sqrt{2 \ln \frac{2Kn}{\delta} } \sum_{k\in[K]} \1_k.
    \]
\end{lemma}

\begin{proof}
    The above inequality holds if \(\sum_{k\in[K]} \1_k =0\), otherwise, we have \(\sum_{k\in[K]} \1_k > 0\).
    Noting that the maximum is not less than the average of \(\left\{ \|\bz_k\|_\infty \mid \1_k = 1, k\in[K] \right\}\),
    and the maximum whose index is in the subset \(\{k\mid \1_k = 1, k\in[K]\}\) is less than the maximum of the entire trajectory \([K]\):
    \[
        \sum_{k\in[K]} \|\bz_k\|_\infty \1_k \Bigg/ \sum_{k\in[K]} \1_k 
        \leq \max_{k: \1_k = 1, k\in[K]} \|\bz_k\|_\infty
        \leq \max_{k\in[K]} \|\bz_k\|_\infty 
        = \max_{k\in[K]} \max_{i\in \calI} |[\bz_k]_i|
    \]
The probability of \(\max_{k\in[K]} \|\bz_k\|_\infty  > t\) for some \(t \geq 0\) is bounded 
by the probability of all \(|[\bz_k]_i| >t\) for any \(i \in \calI\) and \(k\in[K]\),
then taking the union bound over all \(i \in \calI\) and \(k\in[K]\).
\[
\Pr\left(\max_{k\in[K]} \max_{i \in \calI} |[\bz_k]_i| > t\right) \leq
\Pr\left(|[\bz_k]_i| > t, \forall i \in \calI, k\in[K]\right) \leq
\sum_{k\in[K]} \sum_{i\in \calI} \Pr\left(|[\bz_k]_i| > t\right) 
\]
For any $t > 0$, by the Law of Total Expectation and the conditional Chernoff bound 
using the sub-Gaussian tail bound \(\ln \E\left[\exp(\lambda [\bz_k]_i) \mid \mathcal{F}_k\right] \leq \frac{\lambda^2 \nu}{2}\):
\[
    \Pr([\bz_k]_i > t) = \E[\Pr([\bz_k]_i > t \mid \mathcal{F}_k)] 
    \leq \E\left[\inf_{\lambda > 0} \exp\left(-\lambda t + \ln \E\left[\exp(\lambda [\bz_k]_i) \mid \mathcal{F}_k\right]\right)\right]
    \leq \exp\left(\inf_{\lambda > 0}-\lambda t + \frac{\lambda^2}{2}\right)
    = \mathe^{-\frac{t^2}{2}}.
\]
Similary, for any \(t > 0\), 
we have \(\Pr([\bz_k]_i < -t) \leq \exp\left( -\frac{t^2}{2} \right)\), and 
therefore, the probability of \(|[\bz_k]_i| > t\) is bounded by \(2 \exp\left( -\frac{t^2}{2} \right)\).
Setting the failure probability \(2Kn \exp(-t^2/2) = \delta\) and solving for \(t\) yields the high-probability threshold \(t = \sqrt{2 \ln (2Kn/\delta)}\). 
Since the ratio is bounded by \(\max_{k\in[K]} \|\bz_k\|_\infty\), it is bounded by \(t\) with probability at least \(1 - \delta\).
\[
\Pr\left( \frac{\sum\limits_{k\in[K]} \|\bz_k\|_\infty \1_k}{\sum\limits_{k\in[K]} \1_k} > \sqrt{2 \ln \frac{2Kn}{\delta} } \right) 
\leq \Pr\left( \max_{k\in[K]} \|\bz_k\|_\infty > \sqrt{2 \ln \frac{2Kn}{\delta} } \right)
\leq 2Kn \exp\left( -\frac{2 \ln \frac{2Kn}{\delta} }{2} \right)
= \delta.
\]
\end{proof}


\begin{lemma}[Upper Bound for Binomial Coefficient]\label{lemma:upperbound_binom}
    Let \(n\in\mathbb{Z}_+\) and \(n', m\in \mathbb{Z}_{\geq 0}\) such that \(n \geq n' \geq m\), then the following inequality holds:
    \[
    \frac{\binom{n'}{m}}{\binom{n}{m}} \leq \left( 1- \frac{m}{n}\right)^{n\left(1-\frac{n'}{n}\right)}
    \]
\end{lemma}

\begin{proof}
Rewriting the ratio of two binomial coefficients as a product, applying the elementary inequality \(1+t \leq \exp(t)\) for any \(t\in \mathbb{R}\),
and lowerbounding a summation with a integral such that \(\sum_{m'\in[m]}\frac{1}{n-m'} \geq \int_0^m \frac{1}{n-m'}\mathd m' = -\ln(1-\frac{m}{n})\).
\[
    \frac{\binom{n'}{m}}{\binom{n}{m}} = \prod_{m'\in[m]} \left(1 - \frac{n-n'}{n-m'}\right)
    \leq \exp\left(-(n-n') \sum_{m'\in[m]}\frac{1}{n-m'}\right)
    \leq \left( 1- \frac{m}{n}\right)^{n\left(1-\frac{n'}{n}\right)} 
\]
\end{proof}

The Assumption of Uniform CDF Bound is equivalent to imposing the restriction on $\bx \in \mathbb{R}^n$, 
\[
\frac{1}{n} \sum_{i \in \calI} \1\{\max_{i'} x_{i'} - x_i \geq t\} \leq \left(1 - \frac{t}{\sigma}\right)_+,\; \forall t \in \mathbb{R}.
\]
Let $x_{(1)} \leq x_{(2)} \leq \cdots \leq x_{(n)}$ denote the ordered elements of $\bx$. The condition above is equivalent to requiring that $\sigma$ bounds the scaled gaps of the order statistics $\max_{i \in \calI, i\neq n} \frac{x_{(n)} - x_{(i)}}{1 - (i-1)/n} \leq \sigma.$

\begin{lemma}[Conditional Exponential Tail Bound for Maximum of a Subset with a Uniform CDF Bound]\label{lemma:conditional_exponential_tail_bound}
Let \(\bx = (x_1, \ldots, x_{n})\in \mathbb{R}^{n}\), which satisfies the following condition (Uniform CDF Bound) with some \(\sigma>0\) and any \(t \in \mathbb{R}\):
\[
\frac{1}{n} \sum_{i\in \calI} \1\{\max_{i'\in \calI} [\bx]_{i'} - [\bx]_i \geq t\} \leq \left(1 - \frac{t}{\sigma}\right)_+
\]
A subset \(\calI'\) with fixed cardinality \(|\calI'| = m \in \mathbb{Z}_+\) is selected from \(\calI\) uniformly at random, which is independent of \(\bx\), namely \(\calI' \sim \text{Unif}(\calC_{m}(\calI))\) with \(\calC_{m}(\calI) := \{A \subseteq \calI: |A| = m\}\), and 
\[
\Pr(\calI' = A)= \Pr(\calI' = A\mid \bx) = \frac{1}{|\calC_{m}(\calI)|} = \frac{1}{\binom{n}{m}}, \quad \forall A \in \calC_{m}(\calI)
\]
Then the difference between the maximum \(\max_{i\in \calI} [\bx]_i\) over the entire set \(\calI\), and the maximum \(\max_{i\in \calI'} [\bx]_i\) over the subset \(\calI'\), is bounded as follows
when \(r:=\frac{m}{n} \neq 1\):
\[
    \Pr\left(\max_{i\in \calI} [\bx]_i - \max_{i\in \calI'} [\bx]_i\geq t \mid \bx \right)
    \leq \exp\left(-\frac{[-\ln(1-r)]n}{\sigma} t\right)
\] 
\[
    \E\left[\max_{i\in \calI} [\bx]_i - \max_{i\in \calI'} [\bx]_i \mid \bx\right]
    \leq \frac{\sigma}{[-\ln(1-r)]n} < \frac{\sigma}{rn} = \frac{\sigma}{m}
\]
Otherwise, if \(r:=\frac{m}{n} =1\), then \(\calI = \calI'\) and \(\max_{i\in \calI} [\bx]_i = \max_{i\in \calI'} [\bx]_i\).
\end{lemma}

\begin{proof}
Let's define the following set \(\calI_t\) for some \(t\geq 0\), then the condition can be rewritten as:
\[
\calI_t(\bx) := \{i\in \calI' \mid \max_{i'\in \calI} [\bx]_{i'} - [\bx]_i \geq t\}, \qquad \frac{|\calI_t(\bx)|}{n} \leq \left(1 - \frac{t}{\sigma}\right)_+
\]
Noting that the event \(\max_{i\in \calI} [\bx]_i - \max_{i\in \calI} [\bx]_i\geq t\) is equivalent to 
\(\max_{i'\in \calI} [\bx]_{i'} -[\bx]_i \geq t, \forall i\in \calI\), namely \(\calI \subseteq \calI_{t}(\bx)\).
Using \(\Pr(\calI = A\mid \bx) = \frac{1}{|\calC_{m}(\calI)|}\), 
\(\Pr(A\subseteq \calI_t(\bx)\mid \calI=A;\bx) = \1\{A\subseteq \calI_t(\bx)\}\), 
letting \(\calC_{m}(\calI_t(\bx)) := \{A \subseteq \calI_t(\bx) \mid |A|=m \}\).
\begin{eqnarray*}
    & &\Pr\left(\max_{i\in \calI} [\bx]_i - \max_{i\in \calI} [\bx]_i\geq t \mid \bx \right)
    = \Pr\left(\calI\subseteq \calI_t(\bx) \mid \bx\right)
    = \sum_{A\in \calC_m(\calI)}\Pr(A\subseteq \calI_t(\bx)\mid \calI=A;\bx) \Pr(\calI=A\mid\bx)\\
    &=& \frac{\sum_{A\in \calC_m(\calI)} \1\{A\subseteq \calI_t(\bx)\}}{|\calC_{m}(\calI)|} 
    = \frac{|\{A \subseteq \calI_t(\bx) \mid |A|=m \}|}{|\calC_{m}(\calI)|}
    = \frac{|\calC_{m}(\calI_t(\bx))|}{|\calC_{m}(\calI)|}
    = \frac{\binom{|\calI_t(\bx)|}{m}}{\binom{n}{m}}\1_{m\leq |\calI_t(\bx)|}
\end{eqnarray*}
Noting \(|\calI'|=m\), applying \cref{lemma:upperbound_binom} to upperbound the ratio of two binomial coefficients 
by substituting \(n'\gets |\calI_t(\bx)|\),
and using the condition such that \(\frac{|\calI_t(\bx)|}{n} \leq \left(1 - \frac{t}{\sigma}\right)_+\)
\[
    \Pr\left(\max_{i\in \calI} [\bx]_i - \max_{i\in \calI'} [\bx]_i\geq t \mid \bx \right)
    = \frac{\binom{|\calI_t(\bx)|}{|\calI'|}}{\binom{|\calI|}{|\calI'|}}\1_{|\calI'|\leq |\calI_t(\bx)|}
    \leq \left(1- \frac{|\calI'|}{|\calI|} \right)^{n\left[1-\left(1 - \frac{t}{\sigma}\right)_+\right]} \1\left\{\frac{|\calI'|}{|\calI|}\leq  \left(1 - \frac{t}{\sigma}\right)_+\right\}
\]
Therefore, by letting \(r:=\frac{m}{n}\), we have \(\Pr\left(\max_{i\in \calI} [\bx]_i - \max_{i\in \calI} [\bx]_i\geq t \mid \bx \right) =0\) when \(t\geq \sigma\).
For any \(t \geq0\), we have:
\[
    \Pr\left(\max_{i\in \calI} [\bx]_i - \max_{i\in \calI'} [\bx]_i\geq t \mid \bx \right)
    \leq (1-r)^{n\frac{t}{\sigma}} \1\{t\leq(1-r)\sigma\}
\]
When \(r := \frac{m}{n}=1\), then \(\Pr\left(\max_{i\in \calI} [\bx]_i - \max_{i\in \calI'} [\bx]_i\geq t \mid \bx \right) =\1\{t=0\}\) for \(t\geq 0\), and \(\max_{i\in \calI} [\bx]_i = \max_{i\in \calI'} [\bx]_i\).
If \(r := \frac{m}{n}\neq 1\), then the upperbound of expectation is established as follows.
\[
    \Pr\left(\max_{i\in \calI} [\bx]_i - \max_{i\in \calI'} [\bx]_i\geq t \mid \bx \right)
    \leq \exp\left(-\frac{[-\ln(1-r)]n}{\sigma} t\right)
\]
\[
    \E\left[\max_{i\in \calI} [\bx]_i - \max_{i\in \calI'} [\bx]_i \mid \bx\right]
    = \int_{t\geq 0}\Pr\left(\max_{i\in \calI} [\bx]_i - \max_{i\in \calI'} [\bx]_i \geq t\right)\mathd t
    \leq \frac{\sigma}{[-\ln(1-r)]n} < \frac{\sigma}{rn} = \frac{\sigma}{m}
\]
\end{proof}

\begin{lemma}[Conditional Exponential Tail Bound with Assumption of Uniformly Bounded Relative Gap]\label{lemma:conditional_exponential_tail_bound_with_uniform_cdf}

Let a filtration \((\calG_k)_{k\in \mathbb{Z}_{\geq 0}}\) be defined with \(\calG_0 = \{\emptyset, \Omega\}\) and 
\(\calG_k = \sigma(\calG_{k-1}, \calI_{k-1}, \calB_{k-1})\),
where a sequence of subsets \((\calI_k)_{k\in \mathbb{Z}_{\geq 0}} \stackrel{\text{i.i.d.}}{\sim} \text{Unif}(\mathcal{C}_m(\mathcal{I})) \) 
with \( \calC_m(\calI) = \{A \subseteq \calI \mid |A| = m\} \) for some \(0 < m \leq n\) and \(\calI \equiv \{1, \cdots, n\}\), 
are independently and identically sampled from \(\calC_m(\calI)\), and are independent of 
the collection of random sample batches \(\calB_k\) at any time \(k\).
Let \(\bw_k \in \Theta\) be \(\calG_k\)-measurable, 
and \(\bff,\bg: \Theta \to \mathbb{R}^{n}\) be two measurable functions,
and assume that \Cref{assumption:uniform_cdf_bound} of Uniformly Bounded Relative Gap holds for 
\(\bff(\bw), \bg(\bw)\) and their maximum values \(F(\bw)\equiv \max_{i\in \calI} [\bff(\bw)]_i, G(\bw)\equiv \max_{i\in \calI} [\bg(\bw)]_i\).
Then for \(F(\bw_k; \calI_k) \equiv \max_{i\in \calI_k} [\bff(\bw_k)]_i\) and \(G(\bw_k; \calI_k) \equiv \max_{i\in \calI_k} [\bg(\bw_k)]_i\),
there exist exponential tail bounds for the conditional probabilities of \(F(\bw_k) - F(\bw_k; \calI_k) \) and \(G(\bw_k) - G(\bw_k; \calI_k) \) as follows for any \( k \in \mathbb{Z}_{\geq 0}, t\in \mathbb{R}_{\geq 0} \):
\begin{eqnarray*}
    \Pr\left(F(\bw_k)-F(\bw_k; \calI_k) \geq t \mid \calG_k\right) &\leq& \exp\left(- \frac{[-\ln(1-m/n)]n}{\sigma} t\right),\\
    \Pr\left(G(\bw_k)-G(\bw_k; \calI_k) \geq t \mid \calG_k\right) &\leq& \exp\left(- \frac{[-\ln(1-m/n)]n}{\sigma} t\right).
\end{eqnarray*}
\end{lemma}
    
    
\begin{proof}
Applying \Cref{assumption:uniform_cdf_bound}, for any fixed value \(\bv \in \Theta\), \(t \in \mathbb{R}_{\geq 0}\), \(k \in \mathbb{Z}_{\geq 0}\)
and subsets \(\calI_k \subseteq \calI\)
and substituting fixed values \(\bx \gets \bff(\bv), \bg(\bv)\) and \(\calI' \gets \calI_k\) in \cref{lemma:conditional_exponential_tail_bound}, respectively.
\[
\E[\1\{ F(\bv)-F(\bv; \calI_k) \geq t \}] 
=\Pr(F(\bv)-F(\bv; \calI_k) \geq t) \leq \exp\left(- \frac{[-\ln(1-m/n)]n}{\sigma} t\right)
\]
\[
\E[\1\{ G(\bv)-G(\bv; \calI_k) \geq t \}] 
=\Pr(G(\bv)-G(\bv; \calI_k) \geq t) \leq \exp\left(- \frac{[-\ln(1-m/n)]n}{\sigma} t\right)
\]
\(\calI_k\) is independent of \(\calG_k= \sigma(\calI_0, \calB_0, \ldots,\calI_{k-1}, \calB_{k-1})\), 
since \(\calI_k\) is indepedndent of prior \(\calI_0, \calB_0, \ldots,\calI_{k-1}, \calB_{k-1}\).
Since \(\bw_k\) is \(\calG_k\)-measurable, then we show the following identities 
by substituting fixed value \(x\gets \bv\),  \(\calG \gets \calG_k, X \gets \bw_k\), \(Z \gets \calI_k; h(X, Z) \gets \1\{ F(\bw_k)-F(\bw_k; \calI_k) \geq t \}, \1\{ G(\bw_k)-G(\bw_k; \calI_k) \geq t \}\) in \cref{lemma:conditional_expectation_independent_rv}, respectively.
\begin{eqnarray*}
    \E[ \1\{ F(\bw_k)-F(\bw_k; \calI_k) \geq t \} \mid \calG_k] &=& \E[ \1\{ F(\bv)-F(\bv; \calI_k) \geq t \} ]_{\bv = \bw_k} \\
    \E[ \1\{ G(\bw_k)-G(\bw_k; \calI_k) \geq t \} \mid \calG_k] &=& \E[ \1\{ G(\bv)-G(\bv; \calI_k) \geq t \} ]_{\bv = \bw_k}
\end{eqnarray*}
Combining the above identities, and the exponential bounds for any fixed value \(\bv\).
\begin{eqnarray*}
\Pr\left(F(\bw_k)-F(\bw_k; \calI_k) \geq t \mid \calG_k\right) &=&
\E[ \1\{ F(\bw_k)-F(\bw_k; \calI_k) \geq t \} \mid \calG_k] \leq \exp\left(- \frac{[-\ln(1-m/n)]n}{\sigma} t\right),\\\
\Pr\left(G(\bw_k)-G(\bw_k; \calI_k) \geq t \mid \calG_k\right) &=& 
\E[ \1\{ G(\bw_k)-G(\bw_k; \calI_k) \geq t \} \mid \calG_k] \leq \exp\left(- \frac{[-\ln(1-m/n)]n}{\sigma} t\right).
\end{eqnarray*}
\end{proof}

\begin{lemma}[Upper Bound for Sum of Random Variables with Conditional Exponential Tail]\label{lemma:upper_bound_average_rv_conditional_exponential_tail}
Let nonnegative random variables \((Y_k)_{k\in \mathbb{Z}_{\geq 0}}\)
 and a sequence of indicator random variables \((\1_k)_{k\in \mathbb{Z}_{\geq 0}}\) be adapted to a fitlaration 
\((\calG_{k+1})_{k\in \mathbb{Z}_{\geq -1}}\) with \(\calG_{0} = \{\emptyset, \Omega\}\), 
namely \(Y_k, \1_k\) are \(\calG_{k+1}\)-measurable, and have the following one sided conditional tail bound with some \(C>0\)
for any \(t \geq 0\) and any \(k\in \mathbb{Z}_{\geq 0}\):
\[
\Pr(Y_k \geq t \mid \calG_{k}) \leq \exp\left(-\frac{t}{C}\right)
\]
Then for \([K] \equiv\{0, 1, 2, \ldots, K-1\}\), with probability at least \(1 - \delta\) for some \(\delta \in (0, 1)\):
\[
\sum_{k\in[K]} Y_{k}  \1_k \leq C \ln\frac{K}{\delta} \sum_{k\in[K]} \1_k
\]
For \(\calS \equiv \{k\in[K] \mid \1_k = 1\}\), 
hence \(|\calS|=\sum_{k\in[K]} \1_k\), with a ratio \(\kappa := |\calS|/K \in[0, 1]\) 
and such a convention of \(\frac{1}{|\calS|}\cdot \sum_{k\in[K]} \1_k =1, \ln\frac{1}{|\calS|}\cdot \sum_{k\in[K]} \1_k = 0\)
when \(|\calS|=0\), 
the following inequalities hold:
\[
\sum_{k\in[K]} Y_{k}  \1_k 
\leq 2C \ln\frac{2/\kappa}{\delta} \sum_{k\in[K]} \1_k, \qquad
\sum_{k \in [K]} Y_{k}  \1_k 
\leq 2C \left(
    \ln\frac{1}{\delta} + \ln\frac{16}{\kappa} \sum_{k\in[K]} \1_k
    \right) 
\]
\end{lemma}

\begin{remark}[Remark for \cref{lemma:upper_bound_average_rv_conditional_exponential_tail}]
By using the 2nd inequality and distributing the constant in logarithm, we have:
\[
\sum_{k\in[K]} Y_{k}  \1_k 
\leq 2C \ln\frac{4}{\delta} \sum_{k\in[K]} \1_k + 2C \ln\frac{1}{2\kappa} \sum_{k\in[K]} \1_k
\]
By using the 3rd inequality, we have the following bound (see its derivation after the proof of the 3rd inequality):
\[
\sum_{k\in[K]} Y_{k}  \1_k 
\leq 2C \ln\frac{4}{\delta} \sum_{k\in[K]} \1_k + 2C \ln\frac{4}{\delta} \cdot K
\]
\end{remark}

\begin{proof}[Proof of \cref{lemma:upper_bound_average_rv_conditional_exponential_tail}]
By letting \(Z_k := Y_k / C\), we have \(\Pr(Z_k\geq t \mid \calG_{k-1}) \leq \exp(-t)\) for any \(t \geq 0\) and any \(k\in \mathbb{Z}_{\geq 0}\).
We give the proofs for these three inequalities separately for \(Z_k\), then we have these inequalities for \(Y_k = C Z_k\) automatically.

\textbf{Proof of 1st inequality (using maximal inequality and union bound)}
The first inequliaty holds trivially if \(\sum_{k\in[K]} \1_k =0\), otherwise, we have \(\sum_{k\in[K]} \1_k > 0\),
then we have the following maximal inequality for nonnegative random variables \(Z_k\):
\[
\sum_{k\in[K]} Z_k \1_k \Bigg/ \sum_{k\in[K]} \1_k \leq \max_{k\in[K], \1_k = 1} Z_k 
\leq \max_{k\in[K]} Z_k 
\]
By using \(\Pr(Z_k \geq t) = \E[\E[\1_{Z_k \geq t} \mid \calG_{k}]] = \E[\Pr(Z_k \geq t \mid \calG_{k})] \leq\exp(-t), \forall t\geq 0\),
for some fixed \(t = \ln\frac{K}{\delta} \geq 0\), the probability of \(\max_{k\in[K]} Z_k \geq t\) is bounded by \(\delta\)
by taking the union bound over all \(k\in[K]\).
\[
\Pr\left(\sum_{k\in[K]} Z_k \1_k > t \sum_{k\in[K]} \1_k\right) 
\leq \Pr\left(\max_{k\in[K]} Z_k \geq t\right) = \Pr(Z_k \geq t, \forall k\in[K]) \leq \sum_{k\in[K]} \Pr(Z_k \geq t) = K \exp(-t) = \delta
\]
By substituting \(t = \ln\frac{K}{\delta}\), we complete the proof of the first inequality for \(Z_k\) with probability at least \(1 - \delta\)

\textbf{Proof of 2nd inequality (using AM-GM inequality and Markov's inequality)}
The second inequality holds trivially if \(|\calS|=\sum_{k\in[K]} \1_k =0\) under the convention of \(|\calS|\cdot \sum_{k\in[K]} \1_k =1, \ln\frac{1}{|\calS|}\cdot \sum_{k\in[K]} \1_k = 0\), 
otherwise, we have \(|\calS|=\sum_{k\in[K]} \1_k > 0\), then by applying the AM-GM inequality for some \(\lambda >0\), we have:
\[
\exp\left(\frac{\lambda}{|\calS|} \sum_{k \in[K]} Z_k \1_k\right) \leq \frac{1}{|\calS|}\sum_{k \in[K]} \exp(\lambda Z_k) \1_k \leq \frac{1}{|\calS|}\sum_{k \in[K]} \exp(\lambda Z_k)
\]
Noting that for any nonnegative random variable \(X\),
 \(\E[X\mid \calF] =\int_{0}^1 + \int_{1}^\infty \Pr(X>u\mid \calF) \mathd u\).
Since \(Z_k= Y_k/C\) is nonnegative, we have \(\Pr(\exp(\lambda Z_k) > u \mid \calG_{k-1}) =1\) for any \(u\in[0, 1], \lambda >0\),
and letting \(u=\exp(\lambda t)\) in the second integral and using the fact that \(\Pr(Z_k\geq t \mid \calG_{k-1}) \leq \exp(-t)\), 
we have the following inequality for any \(k\in \mathbb{Z}_{\geq 0}, \lambda \in (0, 1)\):
\[
\E[\exp(\lambda Z_k) \mid \calG_{k}] = \int_{0}^1 1 \mathd u
+ \int_{1}^\infty \Pr(\exp(\lambda Z_k) > \exp(\lambda t) \mid \calG_{k}) \mathd (\exp(\lambda t))
\leq 1 + \lambda \int_0^\infty \exp(-[1-\lambda]t) \mathd t = \frac{1}{1-\lambda}
\]
By tower property of conditional expectation, \(\E[\exp(\lambda Z_k)] = \E[\E[\exp(\lambda Z_k) \mid \calG_{k}]] \leq \frac{1}{1-\lambda}\),
then by applying Markov's inequality for \(\delta \in (0, 1), t := \frac{K}{(1-\lambda)\delta}\), we have the following bound:
\[
\Pr\left(\sum_{k\in[K]} \exp(\lambda Z_k) > t \right) \leq \E\left[\sum_{k\in[K]} \exp(\lambda Z_k) \right]/t \leq \frac{K}{(1-\lambda) t} = \delta
\]
Therefore, we have the following inequality with probability at least \(1 - \delta\) by substituting \(t:=\frac{K}{(1-\lambda)\delta}, \kappa:=|\calS|/K\):
\[
\sum_{k\in[K]} Z_k\1_k \leq \frac{|\calS|}{\lambda} \ln\left(\frac{1}{|\calS|} \sum_{k\in[K]} \exp(\lambda Z_k) \right)
\leq \frac{\ln (t/|\calS|)}{\lambda} \cdot |\calS|  = \frac{\ln \frac{2/\kappa}{\delta} - \ln(2(1-\lambda))}{\lambda} \sum_{k\in[K]} \1_k
\]
By selecting \(\lambda = \frac{1}{2}\), we complete the proof of the second inequality for \(Z_k\) with probability at least \(1 - \delta\).

\textbf{Proof of 3rd inequality (using supermartingale and combinatorics)}
The third inequality holds trivially if \(|\calS|=\sum_{k\in[K]} \1_k =0\) 
under the convention of \(|\calS|\cdot \sum_{k\in[K]} \1_k =1, \ln\frac{1}{|\calS|}\cdot \sum_{k\in[K]} \1_k = 0\),
otherwise, we have \(|\calS|=\sum_{k\in[K]} \1_k > 0\),
then for non-empty outcomes \(\calA_j\) of the random set \(\calS \subseteq [K]\) 
such that \(|\calS| = j\) for \(j\in \{1, 2, \ldots, K\}\),
\[
\calA_j := \{ A_j \subseteq [K] \mid A_j \neq \emptyset, |A_j| = j \}, \qquad |\calA_j| = \binom{K}{j}
\leq \left(\frac{\mathe K}{j}\right)^j = \exp\left(j \ln\frac{\mathe}{j/K}\right)
\]
For any deterministic set \(A_j\in \calA_j\), we introduce the corresponding \((M_k(A_j))_{k\in \{0, 1, \dots, K\}}\) with \(M_{0}(A_j) = \exp(0) =1,
M(A_j)\equiv M_K(A_j)\) and 
\(M_k(A_j) := \exp\left(\lambda \sum_{k'\in[k]} Z_{k'}\1_{k'\in A_j} + \ln(1-\lambda)\sum_{k'\in[k]} \1_{k'\in A_j} \right)\),
we have \(M_k(A_j)\) is a supermartingale because for any \(k\in [K]\), (see Theorem 4.2.4 on page 189 of~\cite{durrett2019probability}):
\[
\E[M_{k+1}(A_j) \mid \calG_k] =\left[(1-\lambda)\E[\exp(\lambda Z_{k}) \mid \calG_k]\right]^{\1_{k\in A_j}} \cdot M_k(A_j) \leq M_k(A_j)
\]
By using the tower property of conditional expectation, we have the following inequality for any \(k\in [K]\):
\[
\E[M(A_j)] = \E[M_K(A_j)] \leq \ldots \leq \E[M_{k+1}(A_j)] \leq \E[M_k(A_j)] \leq \E[M_0(A_j)] = M_0(A_j) = 1
\]

By applying the Markov's inequality for \(M(A_j)=\exp(\lambda \sum_{k\in A_j} Z_k + |A_j| \ln(1-\lambda))\) 
with some \(\delta_j \in (0, 1), t_j := \ln\frac{1}{\delta_j}\):
\[
\Pr\left(\lambda \sum_{k\in A_j} Z_k + |A_j| \ln(1-\lambda) > t_j\right) \leq \exp(-t_j) \E[M(A_j)] \leq \exp(-t_j) = \delta_j,
\quad \forall A_j\in \calA_j
\]
Therefore, by selecting \(\delta_j\) such that \(\delta/(j(j+1)) = |\calA_j|\delta_j = \binom{K}{j}\delta_j, \forall j\in \{1, 2, \ldots, K\}\)
and taking the union bound over all possible outcomes \(A_j\in \calA_j\), then for any fixed \(\lambda \in (0, 1)\):
\begin{eqnarray*}
& &\Pr\left( \lambda \sum_{k\in \calS} Z_k + |\calS| \ln(1-\lambda) > t_{|\calS|}, |\calS| = j\right) 
= \Pr\left(
    \bigcup_{A_j\in \calA_j} \left\{ \lambda \sum_{k\in A_j} Z_k + |A_j| \ln(1-\lambda) > t_j\right\}
    \cap \{\calS = A_j\} \right)\\
&\leq& \sum_{A_j\in\calA_j} \Pr\left( \lambda \sum_{k\in A_j} Z_k + |A_j| \ln(1-\lambda) > t_j\right) 
\leq |\calA_j| \cdot \delta_j = \frac{\delta}{j(j+1)} = \delta \left(\frac{1}{j} - \frac{1}{j+1}\right)
\end{eqnarray*}
Therefore, for any fixed \(\lambda \in (0, 1)\):
\[
\Pr\left(
    \lambda \sum_{k\in \calS} Z_k + |\calS| \ln(1-\lambda) > t_{|\calS|}
\right)
= \sum_{j=1}^K \Pr\left(
    \lambda \sum_{k\in \calS} Z_k + |\calS| \ln(1-\lambda) > t_{|\calS|}, |\calS| = j
\right)
\leq \sum_{j=1}^K \delta \left(\frac{1}{j} - \frac{1}{j+1}\right) < \delta
\]
By definitions of \(\calS, \kappa\equiv\frac{|\calS|}{K}\),
noting that \(t_{|\calS|} = \ln \frac{1}{\delta_{|\calS|}} 
= \ln\frac{1}{\delta}+\ln(|\calS|(1+|\calS|))+\ln|\calA_{|\calS|}|
\leq \ln\frac{1}{\delta}+|\calS|+|\calS|\ln \frac{\mathe}{\kappa}\) and 
\(|\calS|=\sum_{k\in[K]} \1_k\), 
we have the following inequality with probability at least \(1 - \delta\) for any fixed \(\lambda \in (0, 1)\):
\[
\sum_{k\in[K]} Z_k \1_k \leq \frac{1}{\lambda}
\left(t_{|\calS|} + \ln\frac{1}{1-\lambda} \sum_{k\in[K]} \1_k \right)
\leq \frac{1}{\lambda}\left(\ln\frac{1}{\delta} + \ln\frac{\mathe^2/(1-\lambda)}{\kappa} \sum_{k\in[K]} \1_k \right)
\]
By selecting \(\lambda = \frac{1}{2}\), 
then \(\mathe^2/(1-\lambda) = 2\mathe^2 < 16\), we complete the proof for \(Z_k\) with probability at least \(1 - \delta\).
\end{proof}

\begin{proof}[Proof of Remark for \cref{lemma:upper_bound_average_rv_conditional_exponential_tail}]
The latter bound in the remark holds trivially if \(\sum_{k\in[K]} \1_k =0\), 
otherwise, we have \(\sum_{k\in[K]} \1_k > 0\), and 
noting that \(\sum_{k\in[K]} \1_k + K = (1+\kappa) K, \sum_{k\in[K]} \1_k = \kappa K\) 
from the definition of \(\kappa\), the third inequality in \cref{lemma:upper_bound_average_rv_conditional_exponential_tail} becomes:
\[
\sum_{k\in[K]} Y_k \1_k \leq 
2C \left(\ln\frac{1}{\delta} + \ln\frac{16}{\kappa} \sum_{k\in[K]} \1_k \right)
= 2C\left(\frac{1}{(1+\kappa)K}\ln\frac{1}{\delta} + \frac{\kappa}{1+\kappa}\ln\frac{16}{\kappa}\right)
\left(\sum_{k\in[K]} \1_k + K\right)
\]
By using \(\frac{1}{(1+\kappa)K} \leq 1\) and \(\frac{\kappa}{1+\kappa} \ln\frac{16}{\kappa}
\leq \max_{\kappa\in[0, 1]} \frac{\kappa}{1+\kappa} \ln\frac{16}{\kappa} = \ln 4\), 
and noting \(\ln\frac{1}{\delta} + \ln 4 = \ln\frac{4}{\delta}\), 
we establish the latter bound in the remark:
\(
\sum_{k\in[K]} Y_k \1_k \leq 2C\ln\frac{4}{\delta} \sum_{k\in[K]} \1_k + 2C\ln\frac{4}{\delta} \cdot K
\).
\end{proof}

\newpage
\section{Proof for Result of Federated Learning with Full Participation}\label{sup:fedfull}



\begin{proof}[Proof for \Cref{theorem:convergence_gd_softmax_fedfull_lipschitz} of \Cref{alg:switching_gd_softmax_fedfull}]
    By applying \cref{lemma:polarization} to the update rule \(\bw_{k+1} = \bw_k - \eta \mathbf{u}_k\) with step size \(\eta>0\), we have
    \[
    \langle \mathbf{u}_k, \mathbf{w}_{k} - \mathbf{w}^\ast \rangle
         =  \frac{1}{2\eta} \left(\|\mathbf{w}_k - \mathbf{w}^\ast\|^2 - \|\mathbf{w}_{k+1} - \mathbf{w}^\ast\|^2 \right) 
         + \frac{\eta}{2} \|\bu_k\|^2
    \]
    For brevity, we write \(\1_k := \1_{G_k(\bw_k, \bxi_k) \leq \frac{\epsilon}{2}}\)
    with \(G_k(\bw_k, \bxi_k) = 
    m(\bg(\bw_k, \bxi_k), \alpha)\).
    The local direction \(\mathbf{u}_k\) is given by the following equation 
    and local updates are given by \(\bw_{k,\tau+1}^{(i)} = \bw_{k,\tau}^{(i)} - \gamma 
    \left[\nabla f_i (\bw_{k,\tau}^{(i)}, \bzeta_{k,\tau}^{(i)}) \1_k + \nabla g_i (\bw_{k,\tau}^{(i)}, \bzeta_{k,\tau}^{(i)}) [1-\1_k] \right]\):
    \begin{eqnarray*}
        \bu_k^{(i)}
        = \frac{\bw_{k, 0}^{(i)} - \bw_{k, E}^{(i)}}{\gamma E}
        = \frac{1}{E}\sum_{\tau\in [E]}
        \nabla f_i (\bw_{k,\tau}^{(i)}, \bzeta_{k,\tau}^{(i)})
        \1_k
        + \frac{1}{E}\sum_{\tau\in [E]}
        \nabla g_i (\bw_{k,\tau}^{(i)}, \bzeta_{k,\tau}^{(i)})
        [1-\1_k]
    \end{eqnarray*}
    where \(\bzeta_{k,\tau}^{(i)}\) is the sample at the \(\tau\)-th local update step of the \(k\)-th epoch for the \(i\)-th client,
    and \(\bw_{k, 0}^{(i)} = \bw_k\) is the initial local parameters for the \(i\)-th client.
    The direction \(\mathbf{u}_k\) is given by the following equation with 
    the brevity notations \(\bp_k := \softmax(\alpha\bff(\bw_k, \bxi_k))\) and \(\bq_k := \softmax(\alpha\bg(\bw_k, \bxi_k))\),
    and we introduce \(\br_k := \bp_k \1_k + \bq_k [1-\1_k]\):
    \begin{eqnarray*}
        \bu_k
        = \sum_{i\in \calI} ([\bp_k]_i \1_k + [\bq_k]_i [1-\1_k]) \bu_k^{(i)} = \sum_{i\in \calI} [\br_k]_i \bu_k^{(i)}
    \end{eqnarray*}
    To help the analaysis, we introduce the following notations:
    \begin{eqnarray*}
        \btu_k^{(i)} = \frac{1}{E}\sum_{\tau\in [E]}
        \nabla f_i (\bw_{k,\tau}^{(i)})
        \1_k
        + \frac{1}{E}\sum_{\tau\in [E]}
        \nabla g_i (\bw_{k,\tau}^{(i)})
        [1-\1_k],\quad
        \bbu_k^{(i)} = \nabla f_i (\bw_k) \1_k + \nabla g_i (\bw_k) [1-\1_k]
    \end{eqnarray*}

    \textbf{Step 1: Upperbound \(\eta \sum\limits_{k\in[K]} \|\bu_k\|^2\) and Lowerbound \(\sum\limits_{k\in[K]} \langle \mathbf{u}_k, \bw_k-\bw^\ast\rangle\)}

    Then we can decompose the direction as \(\bu_k^{(i)} 
    = \btu_k^{(i)} + (\bu_k^{(i)} - \btu_k^{(i)})
    = \bbu_k^{(i)} - (\bbu_k^{(i)} - \btu_k^{(i)}) - (\btu_k^{(i)} - \bu_k^{(i)})\).

    \textbf{1.1: upperbound of \(\eta \sum_{k\in[K]} \|\bu_k\|^2\)}

    Then by using \(\frac{1}{2}\|\bx - \by\|^2 \leq \|\bx\|^2 + \|\by\|^2, \E[X]^2 \leq \E[X^2]\) 
    and the assumption of Lipschitz continuity~\cref{assumption:lipschitz_continuity} (\(\|\nabla f_i(\bw)\| \leq L, \|\nabla g_i(\bw)\| \leq L\),
    therefore \(\|\btu_k^{(i)}\| \leq L, \left\| \sum_{i\in \calI} [\br_k]_i \btu_k^{(i)} \right\| \leq L\) 
    using \(\sum_{i\in \calI} [\br_k]_i = 1\) and all \([\br_k]_i \geq 0\)):
    \begin{eqnarray*}
        \frac{\eta}{2} \sum_{k\in[K]} \|\bu_k\|^2 
        = \frac{\eta}{2} \sum_{k\in[K]} \left\|
            \sum_{i\in \calI} [\br_k]_i \btu_k^{(i)}
            + \sum_{i\in \calI} [\br_k]_i (\bu_k^{(i)} - \btu_k^{(i)})
         \right\|^2
        \leq \eta L^2 K
        + \eta \sum_{k\in [K]} \sum_{i\in \calI} [\br_k]_i \|\bu_k^{(i)} - \btu_k^{(i)}\|^2
    \end{eqnarray*}

    \textbf{1.2: decomposition of inner product \(\sum_{k\in[K]} \langle \mathbf{u}_k, \bw_k-\bw^\ast\rangle\)}
    
    We decompose the inner product by
    introducing the Bregman divergence
    \(D_{f_i}[\bw'||\bw]:= f_i(\bw')-f_i(\bw)-\langle \nabla f_i(\bw), \bw'-\bw \rangle\geq 0,
    D_{g_i}[\bw'||\bw]:= g_i(\bw')-g_i(\bw)-\langle \nabla g_i(\bw), \bw'-\bw \rangle \geq 0\)
    for convex functions \(f_i\) and \(g_i\) from~\cref{assumption:convexity},
    therefore \( \langle \mathbf{u}_k, \bw_k-\bw^\ast\rangle = (f_i(\bw_k)-f_i(\bw^\ast)+D_{f_i}[\bw^\ast||\bw_k])\1_k + (g_i(\bw_k)-g_i(\bw^\ast)+D_{g_i}[\bw^\ast||\bw_k])[1-\1_k]\):
    \begin{eqnarray*}
        & &\sum_{k\in[K]} \langle \mathbf{u}_k, \bw_k-\bw^\ast\rangle\\
        & = &
        \sum_{k\in[K]} \sum_{i\in \calI} [\br_k]_i \langle \bbu_k^{(i)}, \bw_k-\bw^\ast\rangle 
       -\sum_{k\in[K]} \sum_{i\in \calI} [\br_k]_i \langle \bbu_k^{(i)} - \btu_k^{(i)}, \bw_k-\bw^\ast\rangle
          -\sum_{k\in[K]} \sum_{i\in \calI} [\br_k]_i \langle \btu_k^{(i)} - \bu_k^{(i)}, \bw_k-\bw^\ast\rangle\\
       & = & \sum_{k\in[K]} \sum_{i\in \calI} [\bp_k]_i (f_i(\bw_k) - f_i(\bw^\ast))\1_k 
          + \sum_{k\in[K]} \sum_{i\in \calI} [\bq_k]_i (g_i(\bw_k) - g_i(\bw^\ast))[1-\1_k]\\
       & &- \sum_{k\in[K]} \sum_{i\in \calI} [\bp_k]_i \frac{1}{E}\sum_{\tau\in[E]}
       \left[-D_{f_i}[\bw^\ast||\bw_k]
       + \langle \nabla f_i(\bw_k) - \nabla f_i(\bw_{k,\tau}^{(i)}), \bw_k-\bw^\ast\rangle\right] \1_k\\
       & &- \sum_{k\in[K]} \sum_{i\in \calI} [\bq_k]_i \frac{1}{E}\sum_{\tau\in[E]}
       \left[- D_{g_i}[\bw^\ast||\bw_k]
       + \langle \nabla g_i(\bw_k) - \nabla g_i(\bw_{k,\tau}^{(i)}), \bw_k-\bw^\ast\rangle\right][1-\1_k]\\
       & &- \sum_{k\in[K]} \sum_{i\in \calI} [\br_k]_i
       \langle \btu_k^{(i)} - \bu_k^{(i)}, \bw_k-\bw^\ast\rangle
    \end{eqnarray*}

    \textbf{1.3: lowerbound of 1st and 2nd terms in the decomposition of inner product \(\sum_{k\in[K]} \langle \mathbf{u}_k, \bw_k-\bw^\ast\rangle\)}
    
    Regarding the first and second terms in the above equation for the inner product,
    using the definition of \(F = \max_{i\in \calI} f_i\) and \(G = \max_{i\in \calI} g_i\), 
    from the definition of \(\bw^\ast\) such that \(G(\bw^\ast) \leq 0\) in~\cref{eq:opt_sol}.
    \[\sum_{i\in \calI} [\bp_k]_i f_i(\bw^\ast) \leq F(\bw^\ast),\quad 
    \sum_{i\in \calI} [\bq_k]_i g_i(\bw^\ast) \leq G(\bw^\ast) \leq 0\]

    Applying the deviation bound of softmax mean \cref{lemma:deviation_bound_softmax_mean} 
    with \(\alpha \geq \frac{2 \ln n}{\epsilon'}\) and substituting \(\bx \gets \bff(\bw_k), \bdelta \gets \bff(\bw_k, \bxi_k) - \bff(\bw_k)\)
    and \(\bx \gets \bg(\bw_k), \bdelta \gets \bg(\bw_k, \bxi_k) - \bg(\bw_k)\), then we have
    \[\sum_{i\in\calI} [\bp_k]_i f_i(\bw_k) \geq F(\bw_k) - 2\|\bff(\bw_k, \bxi_k) - \bff(\bw_k)\|_\infty - \frac{\epsilon'}{2}\] 
    \[\sum_{i\in\calI} [\bq_k]_i g_i(\bw_k) [1-\1_k] 
    \geq \left[G_k(\bw_k, \bxi_k) - \frac{\epsilon}{2}\right] [1-\1_k]
    + \left(\frac{\epsilon}{2} - \|\bg(\bw_k, \bxi_k) - \bg(\bw_k)\|_\infty\right) [1-\1_k]\]
    Combining the above, we have the following lower bound for the first and second terms in the equation for the inner product:
    \begin{eqnarray*}
        &   & \sum_{k\in[K]} \sum_{i\in \calI} [\bp_k]_i (f_i(\bw_k) - f_i(\bw^\ast))\1_k 
        + \sum_{k\in[K]} \sum_{i\in \calI} [\bq_k]_i (g_i(\bw_k) - g_i(\bw^\ast))[1-\1_k]\\
        & \geq & \sum_{k\in[K]} [F(\bw_k) - F(\bw^\ast)] \1_k 
        + \sum_{k\in[K]} \left[G_k(\bw_k, \bxi_k) - \frac{\epsilon}{2}\right] [1-\1_k]
        + \frac{\epsilon K}{2}
        - \frac{\epsilon' + \epsilon}{2} \sum_{k\in[K]} \1_k \\
        & & - 2 \sum_{k\in[K]} \|\bff(\bw_k, \bxi_k) - \bff(\bw_k)\|_\infty \1_k
        - \sum_{k\in[K]} \|\bg(\bw_k, \bxi_k) - \bg(\bw_k)\|_\infty [1-\1_k]
    \end{eqnarray*}
    \textbf{1.4: upperbound of 3rd and 4th terms in the decomposition of inner product \(\sum_{k\in[K]} \langle \mathbf{u}_k, \bw_k-\bw^\ast\rangle\)}
    
    Regarding the third term and fourth term in the equation for the inner product,
    we noting the three-point Bregman divergence identity \cref{lemma:three_point_bregman_divergence_identity} 
    by substituting \(\psi \gets f_i, g_i\), \(\bx \gets \bw_k\), \(\bx' \gets \bw^\ast\), \(\hat{\bx} \gets \bw_{k,\tau}^{(i)}\),
    and using the fact that \(D_{f_i}[\cdot||\cdot] \geq 0, D_{g_i}[\cdot||\cdot] \geq 0\) from the definition of Bregman divergence
    for convex functions \(f_i\) and \(g_i\) from~\cref{assumption:convexity}.
    \[
    -D_{f_i}[\bw^\ast||\bw_k] + \langle \nabla f_i(\bw_k) - \nabla f_i(\bw_{k,\tau}^{(i)}), \bw_{k,\tau}^{(i)}-\bw^\ast\rangle
    =
    - D_{f_i}[\bw^\ast||\bw_{k,\tau}^{(i)}] - D_{f_i}[\bw_{k,\tau}^{(i)}||\bw_k] \leq 0
    \]
    \[
    -D_{g_i}[\bw^\ast||\bw_k] + \langle \nabla g_i(\bw_k) - \nabla g_i(\bw_{k,\tau}^{(i)}), \bw_{k,\tau}^{(i)}-\bw^\ast\rangle
    =
    - D_{g_i}[\bw^\ast||\bw_{k,\tau}^{(i)}] - D_{g_i}[\bw_{k,\tau}^{(i)}||\bw_k] \leq 0
    \]
    Using \((\bw_k - \bw_{k,\tau}^{(i)})\1_k = \gamma \sum_{\tau'\in[\tau]} \nabla f_i(\bw_{k,\tau'}^{(i)}, \bzeta_{k,\tau'}^{(i)})\1_k\)
    and \((\bw_k - \bw_{k,\tau}^{(i)})[1-\1_k] = \gamma \sum_{\tau'\in[\tau]} \nabla g_i(\bw_{k,\tau'}^{(i)}, \bzeta_{k,\tau'}^{(i)})[1-\1_k]\) from the update rule, 
    and applying Abel's summation formula \(\sum_{\tau\in[E]}\langle \ba_\tau, \bB_\tau \rangle 
    = \langle \bA_E, \bB_{E-1} \rangle - \sum_{\tau\in[E-1]} \langle \bA_{\tau+1}, \bb_\tau \rangle\) 
    with \(\ba_0=\vec{0}, \bA_\tau := \sum_{\tau'\in[\tau]} \ba_{\tau'}, \bB_\tau := \sum_{\tau'\in[\tau]} \bb_{\tau'}\)
    and substituting \(\ba_\tau \gets \nabla f_i(\bw_k) -\nabla f_i(\bw_{k,\tau}^{(i)}), 
    \bb_\tau \gets \nabla f_i(\bw_{k,\tau}^{(i)}, \bzeta_{k,\tau}^{(i)}) - \nabla f_i(\bw_{k,\tau}^{(i)})\), then
    again applying Cauchy-Schwarz inequality and Lipschitz continuity~\cref{assumption:lipschitz_continuity}: 
    \begin{eqnarray*}
    & &\frac{1}{E}\sum_{\tau\in[E]} \langle \nabla f_i(\bw_k) - \nabla f_i(\bw_{k,\tau}^{(i)}),\bw_k - \bw_{k,\tau}^{(i)}\rangle \1_k \\
    &=& 2\gamma L^2 \frac{1}{E} \sum_{\tau \in [E]} \sum_{\tau'\in [\tau]} 
    \left\langle \frac{\nabla f_i(\bw_k) -\nabla f_i(\bw_{k,\tau}^{(i)})}{2L}, 
    \frac{\nabla f_i(\bw_{k,\tau'}^{(i)})}{L}\right\rangle \1_k\\
    & & + 4\gamma L(\sigma_g/\sqrt{B_g}) \left\langle 
    \frac{1}{E}\sum_{\tau \in[E]} \frac{\nabla f_i(\bw_k) - \nabla f_i(\bw_{k,\tau}^{(i)})}{2L},
    \sum_{\tau \in [E-1]} \frac{\nabla f_i(\bw_{k,\tau}^{(i)}, \bzeta_{k,\tau}^{(i)}) - \nabla f_i(\bw_{k,\tau}^{(i)})}{2\sigma_g/\sqrt{B_g}}
    \right\rangle \1_k \\
    & & - 4 \gamma L (\sigma_g/\sqrt{B_g}) \sum_{\tau \in[E-1]}
    \tau \left\langle
    \frac{1}{\tau} \sum_{\tau'\in[\tau+1]} \frac{\nabla f_i(\bw_k) - \nabla f_i(\bw_{k,\tau'}^{(i)})}{2L},
    \frac{\nabla f_i(\bw_{k,\tau}^{(i)}, \bzeta_{k,\tau}^{(i)}) - \nabla f_i(\bw_{k,\tau}^{(i)})}{2\sigma_g/\sqrt{B_g}}
    \right\rangle \1_k\\
    & \leq & 2\gamma L^2 (E-1)\1_k
    + 4 \gamma L (\sigma_g/\sqrt{B_g}) \left\| \sum_{\tau \in [E-1]} \frac{\nabla f_i(\bw_{k,\tau}^{(i)}, \bzeta_{k,\tau}^{(i)}) - \nabla f_i(\bw_{k,\tau}^{(i)})}{2\sigma_g/\sqrt{B_g}}\right\| \1_k\\
    & & + 4 \gamma L (\sigma_g/\sqrt{B_g}) \sum_{\tau\in[E-1]} \tau \left\langle \be_{k, \tau}^{\prime(i)}, 
    \frac{\nabla f_i(\bw_{k,\tau}^{(i)}, \bzeta_{k,\tau}^{(i)}) - \nabla f_i(\bw_{k,\tau}^{(i)})}{2\sigma_g/\sqrt{B_g}}
    \right\rangle\1_k
    \end{eqnarray*}
    where the intermediate vector \(\be_{k, \tau}^{\prime(i)} := -\frac{1}{\tau} \sum_{\tau'\in[\tau+1]} \frac{\nabla f_i(\bw_k) - \nabla f_i(\bw_{k,\tau'}^{(i)})}{2L}\) 
    satisfies \(\|\be_{k, \tau}^{\prime(i)}\| \leq 1\) by the assumption of Lipschitz continuity~\cref{assumption:lipschitz_continuity} 
    and hence \(\|\nabla f_i(\bw_k) - \nabla f_i(\bw_{k,\tau}^{(i)})\| \leq 2L\).
    Similarly, by introducing the intermediate vector \(\be_{k, \tau}^{\prime\prime(i)} := -\frac{1}{\tau} \sum_{\tau'\in[\tau+1]} \frac{\nabla g_i(\bw_k) - \nabla g_i(\bw_{k,\tau'}^{(i)})}{2L}\) 
    satisfies \(\|\be_{k, \tau}^{\prime\prime(i)}\| \leq 1\) by the assumption of Lipschitz continuity~\cref{assumption:lipschitz_continuity} 
    and hence \(\|\nabla g_i(\bw_k) - \nabla g_i(\bw_{k,\tau}^{(i)})\| \leq 2L\), then:
    \begin{eqnarray*}
    & &\frac{1}{E}\sum_{\tau\in[E]} \langle \nabla g_i(\bw_k) - \nabla g_i(\bw_{k,\tau}^{(i)}),\bw_k - \bw_{k,\tau}^{(i)}\rangle [1-\1_k] \\
    & \leq & 2\gamma L^2 (E-1)[1-\1_k]
    + 4 \gamma L (\sigma_g/\sqrt{B_g}) \left\| \sum_{\tau \in [E-1]} \frac{\nabla g_i(\bw_{k,\tau}^{(i)}, \bzeta_{k,\tau}^{(i)}) - \nabla g_i(\bw_{k,\tau}^{(i)})}{2\sigma_g/\sqrt{B_g}}\right\|[1-\1_k]\\
    & & + 4 \gamma L (\sigma_g/\sqrt{B_g}) \sum_{\tau\in[E-1]} \tau \left\langle \be_{k, \tau}^{\prime\prime(i)}, \frac{\nabla g_i(\bw_{k,\tau}^{(i)}, \bzeta_{k,\tau}^{(i)}) - \nabla g_i(\bw_{k,\tau}^{(i)})}{2\sigma_g/\sqrt{B_g}}\right\rangle[1-\1_k]
    \end{eqnarray*}
    We introduce the following notations, \(\be_{k, \tau}^{(i)} := \be_{k, \tau}^{\prime(i)} \1_k + \be_{k, \tau}^{\prime\prime(i)} [1-\1_k]\)
    and \(\bz_{k, \tau}^{(i)} :=
    \frac{\nabla f_i(\bw_{k,\tau}^{(i)}, \bzeta_{k,\tau}^{(i)}) - \nabla f_i(\bw_{k,\tau}^{(i)})}{2\sigma_g/\sqrt{B_g}}
    \1_k + 
    \frac{\nabla g_i(\bw_{k,\tau}^{(i)}, \bzeta_{k,\tau}^{(i)}) - \nabla g_i(\bw_{k,\tau}^{(i)})}{2\sigma_g/\sqrt{B_g}}
    [1-\1_k]\).
    Summing over \(\tau\in[E]\) and \(i\in \calI\)
    and noting \(\sum_{i\in\calI} [\bp_k]_i = 1\) and \(\sum_{i\in\calI} [\bq_k]_i = 1\), 
    the upperbound of the 3rd and 4th terms in the decomposition of inner product \(\sum_{k\in[K]} \langle \mathbf{u}_k, \bw_k-\bw^\ast\rangle\) is given by:
    \begin{eqnarray*}
    & &\sum_{k\in[K]} \sum_{i\in \calI} [\bp_k]_i \frac{1}{E}\sum_{\tau\in[E]} 
    \left[
        -D_{f_i}[\bw^\ast||\bw_k]
        + \langle \nabla f_i(\bw_k) - \nabla f_i(\bw_{k,\tau}^{(i)}), \bw_k-\bw^\ast\rangle
    \right] \1_k \\
    &+&
    \sum_{k\in[K]} \sum_{i\in \calI} [\bq_k]_i \frac{1}{E}\sum_{\tau\in[E]} 
    \left[
        -D_{g_i}[\bw^\ast||\bw_k]
        + \langle \nabla g_i(\bw_k) - \nabla g_i(\bw_{k,\tau}^{(i)}), \bw_k-\bw^\ast\rangle
    \right] [1-\1_k]\\
    & \leq & 2\gamma L^2 (E-1) K 
    + \frac{4 \gamma L\sigma_g}{\sqrt{B_g}} \sum_{k\in[K]} \sum_{i\in \calI} [\br_k]_i \left\| \sum_{\tau\in[E-1]} \bz_{k, \tau}^{(i)} \right\|
    + \frac{4 \gamma L\sigma_g}{\sqrt{B_g}} \sum_{k\in[K]} \sum_{i\in \calI} [\br_k]_i \sum_{\tau\in[E-1]} 
    \tau \left\langle \be_{k, \tau}^{(i)}, \bz_{k, \tau}^{(i)} \right\rangle
    \end{eqnarray*}

    \textbf{1.5: rewriting the 5th term in the decomposition of inner product \(\sum_{k\in[K]} \langle \mathbf{u}_k, \bw_k-\bw^\ast\rangle\)}
    
    We introduce the the direction vector \(\bd_k := \frac{\bw^\ast - \bw_k}{D}\) 
    which satisfies \(\|\bd_k\| \leq 1\) by the assumption of finite diameter \(D\) 
    of the parameter space \(\Theta\)~\cref{assumption:diameter_parameter_space}.
    From the definition of \(\bz_{k, \tau}^{(i)}\) and \(\btu_k^{(i)}, \bu_k^{(i)}\), we have:
    \[
    \btu_k^{(i)} - \bu_k^{(i)} = - \frac{2\sigma_g}{\sqrt{B_g}} \cdot \frac{1}{E}\sum_{\tau\in[E]} \bz_{k, \tau}^{(i)} 
    \]
    Therefore, the 5th term in the decomposition of inner product \(\sum_{k\in[K]} \langle \mathbf{u}_k, \bw_k-\bw^\ast\rangle\) 
    can be rewritten as:
    \[
    \sum_{k\in[K]}\sum_{i\in \calI} [\br_k]_i \langle \btu_k^{(i)} - \bu_k^{(i)}, \bw_k - \bw^\ast \rangle
    = \frac{2D \sigma_g}{\sqrt{B_g}}\cdot \frac{1}{E}
    \sum_{k\in[K]}\sum_{i\in \calI} [\br_k]_i 
    \sum_{\tau\in[E]} \left\langle \bd_k, \bz_{k, \tau}^{(i)}\right\rangle
    \]

    \textbf{1.6: rewriting \(\|\bu_k^{(i)}-\btu_k^{(i)}\|^2\) in the upper bound of \(\frac{\eta}{2} \sum_{k\in[K]} \|\bu_k\|^2\)}
    By using the above equation with \(\bz_{k, \tau}^{(i)}\) and \(\btu_k^{(i)}, \bu_k^{(i)}\), 
    the term in the upper bound of \(\frac{\eta}{2} \sum_{k\in[K]} \|\bu_k\|^2\) can be rewritten as:
    \[
    \eta \sum_{k\in[K]} \sum_{i\in\calI} [\br_k]_i \|\bu_k^{(i)}-\btu_k^{(i)}\|^2
    = \frac{4\eta\sigma^2_g}{B_g}\cdot \frac{1}{E^2}
    \sum_{k\in[K]}\sum_{i\in \calI} [\br_k]_i
    \left\| \sum_{\tau\in[E]} \bz_{k, \tau}^{(i)} \right\|^2
    \]

    \textbf{Step 2: Upperbound \(\sum\limits_{k\in[K]} [F(\bw_k) - F(\bw^\ast)]\1_k\) and \(\sum\limits_{k\in[K]} G(\bw_k)\1_k\)}

    We introduce a filtration \((\calF_t)_{t\in\mathbb{Z}_{\geq -1}}\) to track the information up to time
    \(t := \text{ind}(k, i, \tau) \equiv k\cdot n E + i\cdot E + \tau\) and 
    \(\calF_t := \sigma((\bzeta_{k',\tau'}^{(i')})_{\text{ind}(k', i', \tau') < t}, 
    (\bxi_{k'})_{\text{ind}(k', 0, 0) \leq t})\) with \(\calF_{-1} = \{\emptyset, \Omega\}\).
    These introduced notations satisfy the following: 

    1. the indicator \(\1_k\equiv \1_{G_k(\bw_k)\leq \frac{\epsilon}{2}}\) is \(\calF_t\)-measurable; 
    (from the definition of \(G_k(\bw_k, \bxi_k) = m(\bg(\bw_k, \bxi_k), \alpha)\))

    2. the softmax weights \(\br_k\) such that \(\sum_{i\in\calI} [\br_k]_i = 1\) 
     and \([\br_k]_i \geq 0\) and is \(\calF_t\)-measurable; (from the defintion of \(\br_k\))
     \[
     \br_k \equiv \bp_k \1_k + \bq_k[1-\1_k] = 
     \softmax(\alpha \bff(\bw_k, \bxi_k)) \1_k + \softmax(\alpha \bg(\bw_k, \bxi_k))[1-\1_k]
     \]
    3. the direction vectors \(\bd_k\) and \(\be_k^{(i)}\) such that \(\|\bd_k\| \leq 1\) 
    (from the definition of \(\bd_k\) and the assumption of finite diameter of the parameter space \(\Theta\)~\cref{assumption:diameter_parameter_space})
    and \(\|\be_k^{(i)}\| \leq 1\) (from the definition of \(\be_k^{(i)}\) and the assumption of Lipschitz continuity~\cref{assumption:lipschitz_continuity})
     and are \(\calF_t\)-measurable; (from the definition of \(\bd_k\) and \(\be_k^{(i)}\))
     \[
     \bd_k \equiv \frac{\bw^\ast - \bw_k}{D},\quad
     \be_k^{(i)} \equiv \frac{1}{\tau}\sum_{\tau'=1}^\tau 
     \frac{\nabla f_i(\bw_{k,\tau'}^{(i)}) - \nabla f_i(\bw_k)}{2L} \1_k
     + \frac{1}{\tau}\sum_{\tau'=1}^\tau \frac{\nabla g_i(\bw_{k,\tau'}^{(i)}) - \nabla g_i(\bw_k)}{2L} [1-\1_k]
     \]
    4. the condtional 1-subgaussian random variable \(\bz_{k,\tau}^{(i)}\) 
    such that \(\E[\bz_{k,\tau}^{(i)}\mid \calF_t] = 0\) 
    and \(\E[\exp(\|\bz_{k,\tau}^{(i)}\|^2)\mid \calF_t] \leq 2\)
    and therefore
    \(\ln\E[\exp(\lambda\langle \be, \bz_{k,\tau}^{(i)}\rangle) \mid \calF_t]
    \leq \ln\E[\exp(\lambda\|\bz_{k,\tau}^{(i)}\|)\mid \calF_t] \leq \frac{\lambda^2}{2}, \forall \be\in \mathbb{S}^{d-1}, \forall\lambda \in \mathbb{R}\),
    and is \(\calF_{t+1}\)-measurable. (from the definition of \(\bz_{k,\tau}^{(i)}\),
    independence of all \(\bzeta_{k,\tau}^{(i)}\) and \(\bxi_{k}\equiv (\bxi_k^{(i)})_{i\in\calI}\) 
    and the assumption of sub-Gaussianity of stochastic gradients~\cref{assumption:subgaussianity_stochastic_estimates},
    and \cref{lemma:subgaussianity} for subgaussianity of random variables,
    \cref{lemma:average_subgaussian_random_vectors} for subgaussianity of the average of subgaussian random vectors)
    \[
    \bz_{k,\tau}^{(i)} \equiv 
    \frac{\nabla f_i(\bw_{k,\tau}^{(i)},\bzeta_{k,\tau}^{(i)}) -\nabla f_i(\bw_{k,\tau}^{(i)})}{2\sigma_g/\sqrt{B_g}}\1_k
    + \frac{\nabla g_i(\bw_{k,\tau}^{(i)},\bzeta_{k,\tau}^{(i)}) -\nabla g_i(\bw_{k,\tau}^{(i)})}{2\sigma_g/\sqrt{B_g}}[1-\1_k]
    \]

    \textbf{2.1: terms in the upperbound of \(\sum_{k\in[K]} [F(\bw_k) - F(\bw^\ast)]\1_k\)}

    By using the polarization identity \cref{lemma:polarization} and 
    the established lower bound of \(\eta \sum_{k\in[K]} \|\bu_k\|^2\) and 
    the lowerbound of \(\sum_{k\in[K]} \langle \mathbf{u}_k, \bw_k-\bw^\ast\rangle\) in the previous step,
    dropping the nonnegative term \(\frac{\|\bw_K-\bw^\ast\|^2}{2\eta}\),
    noting the assumption of finite diameter of the parameter space \(\Theta\)~\cref{assumption:diameter_parameter_space}
    then \(\|\bw_0-\bw^\ast\| \leq D\), 
    we rearrange the terms and obtain the following upperbound of \(\sum_{k\in[K]} [F(\bw_k) - F(\bw^\ast)]\1_k\):
    \begin{eqnarray*}
    & & \sum_{k\in[K]} [F(\bw_k) - F(\bw^\ast)]\1_k\\
    & \leq &
    \frac{\epsilon'+\epsilon}{2} \sum_{k\in[K]} \1_k
    - \sum_{k\in[K]} \left[G_k(\bw_k, \bxi_k) - \frac{\epsilon}{2}\right] [1-\1_k]
    - \frac{\epsilon K}{2}
    + \frac{D^2}{2\eta}
    + \eta L^2 K + 2 \gamma L^2 (E-1)K \\
    &+& 2 \sum_{k\in[K]} \|\bff(\bw_k, \bxi_k) - \bff(\bw_k) \|_\infty \1_k
    + \sum_{k\in[K]} \|\bg(\bw_k, \bxi_k) - \bg(\bw_k) \|_\infty [1-\1_k]\\
    &+& \frac{4\eta \sigma^2_g}{B_g} \cdot \frac{1}{E^2} 
    \sum_{k\in[K]}\sum_{i\in \calI} [\br_k]_i \left\| \sum_{\tau\in[E]} \bz_{k, \tau}^{(i)} \right\|^2
    + \frac{4\gamma L\sigma_g}{\sqrt{B_g}} 
    \sum_{k\in[K]} \sum_{i\in \calI} [\br_k]_i \left\| \sum_{\tau\in[E-1]} \bz_{k, \tau}^{(i)} \right\|\\
    &+& \frac{2D \sigma_g}{\sqrt{B_g}}\cdot \frac{1}{E} 
    \sum_{k\in[K]}\sum_{i\in \calI} [\br_k]_i \sum_{\tau\in[E]} \left\langle \bd_k, \bz_{k, \tau}^{(i)}\right\rangle
    + \frac{4 \gamma L\sigma_g}{\sqrt{B_g}} 
    \sum_{k\in[K]} \sum_{i\in \calI} [\br_k]_i \sum_{\tau\in[E-1]} \tau \left\langle \be_{k, \tau}^{(i)}, \bz_{k, \tau}^{(i)} \right\rangle
    \end{eqnarray*}

    By choosing the local step size \(\gamma = \frac{\eta}{E}\), 
    and applying triangle inequality \( \|\bx-\by\| \leq \|\bx\| +\|\by\|\), then:
    \begin{eqnarray*}
        \frac{4\gamma L\sigma_g}{\sqrt{B_g}} \sum_{k\in[K]} \sum_{i\in \calI} [\br_k]_i 
        \left\| \sum_{\tau\in[E-1]} \bz_{k, \tau}^{(i)} \right\|
        \leq \frac{4 \eta L \sigma_g}{\sqrt{B_g}} \frac{1}{E} \sum_{k\in[K]} \sum_{i\in \calI} [\br_k]_i \left\| \sum_{\tau\in[E]} \bz_{k, \tau}^{(i)} \right\|
        + \frac{4 \eta L \sigma_g}{\sqrt{B_g}} \frac{1}{E} \sum_{k\in[K]} \sum_{i\in \calI} [\br_k]_i \| \bz_{k, E-1}^{(i)} \|
    \end{eqnarray*}

    Noting \(\E[X]^2 \leq \E[X^2]\) and \(2\sqrt{x y} \leq x + y\) for any \(x, y \geq 0\), then we have:
    \begin{eqnarray*}
        &  &\frac{4\eta L\sigma_g}{\sqrt{B_g}} \cdot \frac{1}{E} 
        \sum_{k\in[K]} \sum_{i\in \calI} [\br_k]_i \left\| \sum_{\tau\in[E-1]} \bz_{k, \tau}^{(i)} \right\| 
        = \sum_{k\in [K]} 2\sqrt{\eta L^2 \times \frac{4\eta \sigma^2_g}{B_g} \cdot \frac{1}{E^2}}
        \sum_{i\in \calI} [\br_k]_i \left\| \sum_{\tau\in[E-1]} \bz_{k, \tau}^{(i)} \right\|\\
        &\leq & \sum_{k\in [K]} 2\sqrt{\eta L^2 \times \frac{4\eta \sigma^2_g}{B_g} \cdot \frac{1}{E^2}
        \sum_{i\in \calI} [\br_k]_i \left\| \sum_{\tau\in[E-1]} \bz_{k, \tau}^{(i)} \right\|^2}
        \leq
        \eta L^2 K + \frac{4\eta \sigma^2_g}{B_g} \cdot \frac{1}{E^2} \sum_{k\in[K]}\sum_{i\in \calI} [\br_k]_i \left\| \sum_{\tau\in[E-1]} \bz_{k, \tau}^{(i)} \right\|^2
    \end{eqnarray*}
    Hence, we can upperbound \(\sum_{k\in[K]} [F(\bw_k) - F(\bw^\ast)]\1_k\) as:
    \begin{eqnarray*}
        & & \sum_{k\in[K]} [F(\bw_k) - F(\bw^\ast)]\1_k\\
        &\leq &
        \frac{\epsilon'+\epsilon}{2} \sum_{k\in[K]} \1_k
        - \sum_{k\in[K]} \left[G_k(\bw_k, \bxi_k) - \frac{\epsilon}{2}\right] [1-\1_k]
        - \frac{\epsilon K}{2}
        + \frac{D^2}{2\eta}
        + 4\eta L^2 K\left(1-\frac{1}{2E}\right)\\
        &+& 2\sum_{k\in[K]} \|\bff(\bw_k, \bxi_k) - \bff(\bw_k) \|_\infty \1_k 
        + \sum_{k\in[K]} \|\bg(\bw_k, \bxi_k) - \bg(\bw_k) \|_\infty [1-\1_k]\\
        &+& \frac{8\eta \sigma^2_g}{B_g} \cdot \frac{1}{E^2} \sum_{k\in[K]}\sum_{i\in \calI} [\br_k]_i \left\| \sum_{\tau\in[E]} \bz_{k, \tau}^{(i)} \right\|^2\\
        &+& \frac{2D \sigma_g}{\sqrt{B_g}} \cdot \frac{1}{E} 
        \sum_{k\in[K]}\sum_{i\in \calI} [\br_k]_i \sum_{\tau\in[E]} \left\langle \bd_k, \bz_{k, \tau}^{(i)}\right\rangle
        + \frac{4 \eta L\sigma_g}{\sqrt{B_g}} \cdot \frac{1}{E}
        \sum_{k\in[K]} \sum_{i\in \calI} [\br_k]_i \sum_{\tau\in[E-1]} \tau \left\langle \be_{k, \tau}^{(i)}, \bz_{k, \tau}^{(i)} \right\rangle\\
        &+& \frac{4 \eta L \sigma_g}{\sqrt{B_g}} \cdot \frac{1}{E} \sum_{k\in[K]} \sum_{i\in \calI} [\br_k]_i \| \bz_{k, E-1}^{(i)} \|
    \end{eqnarray*}

    \textbf{2.2: terms in the upperbound of \(\sum_{k\in[K]} G(\bw_k)\1_k\)}

    Noting \(G_k(\bw_k, \bxi_k)\1_k = G_k(\bw_k, \bxi_k)\1_{G_k(\bw_k, \bxi_k)\leq\frac{\epsilon}{2}} \leq \frac{\epsilon}{2}\1_k\),
    applying \cref{lemma:softmax_mean} of softmax mean \(m(\bx, \alpha)\) with softmax hyperparameter \(\alpha\geq \frac{2 \ln n}{\epsilon'}\)
    then \(G(\bw_k,\bxi_k) - G_k(\bw_k,\bxi_k)
    = \max_{i\in\calI} [\bg(\bw_k, \bxi_k)]_i - m(\bg(\bw_k, \bxi_k), \alpha) \leq \frac{\epsilon'}{2}\),
    using
    \(G(\bw)-G(\bw_k, \bxi_k)=\max_{i\in \calI} [\bg(\bw_k)]_i - \max_{i\in \calI} [\bg(\bw_k, \bxi_k)]_i \leq 
    \max_{i\in \calI} [\bg(\bw_k) - \bg(\bw_k, \bxi_k)]_i \leq \|\bg(\bw_k, \bxi_k) - \bg(\bw_k)\|_\infty\):
    \begin{eqnarray*}
    \sum_{k\in[K]} G(\bw_k)\1_k
    &=& \sum_{k\in[K]} \left[G_k(\bw_k, \bxi_k) + [G(\bw_k, \bxi_k) - G_k(\bw_k, \bxi_k)] + [G(\bw_k) - G(\bw_k, \bxi_k)]\right]\1_k\\
    &\leq& \frac{\epsilon + \epsilon'}{2} \sum_{k\in[K]} \1_k + \sum_{k\in[K]} \|\bg(\bw_k, \bxi_k) - \bg(\bw_k)\|_\infty \1_k 
    \end{eqnarray*}

    \textbf{2.3: upperbounds of 
    \(\sum_{k\in[K]} \|\bff(\bw_k, \bxi_k) - \bff(\bw_k) \|_\infty \1_k\) and 
    \(\sum_{k\in[K]} \|\bg(\bw_k, \bxi_k) - \bg(\bw_k) \|_\infty \1_k, [1-\1_k]\)}

    By applying the established \cref{lemma:subgaussianity,lemma:average_subgaussian_random_vectors} 
    for subgaussianity of random variables and subgaussianity of the average of subgaussian random vectors
    with \cref{assumption:subgaussianity_stochastic_estimates},
    and \cref{lemma:maximal_inequality_subgaussianity}
    of maximal inequality of subgaussian random variables 
    with \(\bz_k\gets \frac{\bff(\bw_k, \bxi_k)-\bff(\bw_k)}{2\sigma_\zeta/\sqrt{B_\zeta}},
    \frac{\bg_k(\bw_k, \bxi_k)-\bg_k(\bw_k)}{2\sigma_\zeta/\sqrt{B_\zeta}}\), \(\1_k \gets \1_k, [1-\1_k]\) and \(\delta \gets \frac{\delta}{8}\), 
    the following upperbounds hold with probability at least \(1- 3\times \frac{\delta}{6} = 1 - \frac{\delta}{2}\):
    \begin{eqnarray*}
        & &\sum_{k\in[K]} \|\bff(\bw_k, \bxi_k) - \bff(\bw_k) \|_\infty \1_k 
        \leq 2\sigma_\zeta \sqrt{\frac{2 \ln \frac{12Kn}{\delta}}{B_\zeta}} \sum_{k\in[K]} \1_k\\
        & &\sum_{k\in[K]} \|\bg(\bw_k, \bxi_k) - \bg(\bw_k) \|_\infty \1_k 
        \leq 2\sigma_\zeta \sqrt{\frac{2 \ln \frac{12Kn}{\delta}}{B_\zeta}} \sum_{k\in[K]} \1_k\\
        & &\sum_{k\in[K]} \|\bg(\bw_k, \bxi_k) - \bg(\bw_k) \|_\infty [1-\1_k] 
        \leq 2\sigma_\zeta \sqrt{\frac{2 \ln \frac{12Kn}{\delta}}{B_\zeta}} \sum_{k\in[K]} [1-\1_k]
    \end{eqnarray*}

    \textbf{2.4: upper bound of \(\sum_{k\in[K]} \sum_{i\in \calI} [\br_k]_i \left\| \sum_{\tau\in[E]} \bz_{k, \tau}^{(i)} \right\|^2\)}

    For brevity, 
    with the correspondence \(t := \text{ind}(k, i, \tau) \equiv k\cdot n E + i\cdot E + \tau\),
    we define such a filtration \((\calG_k)_{k\in\mathbb{Z}_{\geq -1}}\) by 
    \(\calG_{-1} = \{\emptyset, \Omega\}\) and 
    \(\calG_k =  \calF_{nE \lceil \frac{t}{nE} \rceil}\), 
    and \(\bz_{k}^{(i)} := \frac{1}{2\sqrt{E}}\sum_{\tau\in[E]} \bz_{k, \tau}^{(i)}\).
    Then from the definition and properties of \(\bz_{k, \tau}^{(i)}\) listed 
    at the beginning of the Step 2, we have:
    \[
    \E[\bz_{k}^{(i)} \mid \calG_k] = \vec{0}, \quad 
    \ln[\E[\exp(\lambda\|\bz_{k}^{(i)}\|) \mid \calG_k]]
    \leq \frac{\lambda^2}{2},
    \forall\lambda \in \mathbb{R},
    \]
    and \(\bz_{k}^{(i)}\) is \(\calG_{k+1}\)-measurable.
    Morevover, \(\br_k\) is \(\calG_k\)-measurable from definitions of \(\br_k\) and \(\calG_k\).
    We rewrite the summation as:
    \[
    \sum_{k\in[K]} \sum_{i\in \calI} [\br_k]_i \left\| \sum_{\tau\in[E]} \bz_{k, \tau}^{(i)} \right\|^2
    = 4E \sum_{k\in[K]} \sum_{i\in \calI} [\br_k]_i \| \bz_{k}^{(i)} \|^2 
    = 4E \sum_{k\in [K]} X_k
    \]
    where \(X_k := \sum_{i\in \calI} [\br_k]_i \| \bz_{k}^{(i)} \|^2\).
    Then using Jensen's inequality with \(\sum_{i\in \calI} [\br_k]_i =1, [\br_k]_i \geq 0\) 
    and \(\br_k\) is \(\calG_k\)-measurable,
    and the tower property of conditional expectation, we have the following inequality 
    when \(\lambda \in (0, 1)\):
    \[
    \E[\exp(\lambda X_k/2) \mid \calG_k] 
    \leq \sum_{i\in \calI} [\br_k]_i \E[\exp(\lambda \| \bz_{k}^{(i)} \|^2/2) \mid \calG_k]
    \leq \max_{i\in \calI} \E[\exp(\lambda \| \bz_{k}^{(i)} \|^2/2) \mid \calG_k] 
    \leq (1- \lambda)^{-1/2}
    \]
    since by introducing an independent standard Gaussian \(z \sim \mathcal{N}(0, 1)\),
    using \(\E_{z}[\exp(\lambda z^2/2)] = (1- \lambda)^{-1/2}\) for \(\lambda \in (0, 1)\):
    \[
    \E[\exp(\lambda\| \bz_{k}^{(i)} \|^2/2) \mid \calG_k] 
    = \E[\E_{z}[\exp(\sqrt{\lambda} \|\bz_{k}^{(i)}\| z)\mid \calG_k]]
    = \E_{z}[\E[\exp(\sqrt{\lambda} z \|\bz_{k}^{(i)}\|) \mid \calG_k, z]]
    \leq \E_{z}[\mathe^{\frac{\lambda z^2}{2}}]
    = (1- \lambda)^{-1/2}
    \]
    Therefore, noting \(X_k\) is \(\calG_{k+1}\)-measurable, for any \(\lambda \in (0, 1)\), we have:
    \[
    - K\frac{\lambda}{2}+
    \ln\E\left[\exp\left(\frac{\lambda}{2}\sum_{k\in[K]} X_k\right) \right]
    =-\lambda \frac{K}{2}
    +\sum_{k\in[K]} \ln\E\left[\exp\left(\frac{\lambda}{2} X_k\right) \mid \calG_k\right]
    \leq \frac{K}{2} \left[-\lambda-\ln(1-\lambda)\right]
    \leq \frac{K}{4} \frac{\lambda^2}{(1-\lambda)}
    \]
    By using the Chernoff bound, and introducing \(\psi^\ast(u) := \sup_{\lambda \in(0,1)} u \lambda - \frac{\lambda^2}{(1-\lambda)}\), we have:
    \[
    \ln \Pr\left(\sum_{k\in[K]} X_k \geq K(1+u/2)\right)
    \leq \inf_{\lambda> 0} -\frac{K}{4}(2+u)\lambda+ \ln \E\left[\exp\left(\frac{\lambda}{2} \sum_{k\in[K]} X_k\right)\right]
    \leq -\frac{K}{4}\psi^\ast(u)
    \]
    By the characterization of sub-gamma random variables, see also equation (2.5) in section 2.4 sub-gamma random variables, on page 29 of~\cite{boucheron2013concentration}.
    we have \(\psi^\ast(u) = (\sqrt{1+u}-1)^2\) with its inverse function 
    \(\psi^{\ast-1}(v) = 2\sqrt{v} + v\):
    \begin{eqnarray*}
    & &\Pr\left(\sum_{k\in[K]} X_k \geq K(1+\psi^{\ast-1}(v)/2)\right)
    = \Pr\left(\sum_{k\in[K]} X_k \geq K(1+\sqrt{v} + v/2)\right)
    \leq \exp\left(-\frac{K}{4}v\right)
    \end{eqnarray*}
    By selecting \(v = \frac{4\ln\frac{8}{\delta}}{K}\)
    and noting \(1+\sqrt{v} + v/2 \leq \frac{3}{2} + v\), with probability at least \(1-\frac{\delta}{8}\), we have:
    \[
    \sum_{k\in[K]} \sum_{i\in\calI} [\br_k]_i \left\| \sum_{\tau\in[E]} \bz_{k, \tau}^{(i)} \right\|^2 
    = 4E\sum_{k\in[K]} X_k 
    \leq 4E \cdot K(1+\sqrt{v} + v/2)
    \leq 4E \cdot K\left(\frac{3}{2}+v\right) 
    = 2E \left( 3 K + 8 \ln \frac{8}{\delta} \right)
    \]

    \textbf{2.5: upper bound of \(\sum_{k\in[K]} \sum_{i\in \calI} [\br_k]_i \sum_{\tau\in[E]} \left\langle \bd_k, \bz_{k, \tau}^{(i)}\right\rangle\)}

    With the correspondence \(t := \text{ind}(k, i, \tau) \equiv k\cdot n E + i\cdot E + \tau\),
    we define such intermediate quantities, 
    \(T= K\cdot n \cdot E\), 
    \(\bv_t := [\br_k]_i \bd_k\), \(\bz_t := \bz_{k, \tau}^{(i)}\),
    and \(S_t := \sum_{t'\in [t]} \langle \bv_{t'}, \bz_{t'} \rangle\),
    \(V_t := \sum_{t'\in [t]} \|\bv_{t'}\|^2\).
    Then, since \(\bz_t\) is \(\calF_{t+1}\)-measurable
    and \(\bv_t\) is \(\calF_t\)-measurable, then \(S_t\) is \(\calF_t\)-measurable,
    \(V_t\) is \(\calF_{t-1}\)-measurable.

    Then for \(M_t := \exp\left( \lambda S_t-\frac{\lambda^2}{2} V_t \right)\), 
    which is \(\calF_t\)-measurable, 
    we show it is a supermartingale 
    by using Theorem 4.2.4 on page 189 of~\cite{durrett2019probability} ans showing 
    that the following inequality holds:
    \[
    \E[M_{t+1}\mid \calF_t] 
    = M_t\cdot\E\left[\exp\left( \lambda \langle \bv_{t}, \bz_{t} \rangle-\frac{\lambda^2}{2} \|\bv_{t}\|^2 \right)\mid \calF_t\right]
    \leq M_t
    \]
    Therefore, we have \(\E[M_T] \leq M_0 = 1\) by the tower property of conditional expectation, 
    which implies:
    \[
    \E\left[\exp\left(\lambda \sum_{t\in[T]} \langle \bv_{t}, \bz_{t} \rangle\right)\right]
    \leq \E\left[\exp\left(\frac{\lambda^2}{2} \sum_{t\in[T]} \|\bv_{t}\|^2\right)\right]
    \leq \exp\left(\frac{\lambda^2}{2} KE\right)
    \]
    since we have the following fact using \(\|\bd_k\| \leq 1\) and \(\sum_{i\in[\calI]}[\br_k]_i^2 \leq
    \sum_{i\in \calI} [\br_k]_i = 1\):
    \[
    \sum_{t\in[T]} \|\bv_t\|^2
    =
    \sum_{k\in[K]} \sum_{i\in \calI} [\br_k]_i^2 \sum_{\tau\in[E]} \|\bd_k\|^2\leq KE
    \]
    By using the Chernoff bound, then for any \(u > 0\), we have:
    \[
    \Pr\left(\sum_{t\in[T]} \langle \bv_{t}, \bz_{t} \rangle \geq u\right)
    \leq \exp\left(\inf_{\lambda>0} -\lambda u + \frac{\lambda^2}{2} KE\right)
    \leq \exp\left(-\frac{u^2}{2KE}\right)
    \]
    By selecting \(u = \sqrt{2KE \ln\frac{8}{\delta}}\), with probability at least \(1-\frac{\delta}{8}\), we have:
    \[
    \sum_{k\in[K]} \sum_{i\in \calI} [\br_k]_i \sum_{\tau\in[E]} \left\langle \bd_k, \bz_{k, \tau}^{(i)}\right\rangle
    =\sum_{t\in[T]} \langle \bv_{t}, \bz_{t} \rangle \leq \sqrt{2KE \ln\frac{8}{\delta}}
    \]

    \textbf{2.6: upper bound of \(\sum_{k\in[K]} \sum_{i\in \calI} [\br_k]_i \sum_{\tau\in[E-1]} \tau \left\langle \be_{k, \tau}^{(i)}, \bz_{k, \tau}^{(i)} \right\rangle\)}

    With the correspondence \(t := \text{ind}(k, i, \tau) \equiv k\cdot n E + i\cdot E + \tau\),
    we define such intermediate quantities, 
    \(T= K\cdot n \cdot E\), 
    \(\bv_t := [\br_k]_i \tau \be_{k, \tau}^{(i)} \1_{\tau \in[E-1]}\), \(\bz_t := \bz_{k, \tau}^{(i)}\),
    then by the same procedure of showing supermartingale, we also have 
    \[
    \E\left[\exp\left(\lambda \sum_{t\in[T]} \langle \bv_{t}, \bz_{t} \rangle\right)\right]
    \leq \E\left[\exp\left(\frac{\lambda^2}{2} \sum_{t\in[T]} \|\bv_{t}\|^2\right)\right]
    \leq \exp\left(\frac{\lambda^2}{2} \frac{K}{3} (E-1)(E-3/2)(E-2)\right)
    \]
    since we have the following fact using \(\|\be_k\| \leq 1, \sum_{\tau\in[E-1]} \tau^2 = \frac{1}{3} (E-1)(E-3/2)(E-2)\) and \(\sum_{i\in[\calI]}[\br_k]_i^2 \leq
    \sum_{i\in \calI} [\br_k]_i = 1\):
    \[
    \sum_{t\in[T]} \|\bv_t\|^2
    =
    \sum_{k\in[K]} \sum_{i\in \calI} [\br_k]_i^2 \sum_{\tau\in[E-1]} \tau^2 \|\be_{k, \tau}^{(i)}\|^2
    \leq \frac{K}{3} (E-1)(E-3/2)(E-2)
    \]
    By the same procudure of taking Chernoff bound, 
    with probability at least \(1-\frac{\delta}{8}\), we have:
    \[
    \sum_{k\in[K]} \sum_{i\in \calI} [\br_k]_i \sum_{\tau\in[E-1]} \tau \left\langle \be_{k, \tau}^{(i)}, \bz_{k, \tau}^{(i)}\right\rangle
    =
    \sum_{t\in[T]} \langle \bv_{t}, \bz_{t} \rangle
    \leq \sqrt{\frac{2K}{3} (E-1)(E-3/2)(E-2) \ln\frac{8}{\delta}}
    \]

    \textbf{2.7: upper bound of \(\sum_{k\in[K]} \sum_{i\in \calI} [\br_k]_i \| \bz_{k, E-1}^{(i)} \|\)}
    With the correspondence \(t := \text{ind}(k, i, \tau) \equiv k\cdot n E + i\cdot E + \tau\),
    we define such intermediate quantities, 
    \(T= K\cdot n \cdot E\), 
    \(a_t := [\br_k]_i \1_{\tau = E-1}\), \(z_t := \|\bz_{k, \tau}^{(i)}\|\),
    then by the same procedure of showing supermartingale, we also have 
    \[
    \E\left[\exp\left(\sum_{t\in[T]} a_t z_t\right)\right]
    \leq \E\left[\exp\left(\frac{\lambda^2}{2} \sum_{t\in[T]} a_t^2\right)\right]
    \leq \exp\left(\frac{\lambda^2}{2} K\right)
    \]
    Since \(\sum_{i\in\calI} [\br_k]_i^2 \leq \sum_{i\in\calI} [\br_k]_i = 1\), we have:
    \[
    \sum_{t\in[T]} a_t^2
    =
    \sum_{k\in[K]} \sum_{i\in \calI} [\br_k]_i^2 \sum_{\tau\in[E]} \1_{\tau = E-1}^2
    =
    \sum_{k\in[K]} \sum_{i\in \calI} [\br_k]_i^2
    \leq
    K
    \]
    By the same procedure of taking Chernoff bound, 
    with probability at least \(1-\frac{\delta}{8}\), we have:
    \[
    \sum_{k\in[K]} \sum_{i\in \calI} [\br_k]_i \| \bz_{k, E-1}^{(i)} \|
    =
    \sum_{t\in[T]} a_t z_t
    \leq \sqrt{2K \ln\frac{8}{\delta}}
    \]

    \textbf{Step 3. Establish Final Bounds}

    \textbf{3.1: rearranging terms in final bounds}

    By selecting \(\epsilon' = \epsilon - 4\sigma_\zeta \sqrt{\frac{2 \ln \frac{12Kn}{\delta}}{B_\zeta}}\),
    and rearranging the terms with the established bounds in Step 2, we have:
    \begin{eqnarray*}
        \sum_{k\in[K]} [F(\bw_k) - F(\bw^\ast)] \1_k
        &\leq& \epsilon \sum_{k\in[K]} \1_k 
        - \sum_{k\in[K]} \left[G_k(\bw_k, \bxi_k) - \frac{\epsilon}{2}\right] [1-\1_k]
        - \frac{\epsilon K}{2}
        + \frac{D^2}{2\eta}
        + 4\eta L^2 K \left(1-\frac{1}{2E}\right)\\
        & &+ 2\sigma_\zeta \sqrt{\frac{2 \ln \frac{12Kn}{\delta}}{B_\zeta}} \cdot K 
        + \frac{16\eta \sigma_g^2}{B_g E}\left(3 + \frac{8\ln\frac{8}{\delta}}{K}\right)\cdot K\\
        &&+ 2D \sigma_g\sqrt{\frac{2 \ln \frac{8}{\delta}}{B_g KE}}\cdot K
        + 4\eta L \sigma_g
        \sqrt{\frac{2E}{3B_g K} \left(1-\frac{1}{E}\right)\left(1-\frac{3}{2E}\right)\left(1-\frac{2}{E}\right) \ln\frac{8}{\delta}} \cdot K\\
        &&+ 4\eta L \sigma_g \cdot \frac{1}{E} \sqrt{\frac{2\ln \frac{8}{\delta}}{B_g K}}\cdot K\\
        \sum_{k\in[K]} G(\bw_k) \1_k 
        & \leq& \epsilon \sum_{k\in[K]} \1_k 
    \end{eqnarray*}
    with probability at least \(1-\delta\) with 
    step sizes \(\eta, \gamma\) such that \(\gamma = \frac{\eta}{E}\), and
    softmax hyperparameter \(\alpha\geq \frac{2 \ln n}{\epsilon'}\).
    
    \textbf{3.2: selecting step sizes \(\eta, \gamma\), tolerance \(\epsilon\) and softmax hyperparameter \(\alpha\)}
    
    By balancing the terms \(\frac{D^2}{2\eta}, 4\eta L^2 K\), we selecting step sizes as:
    \[
    \eta = \frac{D}{L \sqrt{8K}}, \quad \gamma = \frac{D}{LE \sqrt{8K}}
    \]
    By upper bounding \(1-\frac{1}{2E}\leq 1, \sqrt{(1-\frac{1}{E})( 1-\frac{3}{2E})(1-\frac{2}{E})} \leq 1-\frac{9/4}{E+2}\), 
    ans substituting the global step size \(\eta\), we have:
    \begin{eqnarray*}
        \sum_{k\in[K]} [F(\bw_k) - F(\bw^\ast)] \1_k
        &\leq& \epsilon \sum_{k\in[K]} \1_k 
        - \sum_{k\in[K]} \left[G_k(\bw_k, \bxi_k) - \frac{\epsilon}{2}\right] [1-\1_k]
        - \frac{\epsilon K}{2}
        + \frac{DL}{\sqrt{K/8}}\cdot K\\
        & &+ 2\sigma_\zeta \sqrt{\frac{2 \ln \frac{12Kn}{\delta}}{B_\zeta}} \cdot K\\
        & &+ \frac{D \sigma_g^2}{LBE\sqrt{K/32}}\left(3 + \frac{8\ln\frac{8}{\delta}}{K}\right)\cdot K\\
        & &+ 2D \sigma_g \sqrt{\frac{2\ln \frac{8}{\delta}}{BKE}}
        \left(1
        + \frac{E}{\sqrt{6K}}\left(1-\frac{9/4}{E+2}\right)
        + \frac{1}{\sqrt{2KE}}\right) \cdot K
    \end{eqnarray*}

    Noting that for \(E=1\), we can show that the above bound still holds without \(\frac{E}{\sqrt{6K}}(1-\frac{9/4}{E+2})+\frac{1}{\sqrt{2KE}}\) in the last term 
    since terms with \(\sum_{\tau\in[E-1]}\) are 0 when \(E=1\). For \(E\geq2\), 
    we note that \(\frac{E}{\sqrt{6K}}(1-\frac{9/4}{E+2})+\frac{1}{\sqrt{2KE}} \leq \frac{E}{\sqrt{6K}}\) is valid.
    Combining these two cases, and introducing the ``effective'' gradient variance \(\bar{\sigma}_g^2 := \frac{\sigma_g^2/B_g}{L^2 E}\), we have:
    \begin{eqnarray*}
    & &\sum_{k\in[K]} [F(\bw_k) - F(\bw^\ast)] \1_k
    \leq \epsilon \sum_{k\in[K]} \1_k 
    - \sum_{k\in[K]} \left[G_k(\bw_k, \bxi_k) - \frac{\epsilon}{2}\right] [1-\1_k]\\
    &+& \frac{K}{2}
    \left[
    -\epsilon 
    + \frac{DL}{\sqrt{K/32}}\left[
        1 + 2 \bar{\sigma}_g^2 \left(3 + \frac{8\ln\frac{8}{\delta}}{K}\right)
        + \bar{\sigma}_g \sqrt{ \ln\frac{8}{\delta}} \left( 1 + \frac{E}{\sqrt{6K}}\right)
        \right]
        + 4\sigma_\zeta \sqrt{\frac{2 \ln \frac{12Kn}{\delta}}{B_\zeta}}
    \right]
    \end{eqnarray*}

    By letting the sum of constant terms to be 0, noting \(\epsilon' = \epsilon - 4\sigma_\zeta \sqrt{\frac{2 \ln \frac{12Kn}{\delta}}{B_\zeta}}\), we obtain the following tolerance:
    \[
    \epsilon' = \frac{DL}{\sqrt{K/32}} \left[
        1 + 2 \bar{\sigma}_g^2 \left(3 + \frac{8\ln\frac{8}{\delta}}{K}\right)
        + \bar{\sigma}_g \sqrt{ \ln\frac{8}{\delta}} \left( 1 + \frac{E}{\sqrt{6K}}\right)
        \right] 
    \]
    \[
        \epsilon
        = \frac{DL}{\sqrt{K/32}} \left[ 1 + 2 \bar{\sigma}_g^2 
        \left( 3 + \frac{8\ln\frac{8}{\delta}}{K} \right) 
        + \bar{\sigma}_g \sqrt{ \ln\frac{8}{\delta}} \left( 1 + \frac{E}{\sqrt{6K}} \right) \right]
        + 4\sigma_\zeta \sqrt{\frac{2 \ln \frac{12Kn}{\delta}}{B_\zeta}}
    \]

    Then, we obtain the following inequality with probability at least \(1-\delta\):
    \[
    \sum_{k\in[K]} \left[G_k(\bw_k, \bxi_k) - \frac{\epsilon}{2}\right] [1-\1_k]
    + \sum_{k\in[K]} \left[F(\bw_k)-F(\bw^\ast)\right]\1_k
    \leq \epsilon \sum_{k\in[K]} \1_k
    \]

    We show \(\sum_{k\in[K]} \1_k \neq 0\), otherwise, from the above inequlaity,
    and using the definition of \(\1_k \equiv \1_{G_k(\bw_k, \bxi_k)\leq\frac{\epsilon}{2}}\), we have
    \( G_{k}(\bw_k, \bxi_k) - \frac{\epsilon}{2} > 0, [1-\1_k] = 1, \1_k=0\) for all \(k\in[K]\), 
    which leads to a contradiction as follows:
    \[
    0 < \sum_{k\in[K]} \left[G_k(\bw_k, \bxi_k) - \frac{\epsilon}{2}\right] [1-\1_k] + 0 \leq \epsilon \cdot 0 =0
    \]
    Noting that \(0 \leq [G_k(\bw_k, \bxi_k)-\frac{\epsilon}{2}][1-\1_k]\), 
    we have the following inequality with \(\sum_{k\in[K]} \1_k \neq 0\):
    \[
    \sum_{k\in[K]} \left[F(\bw_k)-F(\bw^\ast)\right]\1_k \leq \epsilon \sum_{k\in[K]} \1_k
    \]
    By introducing a probability measure \(\pr_K\) on the set of iterations \([K]\) such that \(\pr_K(k) = \1_k/\sum_{k\in[K]} \1_k, \forall k\in[K]\), 
    and using the convexity of \(F, G\) (\cref{assumption:convexity} and \cref{lemma:properties_of_weighted_functions}) and defining \(\overline{\bw}_K := \E_{k\sim \pr_K} [\bw_k] = \sum_{k\in[K]} \bw_k \1_k/\sum_{k\in[K]} \1_k\), then:
    \[
    F(\overline{\bw}_K) - F(\bw^\ast) \leq
    \E_{k\sim \pr_K} \left[F(\bw_k) - F(\bw^\ast)\right] \leq  
    \epsilon 
    \]
    \[
    G(\overline{\bw}_K) \leq \E_{k\sim \pr_K} [G(\bw_k)] \leq 
    \epsilon 
    \]
    when the softmax hyperparameter \(\alpha \geq \frac{2 \ln n}{\epsilon'}\) is large enough.
\end{proof}

\newpage
\section{Proof for Result of Switching Strategy with Full Participation}\label{sup:switch}
\begin{proof}[Proof for \Cref{theorem:convergence_gd_softmax_lipschitz} of \Cref{alg:switching_gd_softmax}]
    When the number of local updates \(E=1\) of full participation case
    of \Cref{alg:switching_gd_softmax_fedfull}, 
    it becomes the simple version \Cref{alg:switching_gd_softmax}.
    The Step 1 and Step 2 in the proof for \Cref{alg:switching_gd_softmax} are the same as the Step 1 and Step 2 in the proof of 
    \Cref{alg:switching_gd_softmax_fedfull},
    except replacing \(\frac{\delta}{8}\) with \(\frac{\delta}{4}\)
    in steps 2.4-2.5 since the terms in 2.6-2.7 of \(\sum_{\tau=[E-1]}\) are 0 in the upper bound of \(\sum_{k\in[K]} [F(\bw_k) - F(\bw^\ast)]\1_k\) is 0
    when \(E=1\), 
    noting that in step 2.4 \(\|\bz_k^{(i)}\|\) is \(\frac{1}{4}\)-subgaussian when \(E=1\)
    (tighter than 1-subgaussian in general proof) therefore \(\frac{1}{4}\times\) ``original upper bound''
    is still a valid upper bound.

    \textbf{Step 3. Establish Final Bounds}

    \textbf{3.1: rearranging terms in final bounds}

    By selecting \(\epsilon' = \epsilon - 4\sigma_\zeta \sqrt{\frac{2 \ln \frac{12Kn}{\delta}}{B}}\),
    and rearranging the terms with the established bounds in Step 2 and letting \(E=1\):
    \begin{eqnarray*}
        \sum_{k\in[K]} [F(\bw_k) - F(\bw^\ast)] \1_k
        &\leq& \epsilon \sum_{k\in[K]} \1_k 
        - \sum_{k\in[K]} \left[G_k(\bw_k, \bxi_k)-\frac{\epsilon}{2}\right][1-\1_k]
        - \frac{\epsilon K}{2}
        + \frac{D^2}{2\eta}
        + 2\eta L^2 K\\
        & &+ 2\sigma_\zeta \sqrt{\frac{2 \ln \frac{12Kn}{\delta}}{B_\zeta}} \cdot K 
        + \frac{2\eta \sigma_g^2}{B_g}\left(3 + \frac{8\ln\frac{8}{\delta}}{K}\right)\cdot K
        + 2D \sigma_g\sqrt{\frac{2 \ln \frac{8}{\delta}}{B_g K}}\cdot K\\
        \sum_{k\in[K]} G(\bw_k) \1_k 
        & \leq& \epsilon \sum_{k\in[K]} \1_k 
    \end{eqnarray*}
    with probability at least \(1-\delta\) and
    softmax hyperparameter \(\alpha\geq \frac{2 \ln n}{\epsilon'}\).
    
    \textbf{3.2: selecting step sizes \(\eta, \gamma\), tolerance \(\epsilon\) and softmax hyperparameter \(\alpha\)}
    
    By balancing the terms \(\frac{D^2}{2\eta}, 2\eta L^2 K\), we selecting step sizes as:
    \(
    \eta = \frac{D}{2L \sqrt{K}}
    \).
    Substituting the global step size \(\eta\), 
    and introducing the ``effective'' gradient variance \(\bar{\sigma}_g^2 := \frac{\sigma_g^2/B_g}{L^2 \times 1}\), we have:
    \begin{eqnarray*}
        \sum_{k\in[K]} [F(\bw_k) - F(\bw^\ast)] \1_k
        &\leq& \epsilon \sum_{k\in[K]} \1_k 
        - \sum_{k\in[K]} \left[G_k(\bw_k, \bxi_k)-\frac{\epsilon}{2}\right][1-\1_k]\\
        & &+ \frac{K}{2}
        \left[
            -\epsilon
            + \frac{2DL}{\sqrt{K}} \left[
                1 + \bar{\sigma}_g^2 
                \left( 3 + \frac{8\ln\frac{8}{\delta}}{K} \right) 
                + \bar{\sigma}_g \sqrt{8 \ln\frac{8}{\delta}} 
            \right]
            + 4\sigma_\zeta \sqrt{\frac{2 \ln \frac{12Kn}{\delta}}{B_\zeta}}
        \right]
    \end{eqnarray*}
    By letting the sum of constant terms to be 0, we obtain the following tolerance:
    \[
    \epsilon
    = 
    \frac{2DL}{\sqrt{K}} \left[
        1 + \bar{\sigma}_g^2 
        \left( 3 + \frac{8\ln\frac{8}{\delta}}{K} \right) 
        + \bar{\sigma}_g \sqrt{8 \ln\frac{8}{\delta}} 
    \right]
    + 4\sigma_\zeta \sqrt{\frac{2 \ln \frac{12Kn}{\delta}}{B_\zeta}}
    \]
    \[
    \epsilon'
    = \epsilon - 4\sigma_\zeta \sqrt{\frac{2 \ln \frac{12Kn}{\delta}}{B_\zeta}}
    = \frac{2DL}{\sqrt{K}} \left[
        1 + \bar{\sigma}_g^2 
        \left( 3 + \frac{8\ln\frac{8}{\delta}}{K} \right) 
        + \bar{\sigma}_g \sqrt{8 \ln\frac{8}{\delta}} 
    \right]
    \]

    Then, we obtain the following inequality with probability at least \(1-\delta\):
    \[
    \sum_{k\in[K]} \left[G_k(\bw_k, \bxi_k)-\frac{\epsilon}{2}\right][1-\1_k]
    + \sum_{k\in[K]} \left[F(\bw_k)-F(\bw^\ast)\right]\1_k
    \leq \epsilon \sum_{k\in[K]} \1_k
    \]
    We show \(\sum_{k\in[K]} \1_k \neq 0\), otherwise, from the above inequlaity,
    and using the definition of \(\1_k \equiv \1_{G_k(\bw_k, \bxi_k)\leq\frac{\epsilon}{2}}\), we have
    \( G_{k}(\bw_k, \bxi_k) - \frac{\epsilon}{2} > 0, [1-\1_k] = 1, \1_k=0\) for all \(k\in[K]\), 
    which leads to a contradiction as follows:
    \[
    0 < \sum_{k\in[K]} \left[G_k(\bw_k, \bxi_k)-\frac{\epsilon}{2}\right][1-\1_k] + 0 \leq \epsilon \cdot 0 =0
    \]
    Noting that \(0 \leq [G_k(\bw_k, \bxi_k)-\frac{\epsilon}{2}][1-\1_k]\), 
    we have the following inequality with \(\sum_{k\in[K]} \1_k \neq 0\):
    \[
    \sum_{k\in[K]} \left[F(\bw_k)-F(\bw^\ast)\right]\1_k \leq \epsilon \sum_{k\in[K]} \1_k
    \]
    By introducing a probability measure \(\pr_K\) on the set of iterations \([K]\) such that \(\pr_K(k) = \1_k/\sum_{k\in[K]} \1_k, \forall k\in[K]\), 
    and using the convexity of \(F, G\) (\cref{assumption:convexity} and \cref{lemma:properties_of_weighted_functions}) and defining \(\overline{\bw}_K := \E_{k\sim \pr_K} [\bw_k] = \sum_{k\in[K]} \bw_k \1_k/\sum_{k\in[K]} \1_k\), then:
    \[
    F(\overline{\bw}_K) - F(\bw^\ast) \leq
    \E_{k\sim \pr_K} \left[F(\bw_k) - F(\bw^\ast)\right] \leq  
    \epsilon 
    \]
    \[
    G(\overline{\bw}_K) \leq \E_{k\sim \pr_K} [G(\bw_k)] \leq 
    \epsilon 
    \]
    when the softmax hyperparameter \(\alpha \geq \frac{2 \ln n}{\epsilon'}\) is large enough.
    
\end{proof}

\newpage
\section{Proof for Result of Federated Learning with Partial Participation}\label{sup:fedpart}
\begin{proof}[Proof for \Cref{theorem:convergence_gd_softmax_fedpartial_lipschitz} of \Cref{alg:switching_gd_softmax_fedpartial}]
    By applying \cref{lemma:polarization} to the update rule \(\bw_{k+1} = \bw_k - \eta \mathbf{u}_k\) with step size \(\eta>0\), we have
    \[
    \langle \mathbf{u}_k, \mathbf{w}_{k} - \mathbf{w}^\ast \rangle
         =  \frac{1}{2\eta} \left(\|\mathbf{w}_k - \mathbf{w}^\ast\|^2 - \|\mathbf{w}_{k+1} - \mathbf{w}^\ast\|^2 \right) 
         + \frac{\eta}{2} \|\bu_k\|^2
    \]
    For brevity, we write \(\1_k := \1_{G_k(\bw_k, \bxi_k) \leq \frac{\epsilon}{2}}\)
    with \(G_k(\bw_k, \bxi_k) = 
    m(\bg(\bw_k, \bxi_k), \alpha)\).
    The local direction \(\mathbf{u}_k\) is given by the following equation 
    and local updates are given by \(\bw_{k,\tau+1}^{(i)} = \bw_{k,\tau}^{(i)} - \gamma 
    \left[\nabla f_i (\bw_{k,\tau}^{(i)}, \bzeta_{k,\tau}^{(i)}) \1_k + \nabla g_i (\bw_{k,\tau}^{(i)}, \bzeta_{k,\tau}^{(i)}) [1-\1_k] \right]\):
    \begin{eqnarray*}
        \bu_k^{(i)}
        = \frac{\bw_{k, 0}^{(i)} - \bw_{k, E}^{(i)}}{\gamma E}
        = \frac{1}{E}\sum_{\tau\in [E]}
        \nabla f_i (\bw_{k,\tau}^{(i)}, \bzeta_{k,\tau}^{(i)})
        \1_k
        + \frac{1}{E}\sum_{\tau\in [E]}
        \nabla g_i (\bw_{k,\tau}^{(i)}, \bzeta_{k,\tau}^{(i)})
        [1-\1_k]
    \end{eqnarray*}
    where \(\bzeta_{k,\tau}^{(i)}\) is the sample at the \(\tau\)-th local update step of the \(k\)-th epoch for the \(i\)-th client,
    and \(\bw_{k, 0}^{(i)} = \bw_k\) is the initial local parameters for the \(i\)-th client.
    The direction \(\mathbf{u}_k\) is given by the following equation with 
    the brevity notations \(\bp_k := \softmax(\alpha\bff(\bw_k, \bxi_k))\) and \(\bq_k := \softmax(\alpha\bg(\bw_k, \bxi_k))\),
    and we introduce \(\br_k := \bp_k \1_k + \bq_k [1-\1_k]\):
    \begin{eqnarray*}
        \bu_k
        = \sum_{i\in \calI_k} ([\bp_k]_i \1_k + [\bq_k]_i [1-\1_k]) \bu_k^{(i)} = \sum_{i\in \calI_k} [\br_k]_i \bu_k^{(i)}
    \end{eqnarray*}
    To help the analysis, we introduce the following notations:
    \begin{eqnarray*}
        \btu_k^{(i)} = \frac{1}{E}\sum_{\tau\in [E]}
        \nabla f_i (\bw_{k,\tau}^{(i)})
        \1_k
        + \frac{1}{E}\sum_{\tau\in [E]}
        \nabla g_i (\bw_{k,\tau}^{(i)})
        [1-\1_k],\quad
        \bbu_k^{(i)} = \nabla f_i (\bw_k) \1_k + \nabla g_i (\bw_k) [1-\1_k]
    \end{eqnarray*}

    \textbf{Step 1: Upperbound \(\eta \sum\limits_{k\in[K]} \|\bu_k\|^2\) and Lowerbound \(\sum\limits_{k\in[K]} \langle \mathbf{u}_k, \bw_k-\bw^\ast\rangle\)}

    Then we can decompose the direction as \(\bu_k^{(i)} 
    = \btu_k^{(i)} + (\bu_k^{(i)} - \btu_k^{(i)})
    = \bbu_k^{(i)} - (\bbu_k^{(i)} - \btu_k^{(i)}) - (\btu_k^{(i)} - \bu_k^{(i)})\).

    \textbf{1.1: upperbound of \(\eta \sum_{k\in[K]} \|\bu_k\|^2\)}

    Then by using \(\frac{1}{2}\|\bx - \by\|^2 \leq \|\bx\|^2 + \|\by\|^2, \E[X]^2 \leq \E[X^2]\) 
    and the assumption of Lipschitz continuity~\cref{assumption:lipschitz_continuity} (\(\|\nabla f_i(\bw)\| \leq L, \|\nabla g_i(\bw)\| \leq L\),
    therefore \(\|\btu_k^{(i)}\| \leq L, \left\| \sum_{i\in \calI_k} [\br_k]_i \btu_k^{(i)} \right\| \leq L\) 
    using \(\sum_{i\in \calI_k} [\br_k]_i = 1\) and all \([\br_k]_i \geq 0\)):
    \begin{eqnarray*}
        \frac{\eta}{2} \sum_{k\in[K]} \|\bu_k\|^2 
        = \frac{\eta}{2} \sum_{k\in[K]} \left\|
            \sum_{i\in \calI_k} [\br_k]_i \btu_k^{(i)}
            + \sum_{i\in \calI_k} [\br_k]_i (\bu_k^{(i)} - \btu_k^{(i)})
         \right\|^2
        \leq \eta L^2 K
        + \eta \sum_{k\in [K]} \sum_{i\in \calI_k} [\br_k]_i \|\bu_k^{(i)} - \btu_k^{(i)}\|^2
    \end{eqnarray*}

    \textbf{1.2: decomposition of inner product \(\sum_{k\in[K]} \langle \mathbf{u}_k, \bw_k-\bw^\ast\rangle\)}
    
    We decompose the inner product by
    introducing the Bregman divergence
    \(D_{f_i}[\bw'||\bw]:= f_i(\bw')-f_i(\bw)-\langle \nabla f_i(\bw), \bw'-\bw \rangle\geq 0,
    D_{g_i}[\bw'||\bw]:= g_i(\bw')-g_i(\bw)-\langle \nabla g_i(\bw), \bw'-\bw \rangle \geq 0\)
    for convex functions \(f_i\) and \(g_i\) from~\cref{assumption:convexity},
    therefore \( \langle \mathbf{u}_k, \bw_k-\bw^\ast\rangle = (f_i(\bw_k)-f_i(\bw^\ast)+D_{f_i}[\bw^\ast||\bw_k])\1_k + (g_i(\bw_k)-g_i(\bw^\ast)+D_{g_i}[\bw^\ast||\bw_k])[1-\1_k]\):
    \begin{eqnarray*}
        & &\sum_{k\in[K]} \langle \mathbf{u}_k, \bw_k-\bw^\ast\rangle\\
        & = &
        \sum_{k\in[K]} \sum_{i\in \calI_k} [\br_k]_i \langle \bbu_k^{(i)}, \bw_k-\bw^\ast\rangle 
       -\sum_{k\in[K]} \sum_{i\in \calI_k} [\br_k]_i \langle \bbu_k^{(i)} - \btu_k^{(i)}, \bw_k-\bw^\ast\rangle
          -\sum_{k\in[K]} \sum_{i\in \calI_k} [\br_k]_i \langle \btu_k^{(i)} - \bu_k^{(i)}, \bw_k-\bw^\ast\rangle\\
       & = & \sum_{k\in[K]} \sum_{i\in \calI_k} [\bp_k]_i (f_i(\bw_k) - f_i(\bw^\ast))\1_k 
          + \sum_{k\in[K]} \sum_{i\in \calI_k} [\bq_k]_i (g_i(\bw_k) - g_i(\bw^\ast))[1-\1_k]\\
       & &- \sum_{k\in[K]} \sum_{i\in \calI_k} [\bp_k]_i \frac{1}{E}\sum_{\tau\in[E]}
       \left[-D_{f_i}[\bw^\ast||\bw_k]
       + \langle \nabla f_i(\bw_k) - \nabla f_i(\bw_{k,\tau}^{(i)}), \bw_k-\bw^\ast\rangle\right] \1_k\\
       & &- \sum_{k\in[K]} \sum_{i\in \calI_k} [\bq_k]_i \frac{1}{E}\sum_{\tau\in[E]}
       \left[- D_{g_i}[\bw^\ast||\bw_k]
       + \langle \nabla g_i(\bw_k) - \nabla g_i(\bw_{k,\tau}^{(i)}), \bw_k-\bw^\ast\rangle\right][1-\1_k]\\
       & &- \sum_{k\in[K]} \sum_{i\in \calI_k} [\br_k]_i
       \langle \btu_k^{(i)} - \bu_k^{(i)}, \bw_k-\bw^\ast\rangle
    \end{eqnarray*}

    \textbf{1.3: lowerbound of 1st and 2nd terms in the decomposition of inner product \(\sum_{k\in[K]} \langle \mathbf{u}_k, \bw_k-\bw^\ast\rangle\)}
    
    Regarding the first and second terms in the above equation for the inner product,
    using the definition of \(F(\bw; \calI_k) = \max_{i\in \calI_k} f_i(\bw; \calI_k)\) 
    and \(G(\bw; \calI_k) = \max_{i\in \calI_k} g_i(\bw; \calI_k)\), 
    from the definition of \(\bw^\ast\) such that \(G(\bw^\ast) \leq 0\) in~\cref{eq:opt_sol}.
    \[\sum_{i\in \calI_k} [\bp_k]_i f_i(\bw^\ast) 
    \leq F(\bw^\ast; \calI_k)
    \leq F(\bw^\ast),\quad 
    \sum_{i\in \calI_k} [\bq_k]_i g_i(\bw^\ast)
    \leq G(\bw^\ast; \calI_k)
    \leq G(\bw^\ast) = -|G(\bw^\ast)| \leq 0\]

    Applying the deviation bound of softmax mean \cref{lemma:deviation_bound_softmax_mean} 
    with \(\alpha \geq \frac{2 \ln m}{\epsilon'}\) 
    and substituting \(\bx \gets \bff(\bw_k)_{\calI_k}, \bdelta \gets \bff(\bw_k, \bxi_k)_{\calI_k} - \bff(\bw_k)_{\calI_k}\)
    and \(\bx \gets \bg(\bw_k)_{\calI_k}, \bdelta \gets \bg(\bw_k, \bxi_k)_{\calI_k} - \bg(\bw_k)_{\calI_k}\), then we have
    \begin{eqnarray*}
    \sum_{i\in\calI_k} [\bp_k]_i f_i(\bw_k) 
    & \geq &F(\bw_k; \calI_k) 
    - 2\|\bff(\bw_k, \bxi_k)_{\calI_k} - \bff(\bw_k)_{\calI_k}\|_\infty - \frac{\epsilon'}{2}\\
    & = & F(\bw_k) - [F(\bw_k)-F(\bw_k; \calI_k)] - 2\|\bff(\bw_k, \bxi_k)_{\calI_k} - \bff(\bw_k)_{\calI_k}\|_\infty - \frac{\epsilon'}{2}\\
    \sum_{i\in\calI_k} [\bq_k]_i g_i(\bw_k)_{\calI_k} [1-\1_k] 
    &\geq &\left[G_k(\bw_k, \bxi_k; \calI_k) - \|\bg(\bw_k, \bxi_k)_{\calI_k} - \bg(\bw_k)_{\calI_k}\|_\infty\right] \1_{G_k(\bw_k, \bxi_k; \calI_k) > \frac{\epsilon}{2}}\\
    &= &
    \left[G_k(\bw_k, \bxi_k;\calI_k)-\frac{\epsilon}{2}\right][1-\1_k]
    +\left(\frac{\epsilon}{2} - \|\bg(\bw_k, \bxi_k)_{\calI_k} - \bg(\bw_k)_{\calI_k}\|_\infty\right) [1-\1_k]
    \end{eqnarray*}
    Combining the above, we have the following lower bound for the first and second terms in the equation for the inner product:
    \begin{eqnarray*}
        &   & \sum_{k\in[K]} \sum_{i\in \calI_k} [\bp_k]_i (f_i(\bw_k) - f_i(\bw^\ast))\1_k 
        + \sum_{k\in[K]} \sum_{i\in \calI_k} [\bq_k]_i (g_i(\bw_k) - g_i(\bw^\ast))[1-\1_k]\\
        & \geq & \sum_{k\in[K]} [F(\bw_k) - F(\bw^\ast)] \1_k 
        + \sum_{k\in[K]} \left[G_k(\bw_k, \bxi_k;\calI_k) - \frac{\epsilon}{2}\right][1-\1_k]\\
        & & + |G(\bw^\ast)| \sum_{k\in[K]} [1-\1_k]
        + \frac{\epsilon K}{2}
        - \frac{\epsilon' + \epsilon}{2} \sum_{k\in[K]} \1_k \\
        & & - 2 \sum_{k\in[K]} \|\bff(\bw_k, \bxi_k) - \bff(\bw_k)\|_\infty \1_k
        - \sum_{k\in[K]} \|\bg(\bw_k, \bxi_k) - \bg(\bw_k)\|_\infty [1-\1_k]\\
        & & - \sum_{k\in[K]} [F(\bw_k)-F(\bw_k;\calI_k)] \1_k
    \end{eqnarray*}
    \textbf{1.4: upperbound of 3rd and 4th terms in the decomposition of inner product \(\sum_{k\in[K]} \langle \mathbf{u}_k, \bw_k-\bw^\ast\rangle\)}
    
    Regarding the third term and fourth term in the equation for the inner product,
    we noting the three-point Bregman divergence identity \cref{lemma:three_point_bregman_divergence_identity} 
    by substituting \(\psi \gets f_i, g_i\), \(\bx \gets \bw_k\), \(\bx' \gets \bw^\ast\), \(\hat{\bx} \gets \bw_{k,\tau}^{(i)}\),
    and using the fact that \(D_{f_i}[\cdot||\cdot] \geq 0, D_{g_i}[\cdot||\cdot] \geq 0\) from the definition of Bregman divergence
    for convex functions \(f_i\) and \(g_i\) from~\cref{assumption:convexity}.
    \[
    -D_{f_i}[\bw^\ast||\bw_k] + \langle \nabla f_i(\bw_k) - \nabla f_i(\bw_{k,\tau}^{(i)}), \bw_{k,\tau}^{(i)}-\bw^\ast\rangle
    =
    - D_{f_i}[\bw^\ast||\bw_{k,\tau}^{(i)}] - D_{f_i}[\bw_{k,\tau}^{(i)}||\bw_k] \leq 0
    \]
    \[
    -D_{g_i}[\bw^\ast||\bw_k] + \langle \nabla g_i(\bw_k) - \nabla g_i(\bw_{k,\tau}^{(i)}), \bw_{k,\tau}^{(i)}-\bw^\ast\rangle
    =
    - D_{g_i}[\bw^\ast||\bw_{k,\tau}^{(i)}] - D_{g_i}[\bw_{k,\tau}^{(i)}||\bw_k] \leq 0
    \]
    Using \((\bw_k - \bw_{k,\tau}^{(i)})\1_k = \gamma \sum_{\tau'\in[\tau]} \nabla f_i(\bw_{k,\tau'}^{(i)}, \bzeta_{k,\tau'}^{(i)})\1_k\)
    and \((\bw_k - \bw_{k,\tau}^{(i)})[1-\1_k] = \gamma \sum_{\tau'\in[\tau]} \nabla g_i(\bw_{k,\tau'}^{(i)}, \bzeta_{k,\tau'}^{(i)})[1-\1_k]\) from the update rule, 
    and applying Abel's summation formula \(\sum_{\tau\in[E]}\langle \ba_\tau, \bB_\tau \rangle 
    = \langle \bA_E, \bB_{E-1} \rangle - \sum_{\tau\in[E-1]} \langle \bA_{\tau+1}, \bb_\tau \rangle\) 
    with \(\ba_0=\vec{0}, \bA_\tau := \sum_{\tau'\in[\tau]} \ba_{\tau'}, \bB_\tau := \sum_{\tau'\in[\tau]} \bb_{\tau'}\)
    and substituting \(\ba_\tau \gets \nabla f_i(\bw_k) -\nabla f_i(\bw_{k,\tau}^{(i)}), 
    \bb_\tau \gets \nabla f_i(\bw_{k,\tau}^{(i)}, \bzeta_{k,\tau}^{(i)}) - \nabla f_i(\bw_{k,\tau}^{(i)})\), then
    again applying Cauchy-Schwarz inequality and Lipschitz continuity~\cref{assumption:lipschitz_continuity}: \\

    \begin{eqnarray*}
    & &\frac{1}{E}\sum_{\tau\in[E]} \langle \nabla f_i(\bw_k) - \nabla f_i(\bw_{k,\tau}^{(i)}),\bw_k - \bw_{k,\tau}^{(i)}\rangle \1_k \\
    &=& 2\gamma L^2 \frac{1}{E} \sum_{\tau \in [E]} \sum_{\tau'\in [\tau]} 
    \left\langle \frac{\nabla f_i(\bw_k) -\nabla f_i(\bw_{k,\tau}^{(i)})}{2L}, 
    \frac{\nabla f_i(\bw_{k,\tau'}^{(i)})}{L}\right\rangle \1_k\\
    & & + 4\gamma L(\sigma_g/\sqrt{B_g}) \left\langle 
    \frac{1}{E}\sum_{\tau \in[E]} \frac{\nabla f_i(\bw_k) - \nabla f_i(\bw_{k,\tau}^{(i)})}{2L},
    \sum_{\tau \in [E-1]} \frac{\nabla f_i(\bw_{k,\tau}^{(i)}, \bzeta_{k,\tau}^{(i)}) - \nabla f_i(\bw_{k,\tau}^{(i)})}{2\sigma_g/\sqrt{B_g}}
    \right\rangle \1_k \\
    & & - 4 \gamma L (\sigma_g/\sqrt{B_g}) \sum_{\tau \in[E-1]}
    \tau \left\langle
    \frac{1}{\tau} \sum_{\tau'\in[\tau+1]} \frac{\nabla f_i(\bw_k) - \nabla f_i(\bw_{k,\tau'}^{(i)})}{2L},
    \frac{\nabla f_i(\bw_{k,\tau}^{(i)}, \bzeta_{k,\tau}^{(i)}) - \nabla f_i(\bw_{k,\tau}^{(i)})}{2\sigma_g/\sqrt{B_g}}
    \right\rangle \1_k\\
    & \leq & 2\gamma L^2 (E-1)\1_k
    + 4 \gamma L (\sigma_g/\sqrt{B_g}) \left\| \sum_{\tau \in [E-1]} \frac{\nabla f_i(\bw_{k,\tau}^{(i)}, \bzeta_{k,\tau}^{(i)}) - \nabla f_i(\bw_{k,\tau}^{(i)})}{2\sigma_g/\sqrt{B_g}}\right\| \1_k\\
    & & + 4 \gamma L (\sigma_g/\sqrt{B_g}) \sum_{\tau\in[E-1]} \tau \left\langle \be_{k, \tau}^{\prime(i)}, 
    \frac{\nabla f_i(\bw_{k,\tau}^{(i)}, \bzeta_{k,\tau}^{(i)}) - \nabla f_i(\bw_{k,\tau}^{(i)})}{2\sigma_g/\sqrt{B_g}}
    \right\rangle\1_k
    \end{eqnarray*}
    where the intermediate vector \(\be_{k, \tau}^{\prime(i)} := -\frac{1}{\tau} \sum_{\tau'\in[\tau+1]} \frac{\nabla f_i(\bw_k) - \nabla f_i(\bw_{k,\tau'}^{(i)})}{2L}\) 
    satisfies \(\|\be_{k, \tau}^{\prime(i)}\| \leq 1\) by the assumption of Lipschitz continuity~\cref{assumption:lipschitz_continuity} 
    and hence \(\|\nabla f_i(\bw_k) - \nabla f_i(\bw_{k,\tau}^{(i)})\| \leq 2L\).
    Similarly, by introducing the intermediate vector \(\be_{k, \tau}^{\prime\prime(i)} := -\frac{1}{\tau} \sum_{\tau'\in[\tau+1]} \frac{\nabla g_i(\bw_k) - \nabla g_i(\bw_{k,\tau'}^{(i)})}{2L}\) 
    satisfies \(\|\be_{k, \tau}^{\prime\prime(i)}\| \leq 1\) by the assumption of Lipschitz continuity~\cref{assumption:lipschitz_continuity} 
    and hence \(\|\nabla g_i(\bw_k) - \nabla g_i(\bw_{k,\tau}^{(i)})\| \leq 2L\), then:
    \begin{eqnarray*}
    & &\frac{1}{E}\sum_{\tau\in[E]} \langle \nabla g_i(\bw_k) - \nabla g_i(\bw_{k,\tau}^{(i)}),\bw_k - \bw_{k,\tau}^{(i)}\rangle [1-\1_k] \\
    & \leq & 2\gamma L^2 (E-1)[1-\1_k]
    + 4 \gamma L (\sigma_g/\sqrt{B_g}) \left\| \sum_{\tau \in [E-1]} \frac{\nabla g_i(\bw_{k,\tau}^{(i)}, \bzeta_{k,\tau}^{(i)}) - \nabla g_i(\bw_{k,\tau}^{(i)})}{2\sigma_g/\sqrt{B_g}}\right\|[1-\1_k]\\
    & & + 4 \gamma L (\sigma_g/\sqrt{B_g}) \sum_{\tau\in[E-1]} \tau \left\langle \be_{k, \tau}^{\prime\prime(i)}, \frac{\nabla g_i(\bw_{k,\tau}^{(i)}, \bzeta_{k,\tau}^{(i)}) - \nabla g_i(\bw_{k,\tau}^{(i)})}{2\sigma_g/\sqrt{B_g}}\right\rangle[1-\1_k]
    \end{eqnarray*}
    We introduce the following notations, \(\be_{k, \tau}^{(i)} := \be_{k, \tau}^{\prime(i)} \1_k + \be_{k, \tau}^{\prime\prime(i)} [1-\1_k]\)
    and \(\bz_{k, \tau}^{(i)} :=
    \frac{\nabla f_i(\bw_{k,\tau}^{(i)}, \bzeta_{k,\tau}^{(i)}) - \nabla f_i(\bw_{k,\tau}^{(i)})}{2\sigma_g/\sqrt{B_g}}
    \1_k + 
    \frac{\nabla g_i(\bw_{k,\tau}^{(i)}, \bzeta_{k,\tau}^{(i)}) - \nabla g_i(\bw_{k,\tau}^{(i)})}{2\sigma_g/\sqrt{B_g}}
    [1-\1_k]\).
    Summing over \(\tau\in[E]\) and \(i\in \calI_k\)
    and noting \(\sum_{i\in\calI_k} [\bp_k]_i = 1\) and \(\sum_{i\in\calI_k} [\bq_k]_i = 1\), 
    the upperbound of the 3rd and 4th terms in the decomposition of inner product \(\sum_{k\in[K]} \langle \mathbf{u}_k, \bw_k-\bw^\ast\rangle\) is given by:
    \begin{eqnarray*}
    & &\sum_{k\in[K]} \sum_{i\in \calI_k} [\bp_k]_i \frac{1}{E}\sum_{\tau\in[E]} 
    \left[
        -D_{f_i}[\bw^\ast||\bw_k]
        + \langle \nabla f_i(\bw_k) - \nabla f_i(\bw_{k,\tau}^{(i)}), \bw_k-\bw^\ast\rangle
    \right] \1_k \\
    &+&
    \sum_{k\in[K]} \sum_{i\in \calI_k} [\bq_k]_i \frac{1}{E}\sum_{\tau\in[E]} 
    \left[
        -D_{g_i}[\bw^\ast||\bw_k]
        + \langle \nabla g_i(\bw_k) - \nabla g_i(\bw_{k,\tau}^{(i)}), \bw_k-\bw^\ast\rangle
    \right] [1-\1_k]\\
    & \leq & 2\gamma L^2 (E-1) K 
    + \frac{4 \gamma L\sigma_g}{\sqrt{B_g}} \sum_{k\in[K]} \sum_{i\in \calI_k} [\br_k]_i \left\| \sum_{\tau\in[E-1]} \bz_{k, \tau}^{(i)} \right\|
    + \frac{4 \gamma L\sigma_g}{\sqrt{B_g}} \sum_{k\in[K]} \sum_{i\in \calI_k} [\br_k]_i \sum_{\tau\in[E-1]} 
    \tau \left\langle \be_{k, \tau}^{(i)}, \bz_{k, \tau}^{(i)} \right\rangle
    \end{eqnarray*}

    \textbf{1.5: rewriting the 5th term in the decomposition of inner product \(\sum_{k\in[K]} \langle \mathbf{u}_k, \bw_k-\bw^\ast\rangle\)}
    
    We introduce the the direction vector \(\bd_k := \frac{\bw^\ast - \bw_k}{D}\) 
    which satisfies \(\|\bd_k\| \leq 1\) by the assumption of finite diameter \(D\) 
    of the parameter space \(\Theta\)~\cref{assumption:diameter_parameter_space}.
    From the definition of \(\bz_{k, \tau}^{(i)}\) and \(\btu_k^{(i)}, \bu_k^{(i)}\), we have:
    \[
    \btu_k^{(i)} - \bu_k^{(i)} = - \frac{2\sigma_g}{\sqrt{B_g}} \cdot \frac{1}{E}\sum_{\tau\in[E]} \bz_{k, \tau}^{(i)} 
    \]
    Therefore, the 5th term in the decomposition of inner product \(\sum_{k\in[K]} \langle \mathbf{u}_k, \bw_k-\bw^\ast\rangle\) 
    can be rewritten as:
    \[
    \sum_{k\in[K]}\sum_{i\in \calI_k} [\br_k]_i \langle \btu_k^{(i)} - \bu_k^{(i)}, \bw_k - \bw^\ast \rangle
    = \frac{2D \sigma_g}{\sqrt{B_g}}\cdot \frac{1}{E}
    \sum_{k\in[K]}\sum_{i\in \calI_k} [\br_k]_i 
    \sum_{\tau\in[E]} \left\langle \bd_k, \bz_{k, \tau}^{(i)}\right\rangle
    \]

    \textbf{1.6: rewriting \(\|\bu_k^{(i)}-\btu_k^{(i)}\|^2\) in the upper bound of \(\frac{\eta}{2} \sum_{k\in[K]} \|\bu_k\|^2\)}
    By using the above equation with \(\bz_{k, \tau}^{(i)}\) and \(\btu_k^{(i)}, \bu_k^{(i)}\), 
    the term in the upper bound of \(\frac{\eta}{2} \sum_{k\in[K]} \|\bu_k\|^2\) can be rewritten as:
    \[
    \eta \sum_{k\in[K]} \sum_{i\in\calI_k} [\br_k]_i \|\bu_k^{(i)}-\btu_k^{(i)}\|^2
    = \frac{4\eta\sigma^2_g}{B_g}\cdot \frac{1}{E^2}
    \sum_{k\in[K]}\sum_{i\in \calI_k} [\br_k]_i
    \left\| \sum_{\tau\in[E]} \bz_{k, \tau}^{(i)} \right\|^2
    \]

    \textbf{Step 2: Upperbound \(\sum\limits_{k\in[K]} [F(\bw_k) - F(\bw^\ast)]\1_k\) and \(\sum\limits_{k\in[K]} G(\bw_k)\1_k\)}

    We introduce a filtration \((\calF_t)_{t\in\mathbb{Z}_{\geq -1}}\) to track the information up to time
    \(t := \text{ind}(k, i, \tau) \equiv k\cdot n E + i\cdot E + \tau\) and 
    \(\calF_t := \sigma((\bzeta_{k',\tau'}^{(i')})_{i'\in\calI_{k'},\text{ind}(k', i', \tau') < t}, 
    (\bxi_{k'}^{(i')})_{i'\in\calI_{k'},\text{ind}(k', 0, 0) \leq t},
    (\calI_{k'})_{\text{ind}(k', 0, 0)\leq t})\) with \(\calF_{-1} = \{\emptyset, \Omega\}\).
    These introduced notations satisfy the following: 

    1. the indicator \(\1_k\equiv \1_{G_k(\bw_k)\leq \frac{\epsilon}{2}}\) is \(\calF_t\)-measurable; 
    (from the definition of \(G_k(\bw_k, \bxi_k) = m(\bg(\bw_k, \bxi_k), \alpha)\))

    2. the softmax weights \(\br_k\) such that \(\sum_{i\in\calI_k} [\br_k]_i = 1\) 
     and \([\br_k]_i \geq 0\) and is \(\calF_t\)-measurable; (from the defintion of \(\br_k\))
     \[
     \br_k \equiv \bp_k \1_k + \bq_k[1-\1_k] = 
     \softmax(\alpha \bff(\bw_k, \bxi_k)) \1_k + \softmax(\alpha \bg(\bw_k, \bxi_k))[1-\1_k]
     \]
    3. the direction vectors \(\bd_k\) and \(\be_k^{(i)}\) such that \(\|\bd_k\| \leq 1\) 
    (from the definition of \(\bd_k\) and the assumption of finite diameter of the parameter space \(\Theta\)~\cref{assumption:diameter_parameter_space})
    and \(\|\be_k^{(i)}\| \leq 1\) (from the definition of \(\be_k^{(i)}\) and the assumption of Lipschitz continuity~\cref{assumption:lipschitz_continuity})
     and are \(\calF_t\)-measurable; (from the definition of \(\bd_k\) and \(\be_k^{(i)}\))
     \[
     \bd_k \equiv \frac{\bw^\ast - \bw_k}{D},\quad
     \be_k^{(i)} \equiv \frac{1}{\tau}\sum_{\tau'=1}^\tau 
     \frac{\nabla f_i(\bw_{k,\tau'}^{(i)}) - \nabla f_i(\bw_k)}{2L} \1_k
     + \frac{1}{\tau}\sum_{\tau'=1}^\tau \frac{\nabla g_i(\bw_{k,\tau'}^{(i)}) - \nabla g_i(\bw_k)}{2L} [1-\1_k]
     \]
    4. the conditional 1-subgaussian random variable \(\bz_{k,\tau}^{(i)}\) 
    such that \(\E[\bz_{k,\tau}^{(i)}\mid \calF_t] = 0\) 
    and \(\E[\exp(\|\bz_{k,\tau}^{(i)}\|^2)\mid \calF_t] \leq 2\)
    and therefore
    \(\ln\E[\exp(\lambda\langle \be, \bz_{k,\tau}^{(i)}\rangle) \mid \calF_t]
    \leq \ln\E[\exp(\lambda\|\bz_{k,\tau}^{(i)}\|)\mid \calF_t] \leq \frac{\lambda^2}{2}, \forall \be\in \mathbb{S}^{d-1}, \forall\lambda \in \mathbb{R}\),
    and is \(\calF_{t+1}\)-measurable. (from the definition of \(\bz_{k,\tau}^{(i)}\),
    independence of all \(\bzeta_{k,\tau}^{(i)}\) and \(\bxi_{k}\equiv (\bxi_k^{(i)})_{i\in\calI_k}\) 
    and the assumption of sub-Gaussianity of stochastic gradients~\cref{assumption:subgaussianity_stochastic_estimates},
    and \cref{lemma:subgaussianity} for subgaussianity of random variables,
    \cref{lemma:average_subgaussian_random_vectors} for subgaussianity of the average of subgaussian random vectors)
    \[
    \bz_{k,\tau}^{(i)} \equiv 
    \frac{\nabla f_i(\bw_{k,\tau}^{(i)},\bzeta_{k,\tau}^{(i)}) -\nabla f_i(\bw_{k,\tau}^{(i)})}{2\sigma_g/\sqrt{B_g}}\1_k
    + \frac{\nabla g_i(\bw_{k,\tau}^{(i)},\bzeta_{k,\tau}^{(i)}) -\nabla g_i(\bw_{k,\tau}^{(i)})}{2\sigma_g/\sqrt{B_g}}[1-\1_k]
    \]

    \textbf{2.1: terms in the upperbound of \(\sum_{k\in[K]} [F(\bw_k) - F(\bw^\ast)]\1_k\)}

    By using the polarization identity \cref{lemma:polarization} and 
    the established lower bound of \(\eta \sum_{k\in[K]} \|\bu_k\|^2\) and 
    the lowerbound of \(\sum_{k\in[K]} \langle \mathbf{u}_k, \bw_k-\bw^\ast\rangle\) in the previous step,
    dropping the nonnegative term \(\frac{\|\bw_K-\bw^\ast\|^2}{2\eta}\),
    noting the assumption of finite diameter of the parameter space \(\Theta\)~\cref{assumption:diameter_parameter_space}
    then \(\|\bw_0-\bw^\ast\| \leq D\), 
    we rearrange the terms and obtain the following upperbound of \(\sum_{k\in[K]} [F(\bw_k) - F(\bw^\ast)]\1_k\):
    \begin{eqnarray*}
    & & \sum_{k\in[K]} [F(\bw_k) - F(\bw^\ast)]\1_k\\
    & \leq &
    \frac{\epsilon'+\epsilon}{2} \sum_{k\in[K]} \1_k
    - \sum_{k\in[K]} \left[G_k(\bw_k, \bxi_k;\calI_k)-\frac{\epsilon}{2}\right][1-\1_k]
    - |G(\bw^\ast)| \sum_{k\in[K]} [1-\1_k]\\
    &-& \frac{\epsilon K}{2}
    + \frac{D^2}{2\eta}
    + \eta L^2 K + 2 \gamma L^2 (E-1)K \\
    &+& 2 \sum_{k\in[K]} \|\bff(\bw_k, \bxi_k) - \bff(\bw_k) \|_\infty \1_k
    + \sum_{k\in[K]} \|\bg(\bw_k, \bxi_k) - \bg(\bw_k) \|_\infty [1-\1_k]\\
    &+& \frac{4\eta \sigma^2_g}{B_g} \cdot \frac{1}{E^2} 
    \sum_{k\in[K]}\sum_{i\in \calI_k} [\br_k]_i \left\| \sum_{\tau\in[E]} \bz_{k, \tau}^{(i)} \right\|^2
    + \frac{4\gamma L\sigma_g}{\sqrt{B_g}} 
    \sum_{k\in[K]} \sum_{i\in \calI_k} [\br_k]_i \left\| \sum_{\tau\in[E-1]} \bz_{k, \tau}^{(i)} \right\|\\
    &+& \frac{2D \sigma_g}{\sqrt{B_g}}\cdot \frac{1}{E} 
    \sum_{k\in[K]}\sum_{i\in \calI_k} [\br_k]_i \sum_{\tau\in[E]} \left\langle \bd_k, \bz_{k, \tau}^{(i)}\right\rangle
    + \frac{4 \gamma L\sigma_g}{\sqrt{B_g}} 
    \sum_{k\in[K]} \sum_{i\in \calI_k} [\br_k]_i \sum_{\tau\in[E-1]} \tau \left\langle \be_{k, \tau}^{(i)}, \bz_{k, \tau}^{(i)} \right\rangle\\
    &+& \sum_{k\in[K]} [F(\bw_k)-F(\bw_k;\calI_k)] \1_k
    \end{eqnarray*}
    By choosing the local step size \(\gamma = \frac{\eta}{E}\), 
    and applying triangle inequality \( \|\bx-\by\| \leq \|\bx\| +\|\by\|\), then:
    \begin{eqnarray*}
        \frac{4\gamma L\sigma_g}{\sqrt{B_g}} \sum_{k\in[K]} \sum_{i\in \calI_k} [\br_k]_i 
        \left\| \sum_{\tau\in[E-1]} \bz_{k, \tau}^{(i)} \right\|
        \leq \frac{4 \eta L \sigma_g}{\sqrt{B_g}} \frac{1}{E} \sum_{k\in[K]} \sum_{i\in \calI_k} [\br_k]_i \left\| \sum_{\tau\in[E]} \bz_{k, \tau}^{(i)} \right\|
        + \frac{4 \eta L \sigma_g}{\sqrt{B_g}} \frac{1}{E} \sum_{k\in[K]} \sum_{i\in \calI_k} [\br_k]_i \| \bz_{k, E-1}^{(i)} \|
    \end{eqnarray*}

    Noting \(\E[X]^2 \leq \E[X^2]\) and \(2\sqrt{x y} \leq x + y\) for any \(x, y \geq 0\), then we have:
    \begin{eqnarray*}
        &  &\frac{4\eta L\sigma_g}{\sqrt{B_g}} \cdot \frac{1}{E} 
        \sum_{k\in[K]} \sum_{i\in \calI_k} [\br_k]_i \left\| \sum_{\tau\in[E-1]} \bz_{k, \tau}^{(i)} \right\| 
        = \sum_{k\in [K]} 2\sqrt{\eta L^2 \times \frac{4\eta \sigma^2_g}{B_g} \cdot \frac{1}{E^2}}
        \sum_{i\in \calI_k} [\br_k]_i \left\| \sum_{\tau\in[E-1]} \bz_{k, \tau}^{(i)} \right\|\\
        &\leq & \sum_{k\in [K]} 2\sqrt{\eta L^2 \times \frac{4\eta \sigma^2_g}{B_g} \cdot \frac{1}{E^2}
        \sum_{i\in \calI_k} [\br_k]_i \left\| \sum_{\tau\in[E-1]} \bz_{k, \tau}^{(i)} \right\|^2}
        \leq
        \eta L^2 K + \frac{4\eta \sigma^2_g}{B_g} \cdot \frac{1}{E^2} \sum_{k\in[K]}\sum_{i\in \calI} [\br_k]_i \left\| \sum_{\tau\in[E-1]} \bz_{k, \tau}^{(i)} \right\|^2
    \end{eqnarray*}
    Hence, we can upperbound \(\sum_{k\in[K]} [F(\bw_k) - F(\bw^\ast)]\1_k\) as:
    \begin{eqnarray*}
        & & \sum_{k\in[K]} [F(\bw_k) - F(\bw^\ast)]\1_k\\
        &\leq &
        \frac{\epsilon'+\epsilon}{2} \sum_{k\in[K]} \1_k
        - \sum_{k\in[K]} \left[G_k(\bw_k, \bxi_k;\calI_k)-\frac{\epsilon}{2}\right][1-\1_k]\\
        & & - |G(\bw^\ast)| \sum_{k\in[K]} [1-\1_k]
        - \frac{\epsilon K}{2}
        + \frac{D^2}{2\eta}
        + 4\eta L^2 K\left(1-\frac{1}{2E}\right)\\
        &+& 2\sum_{k\in[K]} \|\bff(\bw_k, \bxi_k) - \bff(\bw_k) \|_\infty \1_k 
        + \sum_{k\in[K]} \|\bg(\bw_k, \bxi_k) - \bg(\bw_k) \|_\infty [1-\1_k]\\
        &+& \frac{8\eta \sigma^2_g}{B_g} \cdot \frac{1}{E^2} \sum_{k\in[K]}\sum_{i\in \calI_k} [\br_k]_i \left\| \sum_{\tau\in[E]} \bz_{k, \tau}^{(i)} \right\|^2\\
        &+& \frac{2D \sigma_g}{\sqrt{B_g}} \cdot \frac{1}{E} 
        \sum_{k\in[K]}\sum_{i\in \calI_k} [\br_k]_i \sum_{\tau\in[E]} \left\langle \bd_k, \bz_{k, \tau}^{(i)}\right\rangle
        + \frac{4 \eta L\sigma_g}{\sqrt{B_g}} \cdot \frac{1}{E}
        \sum_{k\in[K]} \sum_{i\in \calI_k} [\br_k]_i \sum_{\tau\in[E-1]} \tau \left\langle \be_{k, \tau}^{(i)}, \bz_{k, \tau}^{(i)} \right\rangle\\
        &+& \frac{4 \eta L\sigma_g}{\sqrt{B_g}} \cdot \frac{1}{E}
        \sum_{k\in[K]} \sum_{i\in \calI_k} [\br_k]_i \| \bz_{k, E-1}^{(i)} \|
        + \sum_{k\in[K]} [F(\bw_k)-F(\bw_k;\calI_k)] \1_k
    \end{eqnarray*}

    \textbf{2.2: terms in the upperbound of \(\sum_{k\in[K]} G(\bw_k)\1_k\)}

    Noting \(G_k(\bw_k, \bxi_k; \calI_k)\1_k = G_k(\bw_k, \bxi_k; \calI_k)\1_{G_k(\bw_k, \bxi_k; \calI_k)\leq\frac{\epsilon}{2}} \leq \frac{\epsilon}{2}\1_k\),
    applying \cref{lemma:softmax_mean} of softmax mean \(m(\bx, \alpha)\) with softmax hyperparameter \(\alpha\geq \frac{2 \ln m}{\epsilon'}\)
    then \(G(\bw_k,\bxi_k; \calI_k) - G_k(\bw_k,\bxi_k; \calI_k)
    = \max_{i\in\calI_k} [\bg(\bw_k, \bxi_k)]_i - m(\bg(\bw_k, \bxi_k)_{\calI_k}, \alpha) \leq \frac{\epsilon'}{2}\),
    using
    \(G(\bw; \calI_k)-G(\bw_k, \bxi_k; \calI_k)=\max_{i\in \calI_k} [\bg(\bw_k)]_i - \max_{i\in \calI_k} [\bg(\bw_k, \bxi_k)]_i \leq 
    \max_{i\in \calI_k} [\bg(\bw_k) - \bg(\bw_k, \bxi_k)]_i \leq \|\bg(\bw_k, \bxi_k)_{\calI_k} - \bg(\bw_k)_{\calI_k}\|_\infty\):
    \begin{eqnarray*}
    \sum_{k\in[K]} G(\bw_k)\1_k
    &=& \sum_{k\in[K]} \Big[ G_k(\bw_k, \bxi_k; \calI_k) 
    + [G(\bw_k, \bxi_k; \calI_k) - G_k(\bw_k, \bxi_k; \calI_k)]\\
    & & \quad+ [G(\bw_k; \calI_k) - G(\bw_k, \bxi_k; \calI_k)]
    + [G(\bw_k) - G(\bw_k; \calI_k)]\Big]\1_k\\
    &\leq& \frac{\epsilon + \epsilon'}{2} \sum_{k\in[K]} \1_k + \sum_{k\in[K]} \|\bg(\bw_k, \bxi_k)_{\calI_k} - \bg(\bw_k)_{\calI_k}\|_\infty \1_k 
    + \sum_{k\in[K]} [G(\bw_k) - G(\bw_k; \calI_k)] \1_k
    \end{eqnarray*}

    \textbf{2.3: upperbounds of 
    \(\sum_{k\in[K]} \|\bff(\bw_k, \bxi_k)_{\calI_k} - \bff(\bw_k)_{\calI_k} \|_\infty \1_k\) and 
    \(\sum_{k\in[K]} \|\bg(\bw_k, \bxi_k)_{\calI_k} - \bg(\bw_k)_{\calI_k} \|_\infty \1_k, [1-\1_k]\)}

    By applying the established \cref{lemma:subgaussianity,lemma:average_subgaussian_random_vectors} 
    for subgaussianity of random variables and subgaussianity of the average of subgaussian random vectors
    with \cref{assumption:subgaussianity_stochastic_estimates},
    and \cref{lemma:maximal_inequality_subgaussianity}
    of maximal inequality of subgaussian random variables 
    with \(\bz_k\gets \frac{\bff(\bw_k, \bxi_k)_{\calI_k}-\bff(\bw_k)_{\calI_k}}{2\sigma_\zeta/\sqrt{B_\zeta}},
    \frac{\bg_k(\bw_k, \bxi_k)_{\calI_k}-\bg_k(\bw_k)_{\calI_k}}{2\sigma_\zeta/\sqrt{B_\zeta}}\), \(\1_k \gets \1_k, [1-\1_k]\) and \(\delta \gets \frac{\delta}{8}\), 
    the following upperbounds hold with probability at least \(1- 3\times \frac{\delta}{12} = 1 - \frac{\delta}{4}\):
    \begin{eqnarray*}
        & &\sum_{k\in[K]} \|\bff(\bw_k, \bxi_k)_{\calI_k} - \bff(\bw_k)_{\calI_k} \|_\infty \1_k 
        \leq 2\sigma_\zeta \sqrt{\frac{2 \ln \frac{24Km}{\delta}}{B_\zeta}} \sum_{k\in[K]} \1_k\\
        & &\sum_{k\in[K]} \|\bg(\bw_k, \bxi_k)_{\calI_k} - \bg(\bw_k)_{\calI_k} \|_\infty \1_k 
        \leq 2\sigma_\zeta \sqrt{\frac{2 \ln \frac{24Km}{\delta}}{B_\zeta}} \sum_{k\in[K]} \1_k\\
        & &\sum_{k\in[K]} \|\bg(\bw_k, \bxi_k)_{\calI_k} - \bg(\bw_k)_{\calI_k} \|_\infty [1-\1_k] 
        \leq 2\sigma_\zeta \sqrt{\frac{2 \ln \frac{24Km}{\delta}}{B_\zeta}} \sum_{k\in[K]} [1-\1_k]
    \end{eqnarray*}

    \textbf{2.4: upper bound of \(\sum_{k\in[K]} \sum_{i\in \calI_k} [\br_k]_i \left\| \sum_{\tau\in[E]} \bz_{k, \tau}^{(i)} \right\|^2\)}

    For brevity, 
    with the correspondence \(t := \text{ind}(k, i, \tau) \equiv k\cdot n E + i\cdot E + \tau\),
    we define such a filtration \((\calG_k)_{k\in\mathbb{Z}_{\geq -1}}\) by 
    \(\calG_{-1} = \{\emptyset, \Omega\}\) and 
    \(\calG_k =  \calF_{nE \lceil \frac{t}{nE} \rceil}\), 
    and \(\bz_{k}^{(i)} := \frac{1}{2\sqrt{E}}\sum_{\tau\in[E]} \bz_{k, \tau}^{(i)}\).
    Then from the definition and properties of \(\bz_{k, \tau}^{(i)}\) listed 
    at the beginning of the Step 2, we have:
    \[
    \E[\bz_{k}^{(i)} \mid \calG_k] = \vec{0}, \quad 
    \ln[\E[\exp(\lambda\|\bz_{k}^{(i)}\|) \mid \calG_k]]
    \leq \frac{\lambda^2}{2},
    \forall\lambda \in \mathbb{R},
    \]
    and \(\bz_{k}^{(i)}\) is \(\calG_{k+1}\)-measurable.
    Moreover, \(\br_k\) is \(\calG_k\)-measurable from definitions of \(\br_k\) and \(\calG_k\).
    We rewrite the summation as:
    \[
    \sum_{k\in[K]} \sum_{i\in \calI_k} [\br_k]_i \left\| \sum_{\tau\in[E]} \bz_{k, \tau}^{(i)} \right\|^2
    = 4E \sum_{k\in[K]} \sum_{i\in \calI_k} [\br_k]_i \| \bz_{k}^{(i)} \|^2 
    = 4E \sum_{k\in [K]} X_k
    \]
    where \(X_k := \sum_{i\in \calI_k} [\br_k]_i \| \bz_{k}^{(i)} \|^2\).
    Then using Jensen's inequality with \(\sum_{i\in \calI_k} [\br_k]_i =1, [\br_k]_i \geq 0\) 
    and \(\br_k\) is \(\calG_k\)-measurable,
    and the tower property of conditional expectation, we have the following inequality 
    when \(\lambda \in (0, 1)\):
    \[
    \E[\exp(\lambda X_k/2) \mid \calG_k] 
    \leq \sum_{i\in \calI_k} [\br_k]_i \E[\exp(\lambda \| \bz_{k}^{(i)} \|^2/2) \mid \calG_k]
    \leq \max_{i\in \calI_k} \E[\exp(\lambda \| \bz_{k}^{(i)} \|^2/2) \mid \calG_k] 
    \leq (1- \lambda)^{-1/2}
    \]
    since by introducing an independent standard Gaussian \(z \sim \mathcal{N}(0, 1)\),
    using \(\E_{z}[\exp(\lambda z^2/2)] = (1- \lambda)^{-1/2}\) for \(\lambda \in (0, 1)\):
    \[
    \E[\exp(\lambda\| \bz_{k}^{(i)} \|^2/2) \mid \calG_k] 
    = \E[\E_{z}[\exp(\sqrt{\lambda} \|\bz_{k}^{(i)}\| z)\mid \calG_k]]
    = \E_{z}[\E[\exp(\sqrt{\lambda} z \|\bz_{k}^{(i)}\|) \mid \calG_k, z]]
    \leq \E_{z}[\mathe^{\frac{\lambda z^2}{2}}]
    = (1- \lambda)^{-1/2}
    \]
    Therefore, noting \(X_k\) is \(\calG_{k+1}\)-measurable, for any \(\lambda \in (0, 1)\), we have:
    \[
    - K\frac{\lambda}{2}+
    \ln\E\left[\exp\left(\frac{\lambda}{2}\sum_{k\in[K]} X_k\right) \right]
    =-\lambda \frac{K}{2}
    +\sum_{k\in[K]} \ln\E\left[\exp\left(\frac{\lambda}{2} X_k\right) \mid \calG_k\right]
    \leq \frac{K}{2} \left[-\lambda-\ln(1-\lambda)\right]
    \leq \frac{K}{4} \frac{\lambda^2}{(1-\lambda)}
    \]
    By using the Chernoff bound, and introducing \(\psi^\ast(u) := \sup_{\lambda \in(0,1)} u \lambda - \frac{\lambda^2}{(1-\lambda)}\), we have:
    \[
    \ln \Pr\left(\sum_{k\in[K]} X_k \geq K(1+u/2)\right)
    \leq \inf_{\lambda> 0} -\frac{K}{4}(2+u)\lambda+ \ln \E\left[\exp\left(\frac{\lambda}{2} \sum_{k\in[K]} X_k\right)\right]
    \leq -\frac{K}{4}\psi^\ast(u)
    \]
    By the characterization of sub-gamma random variables, see also equation (2.5) in section 2.4 sub-gamma random variables, on page 29 of~\cite{boucheron2013concentration}.
    we have \(\psi^\ast(u) = (\sqrt{1+u}-1)^2\) with its inverse function 
    \(\psi^{\ast-1}(v) = 2\sqrt{v} + v\):
    \begin{eqnarray*}
    & &\Pr\left(\sum_{k\in[K]} X_k \geq K(1+\psi^{\ast-1}(v)/2)\right)
    = \Pr\left(\sum_{k\in[K]} X_k \geq K(1+\sqrt{v} + v/2)\right)
    \leq \exp\left(-\frac{K}{4}v\right)
    \end{eqnarray*}
    By selecting \(v = \frac{4\ln\frac{8}{\delta}}{K}\)
    and noting \(1+\sqrt{v} + v/2 \leq \frac{3}{2} + v\), with probability at least \(1-\frac{\delta}{8}\), we have:
    \[
    \sum_{k\in[K]} \sum_{i\in\calI_k} [\br_k]_i \left\| \sum_{\tau\in[E]} \bz_{k, \tau}^{(i)} \right\|^2 
    = 4E\sum_{k\in[K]} X_k 
    \leq 4E \cdot K(1+\sqrt{v} + v/2)
    \leq 4E \cdot K\left(\frac{3}{2}+v\right) 
    = 2E \left( 3 K + 8 \ln \frac{8}{\delta} \right)
    \]

    \textbf{2.5: upper bound of \(\sum_{k\in[K]} \sum_{i\in \calI_k} [\br_k]_i \sum_{\tau\in[E]} \left\langle \bd_k, \bz_{k, \tau}^{(i)}\right\rangle\)}

    With the correspondence \(t := \text{ind}(k, i, \tau) \equiv k\cdot n E + i\cdot E + \tau\),
    we define such intermediate quantities, 
    \(T= K\cdot n \cdot E\), 
    \(\bv_t := [\br_k]_i \bd_k\), \(\bz_t := \bz_{k, \tau}^{(i)}\),
    and \(S_t := \sum_{t'\in [t]} \langle \bv_{t'}, \bz_{t'} \rangle\),
    \(V_t := \sum_{t'\in [t]} \|\bv_{t'}\|^2\).
    Then, since \(\bz_t\) is \(\calF_{t+1}\)-measurable
    and \(\bv_t\) is \(\calF_t\)-measurable, then \(S_t\) is \(\calF_t\)-measurable,
    \(V_t\) is \(\calF_{t-1}\)-measurable.

    Then for \(M_t := \exp\left( \lambda S_t-\frac{\lambda^2}{2} V_t \right)\), 
    which is \(\calF_t\)-measurable, 
    we show it is a supermartingale 
    by using Theorem 4.2.4 on page 189 of~\cite{durrett2019probability} and showing 
    that the following inequality holds:
    \[
    \E[M_{t+1}\mid \calF_t] 
    = M_t\cdot\E\left[\exp\left( \lambda \langle \bv_{t}, \bz_{t} \rangle-\frac{\lambda^2}{2} \|\bv_{t}\|^2 \right)\mid \calF_t\right]
    \leq M_t
    \]
    Therefore, we have \(\E[M_T] \leq M_0 = 1\) by the tower property of conditional expectation, 
    which implies:
    \[
    \E\left[\exp\left(\lambda \sum_{t\in[T]} \langle \bv_{t}, \bz_{t} \rangle\right)\right]
    \leq \E\left[\exp\left(\frac{\lambda^2}{2} \sum_{t\in[T]} \|\bv_{t}\|^2\right)\right]
    \leq \exp\left(\frac{\lambda^2}{2} KE\right)
    \]
    since we have the following fact using \(\|\bd_k\| \leq 1\) and \(\sum_{i\in[\calI_k]}[\br_k]_i^2 \leq
    \sum_{i\in \calI_k} [\br_k]_i = 1\):
    \[
    \sum_{t\in[T]} \|\bv_t\|^2
    =
    \sum_{k\in[K]} \sum_{i\in \calI_k} [\br_k]_i^2 \sum_{\tau\in[E]} \|\bd_k\|^2\leq KE
    \]
    By using the Chernoff bound, then for any \(u > 0\), we have:
    \[
    \Pr\left(\sum_{t\in[T]} \langle \bv_{t}, \bz_{t} \rangle \geq u\right)
    \leq \exp\left(\inf_{\lambda>0} -\lambda u + \frac{\lambda^2}{2} KE\right)
    \leq \exp\left(-\frac{u^2}{2KE}\right)
    \]
    By selecting \(u = \sqrt{2KE \ln\frac{8}{\delta}}\), with probability at least \(1-\frac{\delta}{8}\), we have:
    \[
    \sum_{k\in[K]} \sum_{i\in \calI_k} [\br_k]_i \sum_{\tau\in[E]} \left\langle \bd_k, \bz_{k, \tau}^{(i)}\right\rangle
    =\sum_{t\in[T]} \langle \bv_{t}, \bz_{t} \rangle \leq \sqrt{2KE \ln\frac{8}{\delta}}
    \]

    \textbf{2.6: upper bound of \(\sum_{k\in[K]} \sum_{i\in \calI_k} [\br_k]_i \sum_{\tau\in[E-1]} \tau \left\langle \be_{k, \tau}^{(i)}, \bz_{k, \tau}^{(i)} \right\rangle\)}

    With the correspondence \(t := \text{ind}(k, i, \tau) \equiv k\cdot n E + i\cdot E + \tau\),
    we define such intermediate quantities, 
    \(T= K\cdot n \cdot E\), 
    \(\bv_t := [\br_k]_i \tau \be_{k, \tau}^{(i)} \1_{\tau \in[E-1]}\), \(\bz_t := \bz_{k, \tau}^{(i)}\),
    then by the same procedure of showing supermartingale, we also have 
    \[
    \E\left[\exp\left(\lambda \sum_{t\in[T]} \langle \bv_{t}, \bz_{t} \rangle\right)\right]
    \leq \E\left[\exp\left(\frac{\lambda^2}{2} \sum_{t\in[T]} \|\bv_{t}\|^2\right)\right]
    \leq \exp\left(\frac{\lambda^2}{2} \frac{K}{3} (E-1)(E-3/2)(E-2)\right)
    \]
    since we have the following fact using \(\|\be_k\| \leq 1, \sum_{\tau\in[E-1]} \tau^2 = \frac{1}{3} (E-1)(E-3/2)(E-2)\) and \(\sum_{i\in[\calI_k]}[\br_k]_i^2 \leq
    \sum_{i\in \calI_k} [\br_k]_i = 1\):
    \[
    \sum_{t\in[T]} \|\bv_t\|^2
    =
    \sum_{k\in[K]} \sum_{i\in \calI_k} [\br_k]_i^2 \sum_{\tau\in[E-1]} \tau^2 \|\be_{k, \tau}^{(i)}\|^2
    \leq \frac{K}{3} (E-1)(E-3/2)(E-2)
    \]
    By the same procedure of taking Chernoff bound, 
    with probability at least \(1-\frac{\delta}{8}\), we have:
    \[
    \sum_{k\in[K]} \sum_{i\in \calI_k} [\br_k]_i \sum_{\tau\in[E-1]} \tau \left\langle \be_{k, \tau}^{(i)}, \bz_{k, \tau}^{(i)}\right\rangle
    =
    \sum_{t\in[T]} \langle \bv_{t}, \bz_{t} \rangle
    \leq \sqrt{\frac{2K}{3} (E-1)(E-3/2)(E-2) \ln\frac{8}{\delta}}
    \]

    \textbf{2.7: upper bound of \(\sum_{k\in[K]} \sum_{i\in \calI_k} [\br_k]_i \| \bz_{k, E-1}^{(i)} \|\)}
    With the correspondence \(t := \text{ind}(k, i, \tau) \equiv k\cdot n E + i\cdot E + \tau\),
    we define such intermediate quantities, 
    \(T= K\cdot n \cdot E\), 
    \(a_t := [\br_k]_i \1_{\tau = E-1}\), \(z_t := \|\bz_{k, \tau}^{(i)}\|\),
    then by the same procedure of showing supermartingale, we also have 
    \[
    \E\left[\exp\left(\sum_{t\in[T]} a_t z_t\right)\right]
    \leq \E\left[\exp\left(\frac{\lambda^2}{2} \sum_{t\in[T]} a_t^2\right)\right]
    \leq \exp\left(\frac{\lambda^2}{2} K\right)
    \]
    Since \(\sum_{i\in\calI_k} [\br_k]_i^2 \leq \sum_{i\in\calI_k} [\br_k]_i = 1\), we have:
    \[
    \sum_{t\in[T]} a_t^2
    =
    \sum_{k\in[K]} \sum_{i\in \calI} [\br_k]_i^2 \sum_{\tau\in[E]} \1_{\tau = E-1}^2
    =
    \sum_{k\in[K]} \sum_{i\in \calI} [\br_k]_i^2
    \leq
    K
    \]
    By the same procedure of taking Chernoff bound, 
    with probability at least \(1-\frac{\delta}{8}\), we have:
    \[
    \sum_{k\in[K]} \sum_{i\in \calI} [\br_k]_i \| \bz_{k, E-1}^{(i)} \|
    =
    \sum_{t\in[T]} a_t z_t
    \leq \sqrt{2K \ln\frac{8}{\delta}}
    \]

    \textbf{2.8: upper bounds of \(\sum_{k\in[K]} [F(\bw_k) - F(\bw_k; \calI_k)] \1_k\) and \(\sum_{k\in[K]} [G(\bw_k) - G(\bw_k; \calI_k)] \1_k\) }


    With the abuse of notation,
    we redefine a filtration \((\calG_k)_{k\in \mathbb{Z}_{\geq 0}}\) 
    by \(\calG_0 = \{\emptyset, \Omega\}\) and \(\calG_k = \sigma(\calG_{k-1}, \calI_{k-1}, \calB_{k-1})\),
    where a sequence of subsets \((\calI_k)_{k\in \mathbb{Z}_{\geq 0}} \stackrel{\text{i.i.d.}}{\sim} \text{Unif}(\mathcal{C}_m(\mathcal{I})) \) 
    with \( \calC_m(\calI) = \{A \subseteq \calI \mid |A| = m\} \) for some \(0 < m \leq n\) and \(\calI \equiv \{1, \cdots, n\}\), 
    are independently and identically sampled from \(\calC_m(\calI)\), and are independent of 
    the collection of random sample batches \(\calB_k \equiv \left((\bxi_k^{(i)})_{i\in\calI_k}, (\bzeta_{k, \tau}^{(i)})_{i\in\calI_k,\tau\in[E]}\right)\) at any time \(k\).
    Since \(\bw_k\)is \(\calG_k\)-measurable, \(\calI_k, \bxi_k\) are \(\calG_{k+1}\)-measurable, 
    we show that \(F(\bw_k)-F(\bw_k;\calI_k), G(\bw_k)-G(\bw_k;\calI_k), \1_k\equiv \1_{G_k(\bw_k, \bxi_k;\calI_k)\leq\frac{\epsilon}{2}}\) are \(\calG_{k+1}\)-measurable,
    and satisfy the following exponential tail bounds with \(C := \sigma/([-\ln(1-m/n)]n)\) for any \(t\geq0\) 
    by using \cref{lemma:conditional_exponential_tail_bound_with_uniform_cdf} under \cref{assumption:uniform_cdf_bound}.
    \[
    \Pr(F(\bw_k)-F(\bw_k;\calI_k) \geq t \mid \calG_k) \leq \exp\left(- \frac{t}{C}\right),\quad
    \Pr(G(\bw_k)-G(\bw_k;\calI_k) \geq t \mid \calG_k) \leq \exp\left(- \frac{t}{C}\right)
    \]

    Keeping \(\calG_k, \1_k\) unchanged,
    and substituting \(Y_k \gets F(\bw_k)-F(\bw_k;\calI_k), G(\bw_k)-G(\bw_k;\calI_k)\) 
    and \(\delta \gets \frac{\delta}{8},C \gets \sigma/([-\ln(1-m/n)]n)\) in the remark for \cref{lemma:upper_bound_average_rv_conditional_exponential_tail}, we have:
    \[
    \sum_{k\in[K]} \left[F(\bw_k)-F(\bw_k;\calI_k)\right]\1_k
    \leq \frac{2\sigma}{[-\ln(1-\frac{m}{n})]n} \ln\frac{32}{\delta} \sum_{k\in[K]} \1_k
    + \frac{2\sigma}{[-\ln(1-\frac{m}{n})]n} \ln\frac{32}{\delta} \cdot K
    \]
    \[
    \sum_{k\in[K]} \left[G(\bw_k)-G(\bw_k;\calI_k)\right]\1_k
    \leq \frac{2\sigma}{[-\ln(1-\frac{m}{n})]n} \ln\frac{32}{\delta} \sum_{k\in[K]} \1_k
    + \frac{2\sigma}{[-\ln(1-\frac{m}{n})]n} \ln\frac{1}{2\kappa} \sum_{k\in[K]} \1_k
    \]
    where \(\kappa := \frac{|\calS|}{K} \in [0, 1]\) 
    is the constraint-satisfied ratio and with such a convention of \(\ln\frac{1}{|\calS|}\cdot \sum_{k\in[K]} \1_k = 0\) when \(\kappa = 0\).

    \textbf{Step 3. Establish Final Bounds}

    \textbf{3.1: rearranging terms in final bounds}

    By selecting \(\epsilon' = \epsilon 
    - 4\sigma_\zeta \sqrt{\frac{2 \ln \frac{24Km}{\delta}}{B_\zeta}}
    - \frac{4 \sigma}{\left[-\ln\left(1-\frac{m}{n}\right)\right]n}
    \ln \frac{32}{\delta}\),
    and rearranging the terms with the established bounds in Step 2: 
    \begin{eqnarray*}
        \sum_{k\in[K]} [F(\bw_k) - F(\bw^\ast)] \1_k
        &\leq& \epsilon \sum_{k\in[K]} \1_k 
        - \sum_{k\in[K]} \left[G_k(\bw_k, \bxi_k;\calI_k)-\frac{\epsilon}{2}\right][1-\1_k]
        - |G(\bw^\ast)| \sum_{k\in[K]} [1-\1_k]\\
        & &- \frac{\epsilon K}{2}
        + \frac{D^2}{2\eta}
        + 4\eta L^2 K \left(1-\frac{1}{2E}\right)\\
        & &+ 2\sigma_\zeta \sqrt{\frac{2 \ln \frac{24Km}{\delta}}{B_\zeta}} \cdot K 
        + \frac{16\eta \sigma_g^2}{B_g E}\left(3 + \frac{8\ln\frac{8}{\delta}}{K}\right)\cdot K\\
        &&+ 2D \sigma_g\sqrt{\frac{2 \ln \frac{8}{\delta}}{B_g KE}}\cdot K
        + 4\eta L \sigma_g
        \sqrt{\frac{2E}{3B_g K} \left(1-\frac{1}{E}\right)\left(1-\frac{3}{2E}\right)\left(1-\frac{2}{E}\right) \ln\frac{8}{\delta}} \cdot K\\
        &&+ 4\eta L \sigma_g \cdot \frac{1}{E} \sqrt{\frac{2\ln \frac{8}{\delta}}{B_g K}}\cdot K
        + \frac{2\sigma}{\left[-\ln\left(1-\frac{m}{n}\right)\right]n}\ln \frac{32}{\delta} \cdot K\\
        \sum_{k\in[K]} G(\bw_k) \1_k 
        & \leq& \left(\epsilon + \frac{2\sigma}{\left[-\ln\left(1-\frac{m}{n}\right)\right]n} \ln \frac{1}{2\kappa}\right) \sum_{k\in[K]} \1_k 
    \end{eqnarray*}
    with probability at least \(1-\delta\) with 
    step sizes \(\eta, \gamma\) such that \(\gamma = \frac{\eta}{E}\), and
    softmax hyperparameter \(\alpha\geq \frac{2 \ln m}{\epsilon'}\).
    
    \textbf{3.2: selecting step sizes \(\eta, \gamma\), tolerance \(\epsilon\) and softmax hyperparameter \(\alpha\)}
    
    By balancing the terms \(\frac{D^2}{2\eta}, 4\eta L^2 K\), we selecting step sizes as:
    \[
    \eta = \frac{D}{L \sqrt{8K}}, \quad \gamma = \frac{D}{LE \sqrt{8K}}
    \]
    By upper bounding \(1-\frac{1}{2E}\leq 1, \sqrt{(1-\frac{1}{E})( 1-\frac{3}{2E})(1-\frac{2}{E})} \leq 1-\frac{9/4}{E+2}\), 
    ans substituting the global step size \(\eta\), we have:
    \begin{eqnarray*}
        \sum_{k\in[K]} [F(\bw_k) - F(\bw^\ast)] \1_k
        &\leq& \epsilon \sum_{k\in[K]} \1_k 
        - \sum_{k\in[K]} \left[G_k(\bw_k, \bxi_k;\calI_k)-\frac{\epsilon}{2}\right][1-\1_k]
        - |G(\bw^\ast)| \sum_{k\in[K]} [1-\1_k]\\
        & &- \frac{\epsilon K}{2}
        + \frac{DL}{\sqrt{K/8}}\cdot K
        + 2\sigma_\zeta \sqrt{\frac{2 \ln \frac{24Km}{\delta}}{B_\zeta}} \cdot K\\
        & &+ \frac{D \sigma_g^2}{LBE\sqrt{K/32}}\left(3 + \frac{8\ln\frac{8}{\delta}}{K}\right)\cdot K\\
        & &+ 2D \sigma_g \sqrt{\frac{2\ln \frac{8}{\delta}}{BKE}}
        \left(1
        + \frac{E}{\sqrt{6K}}\left(1-\frac{9/4}{E+2}\right)
        + \frac{1}{\sqrt{2KE}}\right) \cdot K\\
        &&+ \frac{2\sigma}{\left[-\ln\left(1-\frac{m}{n}\right)\right]n}\ln \frac{32}{\delta} \cdot K
    \end{eqnarray*}

    Noting that for \(E=1\), we can show that the above bound still holds without \(\frac{E}{\sqrt{6K}}(1-\frac{9/4}{E+2})+\frac{1}{\sqrt{2KE}}\) in the last term 
    since terms with \(\sum_{\tau\in[E-1]}\) are 0 when \(E=1\). For \(E\geq2\), 
    we note that \(\frac{E}{\sqrt{6K}}(1-\frac{9/4}{E+2})+\frac{1}{\sqrt{2KE}} \leq \frac{E}{\sqrt{6K}}\) is valid.
    Combining these two cases, and introducing the ``effective'' gradient variance \(\bar{\sigma}_g^2 := \frac{\sigma_g^2/B_g}{L^2 E}\), we have:
    \begin{eqnarray*}
    & &\sum_{k\in[K]} [F(\bw_k) - F(\bw^\ast)] \1_k
    \leq \epsilon \sum_{k\in[K]} \1_k 
    - \sum_{k\in[K]} \left[G_k(\bw_k, \bxi_k;\calI_k)-\frac{\epsilon}{2}\right][1-\1_k]
    - |G(\bw^\ast)| \sum_{k\in[K]} [1-\1_k]\\
    &+& \frac{K}{2}
    \Big[
    -\epsilon 
    + \frac{DL}{\sqrt{K/32}}\left[
        1 + 2 \bar{\sigma}_g^2 \left(3 + \frac{8\ln\frac{8}{\delta}}{K}\right)
        + \bar{\sigma}_g \sqrt{ \ln\frac{8}{\delta}} \left( 1 + \frac{E}{\sqrt{6K}}\right)
        \right]
        + 4\sigma_\zeta \sqrt{\frac{2 \ln \frac{24Km}{\delta}}{B_\zeta}}\\
        &&\quad + \frac{4\sigma}{\left[-\ln\left(1-\frac{m}{n}\right)\right]n} \ln \frac{32}{\delta}
    \Big]
    \end{eqnarray*}

    By letting the sum of constant terms to be 0, noting 
    \(\epsilon' = \epsilon 
    - 4\sigma_\zeta \sqrt{\frac{2 \ln \frac{24Km}{\delta}}{B_\zeta}}
    - \frac{4 \sigma}{\left[-\ln\left(1-\frac{m}{n}\right)\right]n}
    \ln \frac{32}{\delta}\), 
    we obtain the following tolerance:
    \[
    \epsilon' = \frac{DL}{\sqrt{K/32}} \left[
        1 + 2 \bar{\sigma}_g^2 \left(3 + \frac{8\ln\frac{8}{\delta}}{K}\right)
        + \bar{\sigma}_g \sqrt{ \ln\frac{8}{\delta}} \left( 1 + \frac{E}{\sqrt{6K}}\right)
        \right] 
    \]
    \[
        \epsilon
        = \frac{DL}{\sqrt{K/32}} \left[ 1 + 2 \bar{\sigma}_g^2 
        \left( 3 + \frac{8\ln\frac{8}{\delta}}{K} \right) 
        + \bar{\sigma}_g \sqrt{ \ln\frac{8}{\delta}} \left( 1 + \frac{E}{\sqrt{6K}} \right) \right]
        + 4\sigma_\zeta \sqrt{\frac{2 \ln \frac{24Km}{\delta}}{B_\zeta}}
        + \frac{4\sigma}{\left[-\ln\left(1-\frac{m}{n}\right)\right]n} \ln \frac{32}{\delta}
    \]
    Then, we obtain the following inequality with probability at least \(1-\delta\):
    \[
    \sum_{k\in[K]} \left[G_k(\bw_k, \bxi_k;\calI_k)-\frac{\epsilon}{2}\right][1-\1_k]
    + |G(\bw^\ast)| \sum_{k\in[K]} [1-\1_k]
    + \sum_{k\in[K]} \left[F(\bw_k)-F(\bw_k^\ast)\right]\1_k
    \leq \epsilon \sum_{k\in[K]} \1_k
    \]
    We show \(\sum_{k\in[K]} \1_k \neq 0\), otherwise, from the above inequlaity,
    and using the definition of \(\1_k \equiv \1_{G_k(\bw_k, \bxi_k;\calI_k)\leq\frac{\epsilon}{2}}\), we have
    \( G_{k}(\bw_k, \bxi_k;\calI_k) - \frac{\epsilon}{2} > 0, [1-\1_k] = 1, \1_k=0\) for all \(k\in[K]\), 
    which leads to a contradiction as follows:
    \[
    0 < \sum_{k\in[K]} \left[G_k(\bw_k, \bxi_k;\calI_k)-\frac{\epsilon}{2}\right][1-\1_k]
    + |G(\bw^\ast)| \sum_{k\in[K]} [1-\1_k] + 0 \leq \epsilon \cdot 0 =0
    \]
    Noting that \(0 \leq [G_k(\bw_k, \bxi_k;\calI_k)-\frac{\epsilon}{2}][1-\1_k]\), 
    we have the following inequality with \(\sum_{k\in[K]} \1_k \neq 0\):
    \[
    \sum_{k\in[K]} \left[F(\bw_k)-F(\bw_k^\ast)\right]\1_k \leq
    |G(\bw^\ast)| \sum_{k\in[K]} [1-\1_k] + \sum_{k\in[K]} \left[F(\bw_k)-F(\bw_k^\ast)\right]\1_k \leq
    \epsilon \sum_{k\in[K]} \1_k
    \]
    By introducing a probability measure \(\pr_K\) on the set of iterations \([K]\) such that \(\pr_K(k) = \1_k/\sum_{k\in[K]} \1_k, \forall k\in[K]\), 
    and using the convexity of \(F, G\) (\cref{assumption:convexity} and \cref{lemma:properties_of_weighted_functions}) and defining \(\overline{\bw}_K := \E_{k\sim \pr_K} [\bw_k] = \sum_{k\in[K]} \bw_k \1_k/\sum_{k\in[K]} \1_k\), then:
    \[
    F(\overline{\bw}_K) - F(\bw^\ast) \leq
    \E_{k\sim \pr_K} \left[F(\bw_k) - F(\bw^\ast)\right] \leq  
    \epsilon 
    \]
    \[
    G(\overline{\bw}_K) \leq \E_{k\sim \pr_K} [G(\bw_k)] \leq 
    \epsilon + \frac{2\sigma}{[-\ln(1-\frac{m}{n})]n} \ln\frac{1}{2\kappa}
    \]
    when the softmax hyperparameter \(\alpha \geq \frac{2 \ln m}{\epsilon'}\) is large enough.

    \textbf{3.3: analyzing the constraint-satisfied ratio \(\kappa\)}

    Since \(|\calS| = \sum_{k\in[K]} \1_k >0\) has been shown, 
    we have the constraint-satisfied ratio \(\kappa = \frac{|\calS|}{K} \in (0, 1]\) 
    and therefore \(\kappa \geq \frac{1}{K}\) in the analysis of the worst case.
    By rearranging the terms, and noting that \(\sum_{k\in[K]} \1_k = \kappa\cdot K\) and \(\sum_{k\in[K]} [1-\1_k] = (1-\kappa)\cdot K\):
    \[
    |G(\bw^\ast)| (1-\kappa) \leq \left(F(\bw^\ast)-F(\overline{\bw}_K) + \epsilon\right) \kappa
    \]
    Therefore, we can establish the upper bound of \(\frac{1}{\kappa}\) when \(-G(\bw^\ast) = |G(\bw^\ast)| > 0\):
    \[
    \frac{1}{\kappa} \leq 1 + \frac{F(\bw^\ast)-F(\overline{\bw}_K) + \epsilon}{|G(\bw^\ast)|}
    \]
    Suppose that Slater's condition holds, i.e., there exists \(\bw^\ast \in \Theta\) and some \(\nu >0\) 
    such that \(G(\bw^\ast)\leq -\nu < 0\), then:
    \[
    \ln \frac{1}{2\kappa} \leq \ln \min\left(\frac{K}{2}, \frac{1}{2}+ \frac{F(\bw^\ast)-F(\overline{\bw}_K) + \epsilon}{2|G(\bw^\ast)|}\right)
    \leq \ln\min\left( \frac{K}{2}, \frac{1}{2} + \frac{DL + \epsilon}{2\nu}\right)
    \]
\end{proof}

\newpage
\section{General Upper Bound for Sampling Error and Threshold Selection for Constraint Criterion}\label{sup:general}
\subsection{Discussion on the Limitations of Assumptions}\label{supsub:limitations}

Consider a simple setting where functions $f_i$ (and similarly $g_i$) are formed from a base function $f(\mathbf{w})$ corrupted by noise, such that $\{e_i\}_{i\in \mathcal{I}} \stackrel{\text{i.i.d.}}{\sim} \mathcal{P}_e$:
$$
f_i(\mathbf{w}) \equiv f(\mathbf{w}) + e_i.
$$

\textbf{Bounded variance and sub-Gaussianity cannot guarantee scalability.}

Suppose $\mathcal{P}_e = \mathcal{N}(0, \sigma^2)$. Classic assumptions on heterogeneity are trivially satisfied, including bounded variance ($\mathbb{E}[|f_i - f|^2] \le \sigma^2$) and sub-Gaussianity ($\mathbb{E}[\exp(|f_i-f|^2/(4\sigma^2))] \le \sqrt{2} < 2$).

However, for the maximum objective $F := \max_{i\in\mathcal{I}} f_i$ across $\mathcal{I}:=\{1, \dots, n\}$ clients, extreme value theory dictates the expected gap grows unboundedly as $n = |\mathcal{I}|$ increases. The expectation of the maximum of $n$ standard Gaussians scales asymptotically (noting $\lim_{n\to \infty} \mathbb{E}\left[\max_{j\in\mathcal{I}} e_j\right]/(\sigma\sqrt{2 \ln n})=1$; see \citet{boucheron2013concentration}, Exercise~2.17, p.~49):
$$
\mathbb{E}[F - f_i] = \mathbb{E}\left[\max_{j\in\mathcal{I}} f_j - f\right] - \mathbb{E}[f_i - f] = \mathbb{E}\left[\max_{j\in\mathcal{I}} e_j\right] \approx \sigma\sqrt{2 \ln n}.
$$
This unbounded growth ($\sigma \sqrt{\ln n}\to \infty$) makes gap-dependent bounds vacuous for large-scale federated learning ($n \to \infty$). Because the gap's probability mass shifts continuously to the right, the limiting distribution does not exist, making it impossible to define a fixed random variable $U$ that stochastically dominates $F - f_i$ for arbitrarily large $n$.

\textbf{With upper-bounded noise, the limiting distribution exists (validating \Cref{assumption:uniform_cdf_bound}).}

To ensure scalable convergence, we situate our analysis where the noise distribution has strictly upper-bounded support. We define $M$ as the essential supremum:
$$
M := \operatorname{esssup}(e) = \inf\{m \in \mathbb{R} : \mathbb{P}(e > m) = 0\} < \infty.
$$
By standard order statistics~\citep[Lemma on p.~286]{david2003order}, because $M$ is the tightest almost-sure upper bound, $\max_{j\in\mathcal{I}} e_j=F-f$ converges in probability exactly to $M$. Specifically, for any $\varepsilon>0$, $\mathbb{P}(e \leq M-\varepsilon)<1$, giving us
$$
\mathbb{P}(|(F-f)-M|>\varepsilon)=\mathbb{P}(|M-\max_{j\in\mathcal{I}}e_j|>\varepsilon)=\mathbb{P}(\max_{j\in\mathcal{I}}e_j<M-\varepsilon)=\mathbb{P}(e \leq M-\varepsilon)^n\stackrel{n\to\infty}{\to} 0.
$$

Because $F-f\xrightarrow{\mathbb{P}}M$, and $f_i-f=e_i$ implies $f-f_i\xrightarrow{d}-e_i$, we can invoke Slutsky's theorem~\citep[Lemmas~2.7--2.8, pp.~10--11]{van1998asymptotic}. This guarantees the gap converges in distribution as $n \to \infty$:
$$
F - f_i =(F-f) + (f - f_i) \xrightarrow{d} M - e_i.
$$

This essential supremum allows us to construct a universal stochastic upper bound. By defining the nonnegative random variable $U := M - e$ (where $e \sim \mathcal{P}_e$), we observe the almost-sure pointwise inequality for any finite $n$:
$$
F - f_i = \max_{j\in\mathcal{I}} e_j - e_i \le M - e_i \stackrel{d}{=} U.
$$
Pointwise domination implies first-order stochastic dominance, meaning $U$ stochastically dominates the gap for any $n$, validating \Cref{assumption:uniform_cdf_bound}.

\textbf{Limitations of \Cref{assumption:uniform_cdf_bound} and extensions via extreme value theory.}

While \Cref{assumption:uniform_cdf_bound} requires bounded noise, this limitation can be resolved using generalized extreme value theory (EVT). The Fisher--Tippett--Gnedenko theorem dictates that normalizing sequences $a_n > 0$ and $b_n$ must exist such that the standardized maximum $(\max_{j\in\mathcal{I}} e_j - b_n) / a_n$ converges to a Gumbel, Fr\'echet, or Weibull distribution~\citep[Sec.~10.5, p.~296]{david2003order}. Dynamically standardizing the objective gap with these asymptotic sequences could theoretically encompass all possible noise scenarios in a scale-invariant framework. We leave this integration to future work.

\subsection{General Upper Bound for Sampling Error under Relaxed Assumption}\label{supsub:general_bound}

We consider a relaxed version of \Cref{assumption:uniform_cdf_bound} that does not require $U \sim \mathrm{Unif}[0,\sigma]$.

The relative gaps $F-f_i=\max_{j\in\mathcal{I}} f_j - f_i$ and $G-g_i=\max_{j\in\mathcal{I}} g_j - g_i$ in \Cref{assumption:uniform_cdf_bound} are inherently nonnegative. Consequently, any stochastic upper bound should also be restricted to the nonnegative domain. Under \Cref{assumption:uniform_cdf_bound}, a nonnegative random variable $U$ is stochastically superior to the relative gaps for $F$ and $G$. While we initially assumed that $U$ follows a uniform distribution, our analysis holds for any nonnegative random variable $U$. This allows us to bound the tail probabilities of the empirical gap $F(\bw_k) - F(\bw_k; \calI_k)$ (and similarly for $G$) at each iteration $k$, and to further establish a high-probability bound for the sampling error in federated learning. Let $\calI_k \subseteq \calI=\{1, \dots, n\}$ be the sampled index subset with size $|\calI_k|=m \leq n$ and participation ratio $r := m/n$. Following an analysis similar to \Cref{lemma:conditional_exponential_tail_bound,lemma:conditional_exponential_tail_bound_with_uniform_cdf}, we obtain the following tail bound for $t\geq0$:
$$
\mathbb{P}(F(\bw_k)-F(\bw_k;\calI_k)\geq t) \leq \exp\left(-n|\ln(1-r)|\mathcal{P}_r(U<t)\right),
$$
where
$$
\mathcal{P}_r(U<t)=\begin{cases}
\mathbb{P}(U< t) & \text{if } \mathbb{P}(U< t) \leq 1-r,\\
+\infty & \text{otherwise}.
\end{cases}
$$

To standardize our analysis, we normalize $U$ and the threshold $t$ by a scale $\sigma > 0$, defining $Z := U/\sigma$ and $z := t/\sigma \geq 0$, which yields $\mathbb{P}(U<t) = \mathbb{P}(Z<z)$. We then define the nondecreasing quantile function for a probability threshold $\rho \in [0, 1)$ as
$$
Q_Z(\rho) := \inf\{z : \mathbb{P}(Z < z) > \rho\}.
$$

Let $\calS \subseteq [K]$ denote the constraint-satisfying set. By applying a maximal inequality and taking a union bound over the $K$ iterations to enforce a failure probability of $\delta/K$, we can upper bound the average sampling error over $\calS$. With probability at least $1-\delta$, we have
$$
\frac{1}{|\calS|}\sum_{k\in \calS} \left( F(\bw_k)-F(\bw_k;\calI_k) \right) \leq \sigma \cdot Q_Z\left(\min\left\{1-r, \frac{\ln\frac{K}{\delta}}{n |\ln(1-r)|}\right\}\right).
$$

To illustrate the dependency of the sampling error, consider the simplest case where $U=\sigma Z$ with $Z \sim \mathrm{Unif}[0, 1]$. Since $Q_Z(\rho)=\rho$, the bound on the sampling error evaluates exactly to $\frac{\sigma}{n|\ln(1-r)|}\ln\frac{K}{\delta}$. As shown in \Cref{theorem:convergence_gd_softmax_fedpartial_lipschitz}, this can be further improved to $\frac{\sigma}{n|\ln(1-r)|}\ln\frac{1}{\delta}$ using advanced martingale techniques.

Crucially, because we have established a general upper bound for the sampling error using an arbitrary nonnegative random variable $U=\sigma Z$, our analysis naturally generalizes to a broad class of tail behaviors beyond the uniform distribution. We provide the following concrete examples
(see quantile-function expressions in~\citet{johnson1994continuous} and special functions in~\citet{olver2010nist}):

\textbf{(1) Half-Gaussian.} For $Z = |X|$ where $X \sim \mathcal{N}(0, 1)$, the quantile function is defined via the inverse error function:
$$
Q_Z(\rho) = \sqrt{2}\, \operatorname{erf}^{-1}(\rho) \leq \frac{\sqrt{2\pi}}{2}\times\frac{\rho}{1-\rho}.
$$

\textbf{(2) Exponential.} For $Z \sim \mathrm{Exp}(1)$, the quantile function is
$$
Q_Z(\rho) = |\ln(1-\rho)| \leq \frac{\rho}{1-\rho}.
$$

\textbf{(3) Half-Cauchy.} For $Z = |X|$ where $X \sim \mathrm{Cauchy}(0, 1)$, the quantile function is derived from the inverse arctangent:
$$
Q_Z(\rho) = \tan\left(\frac{\pi \rho}{2}\right) \leq \frac{\pi}{2} \times \frac{\rho}{1-\rho}.
$$

\textbf{(4) Chi-square distribution.} For $Z\sim \chi^2_k$ with $k\in\mathbb{Z}_+$ degrees of freedom, the quantile function is defined via the inverse regularized lower incomplete gamma function:
$$
Q_Z(\rho) =  2P^{-1}\left(\frac{k}{2}, \rho\right) \leq \frac{8k}{\mathrm{e}}\times \left(\frac{\rho}{1-\rho}\right)^{2/k}.
$$

\textbf{(5) Beta distribution.} For $Z \sim \mathrm{Beta}(a, b)$ with parameters $a, b>0$, the quantile function is expressed using the inverse regularized incomplete beta function:
$$
Q_Z(\rho) = I^{-1}_\rho(a, b) \leq \min(1, C_{a, b}\rho^{1/a}),
$$
where the coefficient $C_{a, b}$ is given by the beta function $B(a,b)$:
$$
C_{a,b} = \begin{cases}
1 & \text{if }b>1,\\
\big(a B(a, b)\big)^{1/a} & \text{otherwise}.
\end{cases}
$$

\textbf{Derivation of the tail bounds under relaxed assumption.} 
Under the relaxed version of \Cref{assumption:uniform_cdf_bound} 
such that \(D_{\bff(\bw)} := F(\bw) - f_i(\bw) \preceq_{st} U\), 
\(D_{\bg(\bw)} := G(\bw) - g_i(\bw) \preceq_{st} U\), 
for some nonnegative random variable \(U\) and any \(\bw \in \Theta\), 
where \(i\) is chosen uniformly at random from \(\calI \equiv \{1, \dots, n\}\),
then for \(\bw_k\) and the sampled index subset \(\calI_k\) with size \(|\calI_k|=m \leq n\) at iteration \(k\) and participation ratio \(r := m/n\), we have the following tail bounds for \(t\geq0\):
\[
\mathbb{P}(F(\bw)-F(\bw;\calI_k)\geq t),\,\mathbb{P}(G(\bw)-G(\bw;\calI_k)\geq t) \leq \exp\left(-n|\ln(1-r)|\mathcal{P}_r(U<t)\right).
\]
where
\[
\mathcal{P}_r(U<t)=\begin{cases}
\mathbb{P}(U< t) & \text{if } \mathbb{P}(U< t) \leq 1-r,\\
+\infty & \text{otherwise}.
\end{cases}
\]
The tail bounds under relaxed assumption are derived 
by applying the following lemma and substituting \(\bx \gets \bff(\bw_k), \bg(\bw_k)\), \(\calI' \gets \calI_k, \max_{i\in\calI} [\bx]_i \gets F(\bw_k), G(\bw_k), \max_{i\in\calI'} [\bx]_i \gets F(\bw_k;\calI_k), G(\bw_k;\calI_k)\) (similar to the proof for \Cref{lemma:conditional_exponential_tail_bound_with_uniform_cdf}).
\begin{lemma}[Tail bound under relaxed assumption, similar to \Cref{lemma:conditional_exponential_tail_bound}]
\label{lemma:tail_bound_relaxed_assumption}
Let \(\bx= (x_1, \ldots, x_n)\in \R^n\) and \(X := x_i \equiv [\bx]_i\) \(i\) is chosen uniformly at random from \(\calI \equiv \{1, \dots, n\}\), 
and \(D := \max_{i'\in\calI} [\bx]_{i'} - [\bx]_i \preceq_{st} U\) for some nonnegative random variable \(U\) and any fixed value of \(\bx\). Namely, we have the following tail bound for \(t\geq0\):
\[
\frac{1}{n} \sum_{i\in \calI} \1\{\max_{i'\in\calI} [\bx]_{i'} - [\bx]_i \geq t\} \equiv \Pr(D \geq t \mid \bx) \leq \Pr(U \geq t).
\]
A subset \(\calI'\) with fixed cardinality \(|\calI'| = m \in \mathbb{Z}_+\) is selected from \(\calI\) uniformly at random, which is independent of \(\bx\), namely \(\calI' \sim \text{Unif}(\calC_{m}(\calI))\) with \(\calC_{m}(\calI) := \{A \subseteq \calI: |A| = m\}\), and 
\[
\Pr(\calI' = A)= \Pr(\calI' = A\mid \bx) = \frac{1}{|\calC_{m}(\calI)|} = \frac{1}{\binom{n}{m}}, \quad \forall A \in \calC_{m}(\calI)
\]
Then the difference between the maximum \(\max_{i\in \calI} [\bx]_i\) over the entire set \(\calI\), and the maximum \(\max_{i\in \calI'} [\bx]_i\) over the subset \(\calI'\), is bounded as follows
when \(r:=\frac{m}{n} \neq 1\):
\[
    \Pr\left(\max_{i\in \calI} [\bx]_i - \max_{i\in \calI'} [\bx]_i\geq t \mid \bx \right)
    \leq \exp\left(-n|\ln(1-r)|\mathcal{P}_r(U<t)\right).
\]
where
\[
\mathcal{P}_r(U<t)=\begin{cases}
\mathbb{P}(U< t) & \text{if } \mathbb{P}(U< t) \leq 1-r,\\
+\infty & \text{otherwise}.
\end{cases}
\]
Consequently, we also have \(\mathbb{P}(D \geq t) = \mathbb{E}[\mathbb{P}(D \geq t \mid \bx)] \leq \exp\left(-n|\ln(1-r)|\mathcal{P}_r(U<t)\right)\).
\end{lemma}

\begin{proof}[Proof of \Cref{lemma:tail_bound_relaxed_assumption}]
    Let's define the following set \(\calI_t\) for some \(t\geq 0\), then the condition can be rewritten as:
    \[
    \calI_t(\bx) := \{i\in \calI' \mid \max_{i'\in \calI} [\bx]_{i'} - [\bx]_i \geq t\}, \qquad 
    \frac{|\calI_t(\bx)|}{n} \equiv \frac{1}{n} \sum_{i\in \calI} \1\{\max_{i'\in\calI} [\bx]_{i'} - [\bx]_i \geq t\} \leq \Pr(U\geq t)
    \]
    Noting that the event \(\max_{i\in \calI} [\bx]_i - \max_{i\in \calI} [\bx]_i\geq t\) is equivalent to 
    \(\max_{i'\in \calI} [\bx]_{i'} -[\bx]_i \geq t, \forall i\in \calI\), namely \(\calI \subseteq \calI_{t}(\bx)\).
    Using \(\Pr(\calI = A\mid \bx) = \frac{1}{|\calC_{m}(\calI)|}\), 
    \(\Pr(A\subseteq \calI_t(\bx)\mid \calI=A;\bx) = \1\{A\subseteq \calI_t(\bx)\}\), 
    letting \(\calC_{m}(\calI_t(\bx)) := \{A \subseteq \calI_t(\bx) \mid |A|=m \}\).
    \begin{eqnarray*}
        & &\Pr\left(\max_{i\in \calI} [\bx]_i - \max_{i\in \calI} [\bx]_i\geq t \mid \bx \right)
        = \Pr\left(\calI\subseteq \calI_t(\bx) \mid \bx\right)
        = \sum_{A\in \calC_m(\calI)}\Pr(A\subseteq \calI_t(\bx)\mid \calI=A;\bx) \Pr(\calI=A\mid\bx)\\
        &=& \frac{\sum_{A\in \calC_m(\calI)} \1\{A\subseteq \calI_t(\bx)\}}{|\calC_{m}(\calI)|} 
        = \frac{|\{A \subseteq \calI_t(\bx) \mid |A|=m \}|}{|\calC_{m}(\calI)|}
        = \frac{|\calC_{m}(\calI_t(\bx))|}{|\calC_{m}(\calI)|}
        = \frac{\binom{|\calI_t(\bx)|}{m}}{\binom{n}{m}}\1_{m\leq |\calI_t(\bx)|}
    \end{eqnarray*}
    Noting \(|\calI'|=m\), applying \cref{lemma:upperbound_binom} to upperbound the ratio of two binomial coefficients 
    by substituting \(n'\gets |\calI_t(\bx)|\),
    and using the condition such that \(\frac{|\calI_t(\bx)|}{n}=\frac{|\calI_t(\bx)|}{|\calI|} \leq \Pr(U\geq t)\) 
    and using \(\Pr(U<t) = 1- \Pr(U\geq t), r \equiv m/n = \frac{|\calI'|}{|\calI|}\)
    \[
    \begin{aligned}
        \Pr\left(\max_{i\in \calI} [\bx]_i - \max_{i\in \calI'} [\bx]_i\geq t \mid \bx \right)
        &= \frac{\binom{|\calI_t(\bx)|}{|\calI'|}}{\binom{|\calI|}{|\calI'|}}\1_{|\calI'|\leq |\calI_t(\bx)|}
        \leq \left(1- \frac{|\calI'|}{|\calI|} \right)^{n\left[1-\frac{|\calI_t(\bx)|}{|\calI|}\right]} \1\left\{\frac{|\calI'|}{|\calI|}\leq  \frac{|\calI_t(\bx)|}{|\calI|}\right\}\\
        &\leq (1-r)^{n [1-\Pr(U\geq t)]}\1\left\{r\leq \Pr(U\geq t)\right\}
        =(1-r)^{n \Pr(U < t)} \1\left\{\Pr(U < t)\leq 1-r\right\}\\
    \end{aligned}
    \]
    Therefore, by using \((1-r)^n = \exp(-n|\ln(1-r)|)\) and introducing \(\mathcal{P}_r(U<t)\) as defined in the lemma, we have
    \[
        \Pr\left(\max_{i\in \calI} [\bx]_i - \max_{i\in \calI'} [\bx]_i\geq t \mid \bx \right)
        \leq \exp\left(-n|\ln(1-r)|\mathcal{P}_r(U<t)\right)
    \]
\end{proof}

\subsection{Selection of Threshold}\label{supsub:selection_lambda}

We define $\lambda$ as the user-prescribed threshold for the constraint criterion in the switching gradient method, and $\epsilon$ as the feasibility tolerance for constraint violations. Let $\Delta$ represent the theoretical intrinsic sum of errors (optimization, estimation, and sampling), as detailed in the error decomposition in \Cref{box:err_decompose}. Denoting the value of the constraint criterion at iteration $k$ as $\tilde{G}_k$, we define the set of constraint-satisfying iterations as $\mathcal{S} := \{k \in [K] \mid \tilde{G}_k \leq \lambda\}$, with the corresponding satisfaction ratio $\kappa := |\mathcal{S}|/K$. Based on our theoretical analysis in Appendix~\ref{sup:fedfull} 
but replacing the specific selection of $\lambda \gets \frac{\epsilon}{2}$ with more general selection of $\lambda$, we guarantee that $\mathcal{S} \neq \emptyset$ (thus $\kappa \in (0, 1]$) provided the threshold is set such that $\lambda \geq \frac{\Delta}{2}$. Consequently, for the averaged solution $\overline{\mathbf{w}}_K := \frac{1}{|\mathcal{S}|}\sum_{k\in\mathcal{S}}\mathbf{w}_k$, we have
$$
F(\overline{\mathbf{w}}_K) - F(\mathbf{w}^*) \leq \left(\frac{\Delta}{2}+\lambda\right) - \frac{1}{\kappa}\left(\lambda - \frac{\Delta}{2}\right),\qquad
G(\overline{\mathbf{w}}_K) \leq \frac{\Delta}{2}+\lambda.
$$
In particular, when the threshold $\lambda$ is set exactly to $\frac{\Delta}{2}$ (while ensuring $\mathcal{S} \neq \emptyset$ to output a valid average $\overline{\mathbf{w}}_K$), we recover the strict theoretical bounds $F(\overline{\mathbf{w}}_K) - F(\mathbf{w}^*) \leq \epsilon$ and $G(\overline{\mathbf{w}}_K) \leq \epsilon$ with $\epsilon=\Delta$, as shown in \Cref{theorem:convergence_gd_softmax_lipschitz}. As seen in \Cref{alg:switching_gd_softmax_fedpartial}, this threshold is implicitly selected in our theoretical analysis as $\lambda \gets \frac{\epsilon}{2}$, which aligns with the analysis under $\lambda = \frac{\Delta}{2}$ and feasibility tolerance $\epsilon = \Delta$.

In practice, it is difficult to determine the exact numerical value of the theoretical $\Delta$ because intrinsic constants within the error terms (such as $\bar{\sigma}_g$, $\sigma_\zeta$, etc.) are intractable to estimate exactly. Conversely, the threshold $\lambda$ and tolerance $\epsilon$ are user-prescribed hyperparameters.

To bridge this gap, we can select a threshold $\lambda$ that is sufficiently large relative to the underlying theoretical error scale---such as $\lambda \geq 10 \times \frac{\Delta}{2}$---to ensure a valid average $\overline{\mathbf{w}}_K$ and a constraint-satisfying ratio $\kappa \geq (\lambda -\frac{\Delta}{2})/(\lambda + \frac{\Delta}{2}) \approx 1$ when $F(\overline{\mathbf{w}}_K) \geq F(\mathbf{w}^*)$. Then, the inequality relaxes to
$$
G(\overline{\mathbf{w}}_K) \leq \frac{\Delta}{2}+\lambda \leq \left(\frac{1}{10}+1\right)\lambda = 1.1 \lambda.
$$
Therefore, to reliably enforce the empirical feasibility tolerance $G(\overline{\mathbf{w}}_K) \leq \epsilon$ (for example, $\epsilon=0.1$ in our Neyman--Pearson classification experiment), the threshold $\lambda$ should be set to the target value divided by $1.1$ (e.g., $\lambda = \epsilon/1.1 = 0.1/1.1$). This guarantees constraint satisfaction provided that $\lambda$ remains sufficiently large relative to the underlying error scale $\Delta$; when needed, it is tractable to reduce $\Delta$ by running more iterations, increasing the sample batch size, and so on.

Under this generalized setting such that $\lambda \gets \frac{\epsilon}{1+A^{-1}}$ for any constant $A \geq 1$, the theoretical guarantees $F(\overline{\mathbf{w}}_K)-F(\mathbf{w}^\ast) \leq \epsilon$ and $G(\overline{\mathbf{w}}_K)\leq \epsilon$ strictly hold provided that $\lambda \geq A \frac{\Delta}{2}$ (or equivalently, $\epsilon \geq (A+1)\frac{\Delta}{2}$), the theoretical guarantees $F(\overline{\mathbf{w}}_K)-F(\mathbf{w}^\ast) \leq \epsilon$ and $G(\overline{\mathbf{w}}_K)\leq \epsilon$ strictly hold provided that $\lambda \geq A \frac{\Delta}{2}$ (or equivalently, $\epsilon \geq (A+1)\frac{\Delta}{2}$).

\newpage
\section{Trade-off between Adaptivity and Stability}\label{sup:tradeoff}
\textbf{Tradeoff of $\alpha$: Is a larger $\alpha$ always better?}

\begin{itemize}
    \item \textbf{Large $\alpha$ (high adaptivity, low stability).}
    A large $\alpha$ yields high sensitivity. The softmax approaches a hard $\operatorname{argmax}$, allowing the algorithm to rapidly adapt and strictly prioritize the worst-performing objective. However, this high sensitivity creates a steep loss landscape: as it approaches the hard maximum, the objective becomes non-differentiable at the intersections where two or more losses are equal. Minor parameter updates can cause the weights to flip abruptly, leading to severe instability and permanent oscillations (limit cycles).
    \item \textbf{Small $\alpha$ (high stability, low adaptivity).}
    A small $\alpha$ yields low sensitivity. The softmax weights become nearly uniform, smoothing the loss landscape and guaranteeing stable convergence. However, the tradeoff is a loss of adaptivity: the algorithm becomes unresponsive, losing the ability to actively isolate and penalize the worst-case objective, potentially leading to suboptimal solutions.
\end{itemize}

\textbf{Example for the Selection of $\alpha$}

To build intuition, consider a simplified setting with two objectives, $f_1(\mathbf{w}) = \mathbf{w}$ and $f_2(\mathbf{w}) = -\mathbf{w}$, and two dummy constraints, $g_1(\mathbf{w})=g_2(\mathbf{w})\equiv -1<0$, over the domain $\Theta = [-1,1]$. Assuming no stochasticity ($\sigma_g = \sigma_\zeta=0$) and full participation ($\mathcal{I}_k=\mathcal{I}=\{1, 2\}$, with participation ratio $r:=m/n=1$), all required theoretical assumptions are satisfied (convexity, $L$-continuity with $L=1$, bounded parameter space with diameter $D=2$, and a bounded relative gap).

The global objective is $F(\mathbf{w})=\max\{f_1(\mathbf{w}), f_2(\mathbf{w})\}=|\mathbf{w}|$, and the global constraint is $G(\mathbf{w})\equiv -1 <0$. Because constraints are strictly satisfied everywhere, the algorithm minimizes the parameters using only the objective gradients:
$$
\min_{\mathbf{w}\in\Theta} F(\mathbf{w})=\min_{\mathbf{w}\in[-1, 1]} |\mathbf{w}| \quad \text{s.t.} \quad G(\mathbf{w})\equiv -1<0.
$$
Trivially, the optimal solution is $\mathbf{w}^*=0$.

Assuming local updates $E=1$, full participation $m=n=2$, and a deterministic setup, the algorithm simplifies to
$$
\mathbf{w}_{k+1} = \mathbf{w}_k - \eta (\nabla f_1(\mathbf{w}_k), \nabla f_2(\mathbf{w}_k)) \cdot \mathsf{softmax}(\alpha (f_1(\mathbf{w}_k),f_2(\mathbf{w}_k))^\top).
$$
Noting that $(\nabla f_1(\mathbf{w}_k), \nabla f_2(\mathbf{w}_k)) = (+1, -1)$, we can explicitly evaluate the dot product:
$$
(+1, -1)\cdot\mathsf{softmax}(\alpha (f_1(\mathbf{w}_k),f_2(\mathbf{w}_k))^\top) = \frac{\exp(\alpha \mathbf{w}_k)-\exp(-\alpha \mathbf{w}_k)}{\exp(\alpha \mathbf{w}_k)+\exp(-\alpha \mathbf{w}_k)}
= \tanh(\alpha \mathbf{w}_k).
$$
Thus, the update rule reduces precisely to
$$
\mathbf{w}_{k+1} = \mathbf{w}_k - \eta  \tanh(\alpha \mathbf{w}_k).
$$

To simplify the analysis, we introduce the scaled parameter $\alpha_\eta := \alpha \eta$ and the normalized state $x_k := \mathbf{w}_k / \eta$, yielding the recurrence relation
$$
x_{k+1} = x_k - \tanh(\alpha_\eta x_k).
$$

The behavior of this sequence depends entirely on $\alpha_\eta$, exhibiting three distinct regimes:
\begin{enumerate}
    \item \textbf{Smooth convergence ($\alpha_\eta \leq 1$).}
    The sequence $x_k$ converges strictly monotonically to the optimal solution $x^* = 0$.
    \item \textbf{Damped oscillation ($\alpha_\eta \in (1, 2]$).}
    The sequence converges to $x^* = 0$, but the trajectory alternates signs, exhibiting damped oscillations near the origin.
    \item \textbf{Divergent limit cycles ($\alpha_\eta > 2$).}
    The sequence fails to converge, becoming permanently trapped in an oscillating period-2 limit cycle.
\end{enumerate}

To ensure strict stability without oscillations, we require $\alpha_\eta = \alpha \eta \leq 1$. Conversely, to maintain sufficient adaptability, we must satisfy the lower bound established in \Cref{theorem:convergence_gd_softmax_lipschitz}. Substituting the constants $D=2$ and $L=1$ yields the adaptability requirement $\alpha \geq \frac{\ln 2}{2 \eta} \approx \frac{0.35}{\eta}$.

Balancing this fundamental trade-off identifies the optimal operational window for $\alpha_\eta \equiv \alpha \eta$:
$$
0.35 \leq \alpha_\eta \leq 1.
$$
We provide rigorous proofs for these three regimes below.

\begin{figure}[!t]
    \centering
    \includegraphics[width=\linewidth]{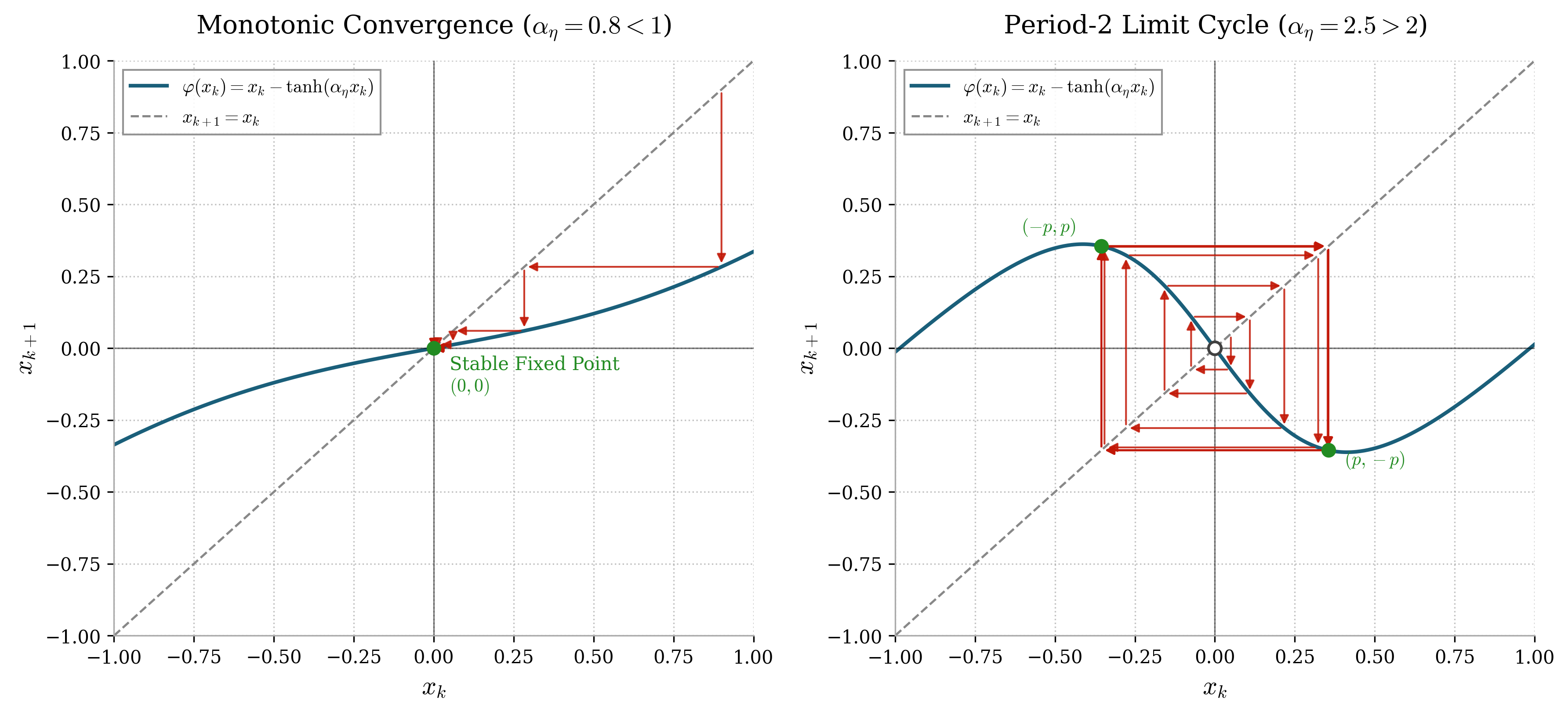}
    \caption{Cobweb plot for trajectories across different $\alpha_\eta$ values.}
    \label{fig:alpha_tradeoff}
\end{figure}

\textbf{(i) Smooth convergence ($\alpha_\eta \leq 1$).}

If $\alpha_\eta \leq 1$, then $x_{k+1}$ and $x_k$ share the same sign ($x_{k+1} x_k \geq 0$). Assuming $x_0 > 0$ without loss of generality, the sequence is strictly monotonically decreasing and bounded below by zero, satisfying $0 \leq x_{k+1} < x_k$. By the Monotone Convergence Theorem~\citep[Theorem~3.14]{rudin1976}, the sequence must converge to a finite limit, which evaluates uniquely to $0$, and thus $\mathbf{w}_k = \eta x_k \to 0$. For sufficiently small $x_k$ and $\alpha_\eta < 1$, the first-order Taylor approximation yields $\tanh(\alpha_\eta x_k) \approx \alpha_\eta x_k$, giving
$$
x_{k+1} \approx (1-\alpha_\eta) x_k.
$$
Thus, the sequence converges smoothly to $0$.

\textbf{(ii) Damped oscillation ($\alpha_\eta \in (1, 2]$).}

If $\alpha_\eta \in (1, 2]$, the sequence is no longer strictly monotonic in value, but its magnitude strictly contracts: $0 \leq |x_{k+1}| < |x_k|$ when $x_0\neq 0$. Applying the Monotone Convergence Theorem to the absolute values guarantees that $|x_k| \to 0$, and thus $\mathbf{w}_k = \eta x_k \to 0$. For sufficiently small $x_k$, the linear approximation applies, yielding
$$
x_{k+1} \approx (1 - \alpha_\eta) x_k = -(\alpha_\eta - 1) x_k.
$$
Because the multiplier $-(\alpha_\eta - 1)$ lies in the interval $[-1, 0)$, the sequence exhibits a strictly damped, alternating decay across the origin toward $0$.

\textbf{(iii) Limit cycles ($\alpha_\eta > 2$).}

If $\alpha_\eta > 2$, the sequence $\mathbf{w}_k = \eta x_k$ fails to converge to the optimal solution $\mathbf{w}^* = 0$. The sequence overshoots the origin and the magnitudes diverge outward from zero until they stabilize. The iterations become permanently trapped in a period-2 limit cycle, oscillating such that $\mathbf{w}_{2k} \to +\eta p$ and $\mathbf{w}_{2k+1} \to -\eta p$, where the fixed amplitude $p$ lies in the interval
$$
p \in \left( \frac{1}{2}\sqrt{1-2/\alpha_\eta}, \sqrt{3(\alpha_\eta-2)/8} \right).
$$

Let $\varphi(x) := x - \tanh(\alpha_\eta x)$. First, the local derivative at the origin is $|\varphi'(0)|=|1 - \alpha_\eta| > 1$. Therefore, the origin is a repelling fixed point~\citep[Proposition and Remarks on pages 33-35]{devaney2022introduction}, and the sequence $x_k$ diverges away from $0$.

Next, we establish the existence of a symmetric period-2 cycle $\{p, -p\}$, which requires a positive point $p > 0$ such that $\varphi(p) = -p$. Let $h(x) := -\varphi(x) - x = \tanh(\alpha_\eta x) - 2x$. A straightforward analysis of $h$---noting that $h(0) = 0$, $h(\infty)=-\infty$, its initial slope is positive ($h'(0) = \alpha_\eta - 2 > 0$), and it is strictly concave for $x > 0$ ($h''(x) < 0$)---shows there exists one unique root $p > 0$ such that $h(p)=0$. Thus, exactly one symmetric period-2 orbit exists.

Finally, we determine its stability by evaluating whether the multiplier satisfies $|\varphi'(p)\varphi'(-p)| < 1$~\citep[Example~10.3.3]{strogatz2024nonlinear}. Since $\varphi'(x)$ is an even function, this simplifies to $|\varphi'(p)|^2$. Because $h(0) = h(p) = 0$ and $h$ is strictly concave, its derivative at the rightmost root must be strictly negative, $h'(p) = \alpha_\eta \operatorname{sech}^2(\alpha_\eta p) - 2< 0$. Expanding this yields the exact bound
$$
|\varphi'(p)|^2 = \left|1 - \alpha_\eta \operatorname{sech}^2(\alpha_\eta p)\right|^2 < 1.
$$
Consequently, the period-2 limit cycle $\{p, -p\}$ is strictly attracting. In addition, we can bound $p$ by solving $h(p) = \tanh(\alpha_\eta p) -2p= 0$ and using $2p/(1-4p^2)>\tanh^{-1}(2p)>2p+8p^3/8$:
$$
\frac{1}{2} \sqrt{1 - \frac{2}{\alpha_\eta}} < p < \sqrt{\frac{3}{8} (\alpha_\eta - 2)}.
$$
\newpage
\section{Experimental Details}\label{sup:exp}
\subsection{Neyman Pearson Classification}
We consider the constrained minimax optimization problem fomulated in \eqref{eq:opt_constrained}, where the objective is to minimize the empirical loss on the majority class while ensuring the minority class loss remains below a prescribed tolerance. For each client $i$, the local objective and constraint are defined as
$$f_i(\mathbf{w}):=\frac{1}{m_{i,0}}\sum_{x\in\mathcal{D}_i^{(0)}}\phi(\mathbf{w};(x,0)), \qquad g_i(\mathbf{w}):=\frac{1}{m_{i,1}}\sum_{x\in\mathcal{D}_i^{(1)}}\phi(\mathbf{w};(x,1)),$$
where $\mathcal{D}_i^{(0)}$ and $\mathcal{D}_i^{(1)}$ denote the local datasets for class-0 and class-1, respectively, and $m_{i,0}$ and $m_{i,1}$ are their corresponding cardinalities. The function $\phi$ represents the binary logistic loss:
$$\phi(\mathbf{w};(x,y)) = -y \, \mathbf{w}^\top x + \log\!\left(1 + e^{\mathbf{w}^\top x}\right), \qquad y \in \{0,1\}.$$
We will distinguish the Neyman-Pearson (NP) classification as our theory-aligned convex application, leaving Adult fair-classification as a nonconvex stress test. In NP, the linear logistic loss $\phi(\mathbf{w}; (x, y)) = -y\mathbf{w}^{\top}x + \log(1+\exp(\mathbf{w}^{\top}x))$ is inherently convex in $\mathbf{w}$. Since client objectives $f_i(\mathbf{w})$ and constraints $g_i(\mathbf{w})$ are empirical averages of this convex loss, the application strictly adheres to the convex structure.

This formulation captures the NP paradigm: $f_i(\textbf{w})$ drives performance on the majority class, while the constraint $g_i(\mathbf{w}) \leq \epsilon$ ensures that the minority class loss does not exceed the predefined tolerance.

We evaluate our approach using the Breast Cancer dataset~\citep{breast_cancer_wisconsin_dataset}, which contains 569 samples with 30 features. We allocate 80\% of the data for training and reserve the remaining 20\% for testing. The training data is distributed in an independent and identically distributed (IID) manner across $n=20$ clients, ensuring each client receives an equal number of samples and an identical class distribution.

Algorithm performance is measured by tracking the majority-class objective loss $f_i(\mathbf{w})$, the minority-class constraint loss $g_i(\mathbf{w})$, and the number of rounds where the feasibility condition $g_i(\mathbf{w}) \leq \epsilon$ is violated. As illustrated in \Cref{fig:NP_classification_main}, our algorithm successfully achieves convergence of the objective while satisfying the constraint. 

\begin{table}[!ht]
    \caption{Detailed setting of NP classification experiment}
    \label{tab:detail_NP}
    \centering
    \begin{tabular}{lc|lc}
        \toprule
        Hyperparameter & Value & Hyperparameter & Value \\
        \midrule
        Number of runs & 5 & Global rounds (K) & 1000 \\
        Step size ($\eta$, Softmax SGM) & 0.5 & Softmax temperature ($\alpha$) & 6400 \\
        Local epochs (E) & 5 & Batch size & 32 \\
        Total number of clients ($n$) & 20 & Participation rate ($m/n$) & 0.5 \\
        Step size ($\eta$, Baseline) & 0.1 & Penalty parameter ($\rho$) & 2.5 \\
        Dual parameter ($\lambda_0$) & 2.5 & Dual step size ($\eta_d$) & 0.01 \\
        Tolerance ($\epsilon$) & 0.1 & Switching criteria & $\epsilon/1.1$ \\
        \bottomrule
    \end{tabular}
\end{table}

Detailed hyperparameters are provided in \Cref{tab:detail_NP}. Global step sizes were selected from the set $\{0.1, 0.2, 0.3, 0.5, 1.0, 1.5\}$. 
With regards to the softmax parameter $\alpha$, the 
For centralized training, all training data is aggregated locally and trained for $K$ rounds, which is equivalent to the federated setting with a single local epoch ($E=1$) and full client participation.


\subsection{Additional NP Classification Result}

\begin{figure}[!t]
    \centering
    \includegraphics[width=\linewidth]{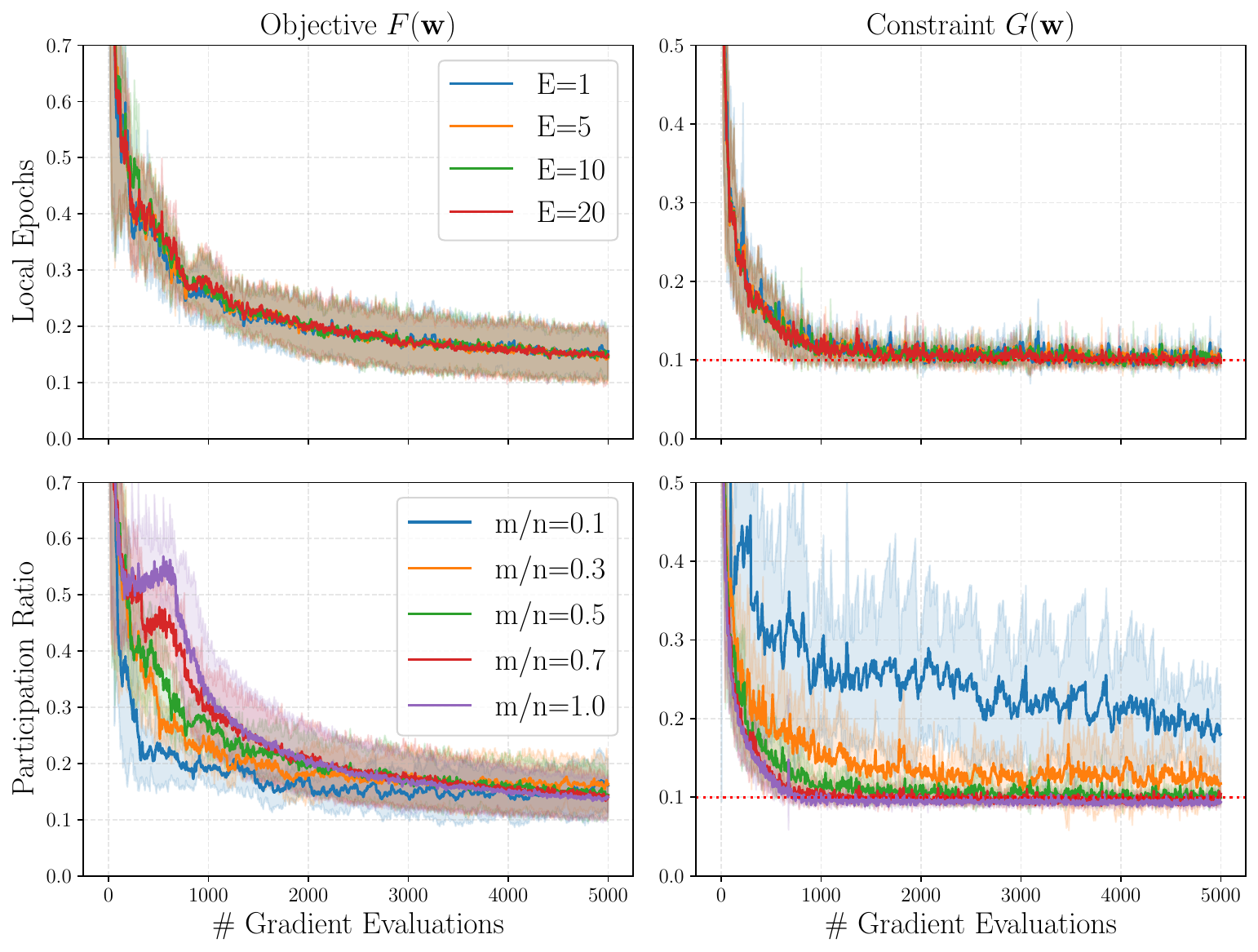}
    \caption{\textbf{Federated learning settings.} Impact of the number of local epochs $E$ (top row) and the client participation ratio $m/n$ (bottom row).}
    \label{fig:fed_settings}
\end{figure}

As shown in \Cref{fig:fed_settings} (top row), the algorithm demonstrates strong robustness to the number of local epochs $E$. This stability arises because the method effectively controls the scale of the update vector by averaging over the local steps. Moreover, higher client participation rates consistently accelerate convergence, while lower participation rates hinder strict constraint satisfaction. Because the softmax operation on a randomly sampled subset effectively evaluates an expectation, diluting the true worst-case value. This dilution biases the switching mechanism heavily toward objective minimization, which corroborates the behavior observed with smaller $\alpha$ values in \Cref{fig:alpha_sensitivity}.


\textbf{Algorithm Comparison (Softmax-SGM vs. Standard SGM)}
The \Cref{fig:sgm_comparison} shows the worst-case and average objective and constraint violation for Softmax-SGM (solving robust minimax problem) and standard SGM (solving average/expected value problem) on the Breast Cancer dataset. Softmax SGM shows consistent constraint satisfaction, while constantly improving objective performance. Although standard SGM satisfies constraint for the average case, it consistently suffers from worst-case constraint violation. This means that for high-stakes scenarios, such as medical or financial applications, standard SGM leaves some clients vulnerable to safety violations. 

\begin{figure}[!t]
    \centering
    \includegraphics[width=\linewidth]{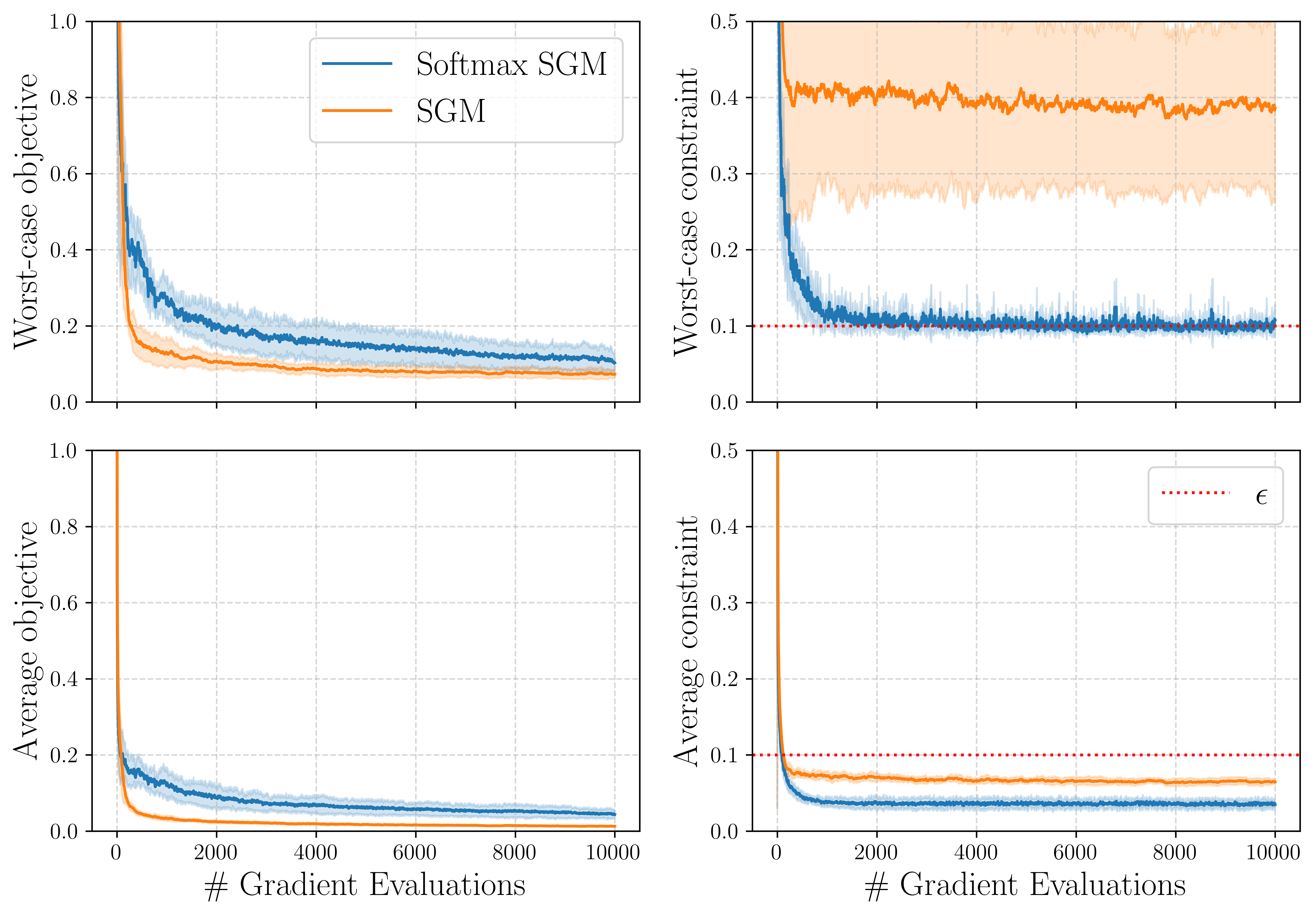}
    \caption{Worst-case and average objective and constraint violation for Softmax-SGM (solving robust minimax problem) and standard SGM (solving average/expected value problem) on the Breast Cancer dataset.}
    \label{fig:sgm_comparison}
\end{figure}

\subsection{Fair Classification}
\begin{table}[!ht]
    \caption{Detailed setting of fair classification experiment on Adult dataset}
    \label{tab:detail_fairness}
    \centering
    \begin{tabular}{lc|lc}
        \toprule
        Hyperparameter & Value & Hyperparameter & Value \\
        \midrule
        Number of runs & 5 & Global rounds (K) & 500 \\
        Local epochs (E) & 2 & Step size ($\eta$) & 0.001 \\
        Tolerance ($\epsilon$) & 0.05 & Neural network dimension ($d$) & 6,529 \\
        Total number of clients ($n$) & 10 & Participation rate & 0.5 \\
        Batch size per client & 128 & Softmax parameter ($\alpha$) & 1.0 \\
        Penalty paramter ($\rho$) & 10 & Dual parameter ($\lambda_0$) & 10 \\
        Dual step size ($\eta_d$) & 0.01 \\
        \bottomrule
    \end{tabular}
\end{table}

Fair classification task is formulated by the local objective of binary cross entropy loss and constraint of demographic parity: $$f_i(\mathbf{w}) := \frac{1}{m_i}\sum_{(x,y)\in\mathcal{D}_i}\left[y\log(\pi(x;\mathbf{w})) + (1-y)\log(1-\pi(x;\mathbf{w}))\right]$$
$$g_i(\mathbf{w}) := \left|\frac{1}{m_{i,p}}\sum_{x\in\mathcal{D}_{i,p}} \pi(x;\mathbf{w}) - \frac{1}{m_{i,u}}\sum_{x\in\mathcal{D}_{i,u}} \pi(x;\mathbf{w})\right|,$$
where the subscript $p$ and $u$ represent protected and unprotected groups respectively. Since the demographic parity is defined as the absolute difference between the average logits of protected and unprotected groups, the aggregation at the server is treated with extra care. At the server, the average logits were aggregated instead of the final constraint value to recalculated the weighted average of logits for the global constraint calculation:
$$g(\mathbf{w}) := \left|\frac{1}{m_{\mathcal{S},p}}\sum_{x\in\mathcal{D}_{\mathcal{S},p}} \pi(x;\mathbf{w}) - \frac{1}{m_{\mathcal{S},u}}\sum_{x\in\mathcal{D}_{\mathcal{S},u}} \pi(x;\mathbf{w})\right|,$$
where the subscript $\mathcal{S}$ denote all sampled clients. Here, we explore the setting of stochastic data sampling and deep neural network, making the problem highly non-convex, stochastic, and non-smooth. The experiments were conducted on Adult dataset \citep{kohavi1996adult}. For the penalty-based methods, the penalty parameter was chosen from $\rho\in[0.1, 1.0, 10.0, 100.0]$ and only the best results are presented. The detailed hyperparameters are found in \Cref{tab:detail_fairness}.

\subsection{Federated Safe RL (Constrained MDP)}

For broader evaluation, we introduce a new \textit{Federated Safe Reinforcement Learning} experiment.
To demonstrate this setting beyond tabular datasets, we introduce a heterogeneous Constrained Markov Decision Process (CMDP) CartPole experiment.

\textbf{Problem Formulation}

Each of $n=10$ clients interacts with a unique CMDP \citep{altman2021constrained} instance with a client-specific safety budget $d_i$. For a policy $\mathbf{w}$, the local objective to maximize is

$$
f_i(\mathbf{w}) = \mathbb{E}_{\mathbf{w},P}\left[\sum_{\tau=0}^{T-1}\gamma^\tau r(s_\tau,a_\tau,s_{\tau+1})\right],
$$

and the local constraint is

$$
g_i(\mathbf{w}) = \mathbb{E}_{\mathbf{w},P}\left[\sum_{\tau=0}^{T-1}\gamma^\tau c(s_\tau,a_\tau,s_{\tau+1})\right] - d_i \le 0.
$$

We optimize the worst-case global objective $F(\mathbf{w}) = \max_i f_i(\mathbf{w})$ subject to the worst-case global constraint $G(\mathbf{w}) = \max_i g_i(\mathbf{w}) \le 0$ using softmax-weighted aggregation.

\textbf{Environment and Heterogeneity}

At each timestep, an agent incurs a cost of 1 whenever one of the following safety conditions is violated:

\begin{enumerate}
    \item entering one of the five prohibited positional regions
    $[-2.4,-2.2]$, $[-1.3,-1.1]$, $[-0.1,0.1]$, $[1.1,1.3]$, and $[2.2,2.4]$; or
    \item the pole angle exceeding 6 degrees.
\end{enumerate}

To introduce severe heterogeneity beyond typical CMDP benchmarks, clients operate under different physics and distinct tolerances drawn from $[-5,5]$, yielding budgets $d_i \in [25, 35]$. The clients are distributed as follows:

\begin{itemize}
    \item 4 Earth clients with gravity 9.8, mass 0.1, and pole length 0.5;
    \item 3 Mars clients with gravity 3.7, mass 0.2, and pole length 0.8;
    \item 3 Jupiter clients with gravity 20.0, mass 0.5, and pole length 0.2.
\end{itemize}

Because data is sampled from distributions shifted by varying physics and policies, standard averaging fails to guarantee system-wide safety.

\textbf{Algorithm and Baselines}

We build on TRPO \citep{schulman2015trust} and compare Softmax SGM against Parallel CRPO \citep{xu2021crpo}, which solves the average-case constrained problem. Softmax SGM controls conservatism via the softmax parameter $\alpha$; a larger $\alpha$ more strongly enforces robust worst-case optimization.

Averaged over 5 seeds, we track the worst-case objective (Mean Return) and worst-case constraint violations.

\begin{table}[!ht]
    \caption{Worst-case performance in the federated safe RL experiment}
    \label{tab:fed_safe_rl_results}
    \centering
    \begin{tabular}{lcc}
        \toprule
        Algorithm & Mean Violations $\pm$ Std & Mean Return $\pm$ Std \\
        \midrule
        Softmax SGM (Ours) & 306.60 $\pm$ 24.55 & 69.73 $\pm$ 67.42 \\
        Parallel CRPO & 379.20 $\pm$ 24.59 & 60.99 $\pm$ 69.52 \\
        \bottomrule
    \end{tabular}
\end{table}

As shown in \Cref{fig:federated_safe_rl}, Softmax SGM outperforms Parallel CRPO in return maximization and constraint satisfaction. By targeting the maximum value problem, our approach ensures the most vulnerable clients remain within their safety envelopes.

\newpage
\section{Related Work}\label{sup:related}
\textbf{Federated Learning.}

Since the introduction of decentralized learning by \citet{mcmahan2017communication}, Federated Learning (FL) has emerged as a cornerstone of privacy-preserving AI. Its impact spans high-stakes sectors, from healthcare diagnostics~\citep{rieke2020future, peng2024depth} and financial modeling~\citep{wen2023survey,chatterjee2023federated} to battery management systems~\citep{wang2024adaptive, zhu2024collaborative}. While most FL research targets unconstrained objectives, real-world deployments must often satisfy fairness requirements, safety protocols, or resource budgets~\citep{du2021fairness, huang2024federated, yang2022federated}. This drives the development of constrained FL frameworks that integrate these requirements directly into the training loop to ensure reliable and compliant learning.
 
There have been many variants of FedAvg~\citep{mcmahan2017communication} to tackle different aspects of FL, such as data heterogeneity~\citep{karimireddy2020scaffold,seo2024relaxed,morafah2024towards}, system heterogeneity~\citep{gong2022fedadmm,li2020federated,wu2024fedbiot}, and fairness~\citep{li2021ditto,badar2024fairtrade}. Beyond FedAvg, ADMM-based FL algorithms have been explored to tackle heterogeneity in FL~\citep{acar2021federated,gong2022fedadmm,zhou2023federated}. Despite this extensive body of research, the literature is heavily skewed toward unconstrained scenarios. This leaves a significant gap in our understanding of how to effectively integrate rigorous functional constraints into the federated training process.
 
A common way to handle this setting is to cast the problem as a saddle-point formulation and apply primal-dual optimization methods~\citep{nemirovski2004prox,chambolle2011first,bertsekas2014constrained,hamedani2021primal,zhang2022solving,hounie2023resilient,boob2024optimal}. Although these techniques are well-supported by theory, they often depend heavily on careful tuning of the dual variables and, in many cases, assume access to a known upper bound on the optimal dual multiplier. In addition, such methods frequently require projecting both the primal and dual iterates onto bounded domains, which are rarely known a priori in practical applications and can complicate their implementation.
 
An alternative primal-only line of work is the switching (sub-)gradient method (SGM), originally introduced by \citet{polyak1967general}, which has since been explored in a variety of optimization settings. Subsequent studies, such as \citet{titov2018mirror,stonyakin2019mirror,alkousa2020modification}, have combined SGM with mirror descent to improve its flexibility and performance. Later, \citet{lan2020algorithms,lan2020algorithmsstochastic} extended SGM to stochastic settings and established iteration complexity guarantees under both convex and strongly convex assumptions. More recently, the theoretical analysis has been pushed further to cover weakly convex objectives, as shown in works like \citet{huang2023oracle,liu2025single,ma2020quadratically,jia2025first}. In the federated learning literature, \citet{islamov2025safe,upadhyay2026fedsgm} developed an important variant of SGM that accounts for communication compression.

Handling statistical heterogeneity across local clients is a foundational challenge in distributed optimization. 
From a statistical perspective, this non-IID landscape closely resembles learning from complex mixture distributions. 
While prior works have extensively studied the theoretical dynamics and non-asymptotic guarantees of iterative weighting algorithms
(such as Expectation-Maximization) in mixed linear regression \citep{luo2024unveiling, luo2025structural, luo2026characterizing}, 
our framework shifts the focus. 
Rather than aiming for maximum likelihood estimation, 
we utilize smoothing techniques~\citep{beck2012smoothing} such as
dynamic softmax weighting~\citep{wang2023task} to enforce worst-case distributionally robust optimization under strict operational constraints.

\textbf{Minimax Optimization.}

The mathematical foundations of minimax optimization were established by the seminal work of \citet{neumann1928theorie}, which proved the existence of saddle-point equilibria in zero-sum games. Early first-order approaches, such as the primal-dual methods of Arrow and Hurwicz, laid the groundwork for Gradient Descent-Ascent (GDA). However, GDA often fails to converge even in simple bilinear settings, a challenge that was later addressed by \citet{korpelevich1976extragradient} through the introduction of the \textit{Extragradient Method}, which remains the gold standard for smooth saddle-point problems.

In the era of large-scale machine learning, research has shifted toward stochastic regimes. \citet{nemirovski2009robust} and \citet{juditsky2011solving} established the optimal \(\calO(\epsilon^{-2})\) stochastic oracle complexity for convex-concave problems, which characterizes the fundamental limits of first-order stochastic optimization. Recently, this framework has been adapted to decentralized environments through \textit{Agnostic Federated Learning} \citep{mohri2019agnostic}, which employs a minimax objective to ensure model robustness across heterogeneous clients. Similarly, \citet{li2020fair} proposed the $q$-FFL framework to achieve fair resource allocation using weighted optimization objectives.

Our formulation is closely related to \textit{Group Distributionally Robust Optimization} (Group-DRO), where the objective is to minimize the loss under the worst-case distribution among a set of pre-defined groups. While traditional DRO solvers often rely on dual reformulations that can be sensitive to noise in federated settings, our switching mechanism provides a robust, primal-only alternative that maintains stability even under heavy client heterogeneity and partial participation.

To contextualize our work within the broader optimization literature, we observe that our problem formulation is fundamentally a specialized instance of \textit{Group Distributionally Robust Optimization} (Group-DRO), similar to the frameworks discussed by \citet{sagawa2020distributionally}, where the uncertainty set is defined over the convex hull of discrete client distributions. Our objective is further aligned with the \textit{Agnostic Federated Learning} (AFL) paradigm introduced by \citet{mohri2019agnostic}, which seeks to minimize the loss of the worst-performing client distribution to ensure global model robustness. However, while AFL and traditional DRO solvers often rely on primal-dual updates~\citep{li2021rate,huang2025stochastic} that can become unstable under the ``dual drift'' inherent in partial participation, our approach adopts the logic-based \textit{Switching Gradient} mechanism established by \citet{lan2020algorithms,lan2020algorithmsstochastic}.

The current state-of-the-art for handling expectation-constrained stochastic problems is the \textit{Switching Gradient} mechanism developed by \citet{lan2020algorithms,lan2020algorithmsstochastic}. This approach avoids the inherent instability of dual-variable drift by alternating updates between the objective and the constraint based on a feasibility trigger. Our work represents the latest advancement in this lineage; by integrating a softmax-weighted smoothing technique, we extend the switching framework to the partial participation regime. Notably, we improve the high-probability complexity from the \(\calO(\log^2\frac{1}{\delta})\) found in existing literature to a sharper \(\calO(\log\frac{1}{\delta})\) guarantee, providing the most robust theoretical bound for constrained federated minimax problems to date.


\end{document}